\documentclass[twoside,11pt]{article}

\usepackage{blindtext}

\usepackage[preprint]{jmlr2e}

\usepackage[utf8]{inputenc} \usepackage[T1]{fontenc}    \usepackage{hyperref}       \hypersetup{hidelinks}
\usepackage{url}            \usepackage{booktabs}       \usepackage{amsfonts}       \usepackage{dsfont}
\usepackage{nicefrac}       \usepackage{microtype}      \usepackage{xcolor}         \usepackage{natbib}
\usepackage{enumitem}

\usepackage{graphicx, wrapfig, subcaption}
\usepackage{floatrow} 

\usepackage{color, colortbl}
\definecolor{darkblue}{rgb}{0.0,0.0,0.65}
\definecolor{darkred}{rgb}{0.65,0.0,0.0}
\hypersetup{
  colorlinks = true,
  citecolor  = darkblue,
  linkcolor  = darkred,
  filecolor  = darkblue,
  urlcolor   = darkblue,
}

\usepackage{amsmath, amssymb}
\usepackage{verbatim}

\usepackage{booktabs,makecell,multirow}       \usepackage{xspace}
\usepackage{footnote, tablefootnote}
\usepackage{float} \floatstyle{plaintop}
\restylefloat{table}

\newcommand{\high}[1]{\textcolor{darkblue}{\bf #1}}

\usepackage{listings}
\definecolor{codegreen}{rgb}{0,0.6,0}
\definecolor{codegray}{rgb}{0.5,0.5,0.5}
\definecolor{codepurple}{rgb}{0.58,0,0.82}
\definecolor{backcolour}{rgb}{0.95,0.95,0.92}

\newtheorem{assumption}{Assumption}

\newcommand{\RR}{\mathbb{R}}
\newcommand{\norm}[1]{\left\lVert #1 \right\rVert}
\DeclareMathOperator*{\argmin}{argmin}
\newcommand{\sgn}{\mathrm{sign}}

\newcommand{\1}{{\rm 1}\kern-0.24em{\rm I}}

\DeclareMathOperator*{\argmax}{argmax}
\newcommand{\breg}[3]{D_{#1}\left(#2,#3\right)}
\newcommand{\brg}[2]{D_{\psi}\left(#1,#2\right)}
 \newcommand{\brgl}[2]{D_{L}\left(#1,#2\right)}

\newcommand{\st}{\mathrm{s.t.}}

\newcommand{\R}{\mathbb{R}}

\newcommand{\inp}[2]{\left\langle#1,~ #2 \right\rangle}

\newcommand{\mir}{\psi}
\newcommand{\reg}[1]{u^{\sf \tiny r}_{#1}}
\newcommand{\mmd}[1]{u^{\sf \tiny m}_{#1}}
\newcommand{\mar}[1]{\hat{\gamma}_{#1}}

\newcommand{\algname}{$p$-{\small \sf GD}\xspace}

\newcommand{\pmval}[2]{#1 \textcolor{gray}{$\pm$  #2}}
\newcommand{\bpmval}[2]{\textbf{#1} \textcolor{gray}{$\pm$  #2}} 

\usepackage{lastpage}

\jmlrheading{24}{2023}{1-\pageref{LastPage}}{6/23; Revised 12/23}{1/24}{23-0836}{Haoyuan Sun, Khashaiar Gatmiry, Kwangjun Ahn, and Navid Azizan}

\ShortHeadings{Controlling Implicit Bias via Mirror Descent}{Sun, Gatmiry, Ahn, and Azizan}
\firstpageno{1}

\begin{document}

\title{A Unified Approach to Controlling Implicit Regularization via Mirror Descent}

\author{\name Haoyuan Sun\email haoyuans@mit.edu \\
        \name Khashayar Gatmiry\email gatmiry@mit.edu \\
    \name Kwangjun Ahn\email kjahn@mit.edu \\
    \name Navid Azizan\email azizan@mit.edu \\
       \addr Massachusetts Institute of Technology\\
       Cambridge, MA 02139, USA
}

\editor{Mahdi Soltanolkotabi}

\maketitle

\begin{abstract}Inspired by the remarkable success of large neural networks, there has been significant interest in understanding the generalization performance of over-parameterized models. Substantial efforts have been invested in characterizing how optimization algorithms impact generalization through their ``preferred'' solutions, a phenomenon commonly referred to as \emph{implicit regularization}.
    In particular, it has been argued that gradient descent (GD) induces an implicit $\ell_2$-norm regularization in regression and classification problems.
    However, the implicit regularization of different algorithms are confined to either a specific geometry or a particular class of learning problems, indicating a gap in a general approach for controlling the implicit regularization.
    To address this, we present a unified approach using mirror descent (MD), a notable generalization of GD, to control implicit regularization in both regression and classification settings.
    More specifically, we show that MD with the general class of homogeneous potential functions converges in direction to a \textit{generalized maximum-margin} solution for linear classification problems, thereby answering a long-standing question in the classification setting.
    Further, we show that MD can be implemented efficiently and enjoys fast convergence under suitable conditions. Through comprehensive experiments, we demonstrate that MD is a versatile method to produce learned models with different regularizers, which in turn have different generalization performances.
\end{abstract}


\section{Introduction}

In recent years, deep neural networks have enjoyed a tremendous amount of success in a wide range of applications~\citep{schrittwieser2020mastering,ramesh2021zero, brown2020language, dosovitskiy2020image}.
Notably, many of these modern machine learning problems operate in the so-called \textit{over-parameterized} regime, where the number of model parameters is sufficiently large to allow for perfectly fitting the training data~\citep{allen2019convergence,belkin2021fit}.
However, such highly expressive models have the capacity to have multiple solutions that interpolate training data, and these solutions can often perform very differently on test data.
Without knowing which of these interpolating solutions the optimizer finds, it would be difficult to identify whether the learned model \textit{overfits}, where it performs well on the training data but generalizes poorly on the test data.
Therefore, a characterization of the optimizers' solutions is essential to the understanding of the generalization performance of modern over-parameterized models, which is one of the most fundamental questions in machine learning.

Notably, it has been observed that even in the absence of any explicit regularization, the interpolating solutions obtained by many gradient-based optimization algorithms, such as (stochastic) gradient descent, tend to generalize well. Recent research has highlighted that these algorithms converge to solutions with certain properties, i.e., they \emph{implicitly regularize} the learned models.
Importantly, it has been argued that such implicit biases play a significant role in determining generalization performance
\citep{neyshabur2014search, zhang2021understanding, wilson2017marginal,liang2020just,donhauser2022fast}.

In the literature, the implicit bias of first-order methods is first studied in linear settings since the analysis is more tractable.
Nevertheless, there is significant theoretical and empirical evidence suggesting that certain insights from linear models translate to the case of deep models, e.g., \cite{jacot2018neural,allen2019convergence,belkin2019reconciling,bartlett2017spectrally,nakkiran2021deep,du2018gradient}. 
In the linear setting, implicit bias has been extensively analyzed in the contexts of different learning problems, of which two have received the most attention.
The first case is least-squares regression, where the loss function has a global minimum attainable at a finite value.
And the second case is classification with logistic or exponential loss, where the loss function does not have an attainable global minimum.
While they are colloquially referred to as regression or classification problems, respectively, we note that there are other examples of the loss function, such as hinge loss, that do not fall into these two categories.
In Section~\ref{sec:priliminaries}, we will define these notions more formally.

For the analysis of implicit regularization, the most well-studied optimization algorithm is gradient descent (GD).
The implicit bias of gradient descent (GD) for the square loss goes back to \cite{engl1996regularization}, and possibly earlier, where it was shown that GD converges to the global minimum that is closest to the initialization in the Euclidean distance.
For the logistic loss, it has been shown that the gradient descent iterates converge to the $\ell_2$-maximum-margin SVM solution in direction~\citep{soudry2018implicit, ji2019implicit}.
Beyond GD, it has been shown that mirror descent (MD), which is an important generalization of GD, converges to the interpolating solution that is closest to the initialization in terms of a Bregman divergence~\citep{gunasekar2018characterizing,azizan2018stochastic}. 
On the classification side, it has been established that the implicit biases of various algorithms such as AdaBoost~\citep{rosset2004boosting,telgarsky2013margins} and steepest descent~\citep{gunasekar2018characterizing} maximize the margin with respect to certain norms.

There are also many counter-examples where an optimization algorithm does not exhibit an implicit bias independent of hyper-parameters such as the step size.
For instance, it has been shown that the AdaGrad algorithm does not have such an implicit bias in the classification setting~\citep{gunasekar2018characterizing,wang2021implicit}. 
Further, it is possible for optimization algorithms to exhibit step-size-invariant implicit bias in regression but not in classification and vice versa (e.g., steepest descent).
To our best knowledge, gradient descent is the only first-order algorithm known to induce a step-size-invariant implicit bias in both settings of regression with square loss and classification with logistic loss.
Therefore, there is still a significant gap in the understanding of implicit regularization for different classes of losses and some believe that the regression and classification settings are ``fundamentally different'' \citep{gunasekar2018characterizing}.
In this paper, we show that mirror descent can extend the implicit bias of gradient descent to more general geometries in the classification setting.
So, we conclude that mirror descent is the first known algorithm exhibiting implicit regularization for both general geometry and different classes of loss functions.
Furthermore, we show that under many circumstances, mirror descent can both be efficiently implemented and quickly converge to its implicitly regularized solution.
Hence, mirror descent is a versatile way of implicitly enforcing desirable properties on the learned model in a variety of tasks.
See Table~\ref{table:main} for a summary.

\begin{table}[t]
\centering
\renewcommand{\arraystretch}{1.25}
\setlength\tabcolsep{8pt}

\begin{tabular}{ |c |c|c| }
\hline  
& Regression (e.g. square loss) & Classification (e.g. logistic loss) \\
& with $w_0=\argmin \psi(w)$ & with any initialization \\
\hline\hline  
\multirow{4}{*}{\begin{tabular}{c}Gradient Descent\\(i.e. $\mir(\cdot) = \frac{1}{2}\norm{\cdot}_2^2$)\end{tabular}}  & $\argmin_w \norm{w}_2$ & $\argmin_w \norm{w}_2$  \\
 & $\st~~w \text{ fits all data} $ & $\st~~w \text{ classifies all data} $    \\
 & 
\multirow{2}{*}{\citep[Thm 6.1]{engl1996regularization}} & \cite{soudry2018implicit}   \\
 & & \cite{ji2019implicit} \\
   \hline
 \multirow{4}{*}{\begin{tabular}{c}Mirror Descent\\ \end{tabular}}  & $\argmin_w \psi(w)$ & $\argmin_w \psi(w)$  \\
 & $\st~~w \text{ fits all data} $ & $\st~~w \text{ classifies all data} $    \\
 & 
\cite{gunasekar2018characterizing} & \multirow{2}{*}{ \large \high{This work}}  \\
& \cite{azizan2018stochastic} & \\
  \hline 
\end{tabular}

\caption{{\bf Conceptual summary of our results.} For both well-specified linear regression and separable linear classification, gradient descent converges to the ``smallest'' global minimum in $\ell_2$-norm.
In the case of linear regression, it is known that mirror descent generalizes the implicit bias of gradient descent to any strictly convex potential.
However, a similar characterization is missing for mirror descent under the classification setting.
In this paper, we prove the implicit regularization of mirror descent with the class of homogeneous potentials and extend the result of gradient descent beyond $\ell_2$-norm.
In Appendix~\ref{sec:table-full}, we present a more complete summary where we consider the implicit bias in the regression setting with any initialization.
}
\label{table:main}
\end{table}

\subsection{Our contributions}
In this paper, our theoretical and empirical contributions are as follows:
\vspace{-0.5em}
\begin{list}{{\tiny $\blacksquare$}}{\leftmargin=1.5em}
\setlength{\itemsep}{-1pt}
    \item We study mirror descent (MD) with the general class of homogeneous potential functions. 
    In Section~\ref{sec:primal-bias-result}, we show that for separable linear classification with logistic loss, MD with homogeneous potential exhibits implicit regularization by converging in direction to a generalized maximum-margin solution with respect to the potential function.
    More generally, we show that, for any strictly decreasing loss function, MD follows the so-called regularization path. The precise terminologies are defined in Sections~\ref{sec:priliminaries} and~\ref{sec:main-result}.
    \item We study the rate at which MD with homogeneous potential converges in direction to the generalized maximum-margin solution. In Section~\ref{sec:asymp-result}, we show that with fixed step sizes, MD converges in direction to the maximum-margin solution at a poly-logarithmic rate in the number of iterates $T$. And in Section~\ref{sec:normalized-asymp}, with additional assumptions on the potential function, we prove the convergence of a variant of normalized MD and show the rate of convergence can be accelerated to be polynomial in the number of iterates $T$.
    \item In Section~\ref{sec:experiments}, we investigate the implications of our theoretical findings by applying a subclass of MD that is both efficient and scalable.
    Our experiments involving linear models corroborate our theoretical results in Section~\ref{sec:main-result}, and real-world experiments with deep neural networks and popular datasets suggest that our findings carry over to such nonlinear settings.  Our deep learning experiments further show that mirror descent with respect to different potential functions can lead to different solutions with significantly different generalization performance.
\end{list}

\subsection{Related work}
\paragraph{Implicit regularization in regression.}
Regression problems are typically concerned with the case of square loss.
For gradient descent (GD) with square loss on a linear model, it is known that GD converges to the minimizer that is closest to the initialization in the Euclidean sense~\citep{engl1996regularization}. 
We can induce implicit bias with respect to other geometries with a family of algorithms called mirror descent (MD), which is an extension of GD. 
In particular, it has been shown that mirror descent converges to the interpolating solution that is closest to the initialization in terms of a Bregman divergence with respect to MD's potential function~\citep{gunasekar2018characterizing,azizan2018stochastic}.\footnote{MD has also been used for explicit regularization, e.g., \cite{azizan2021explicit}, which is not the focus of this work.}
Additionally, \cite{gunasekar2018characterizing} showed that a variant of MD with momentum also converges to the same implicit bias, and \cite{azizan2018stochastic} analyzed stochastic MD.
So, we consider the study of implicit bias in linear regression to be relatively well-understood by now.

\paragraph{Implicit regularization in classification.}
Classification problems are typically concerned with the case of logistic or exponential loss.
Another common loss function for classification problems is the hinge loss, but as we will note in Section~\ref{sec:priliminaries}, this loss is uninteresting in terms of analyzing implicit bias.
A key differentiating factor in the classification setting is that the loss function does not attain its minimum at a finite value, and the weights have to grow to infinity.
For the logistic loss and other strictly decreasing losses, it has been shown that gradient descent iterates converge to the $\ell_2$-maximum-margin solution in direction~\citep{soudry2018implicit, ji2019implicit, ji2020gradient}.
Further applying various schemes of adaptive step size to gradient descent can accelerate its convergence to the solution induced by its implicit bias~\citep{nacson2019convergence,ji2021characterizing,ji2021fast}.
Beyond implicit bias with respect to the $\ell_2$-norm, we also know that AdaBoost converges to the $\ell_1$-maximum margin direction~\citep{rosset2004boosting,telgarsky2013margins}.
Also, \cite{gunasekar2018characterizing} showed that the steepest descent algorithm converges to the maximum-margin direction with respect to a general norm; however, this algorithm cannot be implemented efficiently in practice.
The analysis of mirror descent for this setting had been more limited; the only prior work we are aware of showed a special case where mirror descent's potential is the Mahalanobis distance\footnote{which is equivalent to the Euclidean distance up to a linear transformation in the coordinates.}~\citep{li2021implicit}.

\paragraph{Connections between regression and classification.}
For over-parameterized models, empirical evidence shows that using either least-square loss or cross-entropy loss can lead to comparable performance on classification tasks~\citep{hui2020evaluation}.
This suggests that the distinction between regression and classification problems is blurred when we enter the over-parameterized regime.
There has been some progress in theoretically explaining this phenomenon, for example, \cite{muthukumar2021classification} showed that for over-parameterized linear models and when the data are drawn from a Gaussian distribution, the minimum-$\ell_2$-norm interpolating solution and the $\ell_2$-maximum-margin solutions are equivalent with high probability.
However, it is unknown whether such connections exist for more general geometries extending beyond the $\ell_2$-norm.

\paragraph{Extension to nonlinear models.}
In addition to linear models, several works have analyzed the implicit bias when we optimize over a nonlinear model such as neural networks.
There is now a good understanding of the case of simpler networks such as homogeneous networks without nonlinearity~\citep{lyu2019gradient,vardi2022margin,wang2021implicit}, ReLU networks~\citep{ji2019polylogarithmic,zou2020gradient}, or networks with dropout~\citep{mianjy2018implicit,wei2020implicit}.
However, the development of implicit regularization for more general deep neural networks remains an ongoing research direction.

 \section{Background}
\label{sec:priliminaries}
\subsection{Problem setting}
\label{sec:setting}
In this paper, we are interested in the standard \textit{empirical risk minimization} problem.
Consider a collection of input-output pairs $\{(x_i, y_i)\}_{i=1}^n \subset \RR^d \times \RR$ and a model $f_w(x) : \RR^d \to \RR$ with parameter $w \in \mathcal{W}$.
For some convex \textit{loss function} $\ell : \RR \times \RR \to \RR$, our goal is to minimize the empirical loss:
\[ L(w) = \frac{1}{n}\sum_{i=1}^n \ell(f_w(x_i), y_i).\]
We can categorize the loss functions by the properties of their minimizer.
For simplicity, we assume without loss of generality that $\inf \ell(\cdot) = 0$.
The first type is concerned with regression problems, whereas the next two types of losses are concerned with classification problems, where the output variable is a binary label $y \in \{\pm 1\}$.
\begin{list}{{\tiny $\blacksquare$}}{\leftmargin=1.5em}
\setlength{\itemsep}{-1pt}
    \item \textbf{Losses with a unique minimizer.} In this case, the minima of the loss function is attained if and only if $f_w(x) = y$. Example: square loss $\ell(f_w(x), y) = (f_w(x) - y)^2$.
    \item \textbf{Losses with non-unique minimizer.} In this case, there are other minimizers in addition to the ones at $f_w(x) = y$. Example: hinge loss $\ell(f_w(x), y) = \max(0, 1 - f_w(x)y)$.
    \item \textbf{Strictly monotone losses.} A finite minimizer is not attainable in this case, but for any fixed value of $y$, the loss $\ell(f_w(x), y)$ strictly decreases with respect to $f_w(x)y$. Example: exponential loss $\ell(f_w(x), y) = \exp(- f_w(x)y)$, or logistics loss.
\end{list}

We note that gradient descent does not exhibit an implicit bias for losses with a non-unique minimizer because the final solution will depend on the step size (see the example below).
Therefore, we are mainly interested in only the first and third types of loss functions.
For simplicity, unless otherwise stated, we informally refer to regression as problems with square loss and classification as problems with exponential loss.

\begin{example}
We give a simple example where the solution found by gradient descent with the hinge loss is dependent on the step size.
Consider a classification dataset with three points in $\RR^3$: $((1, 0, 0), +1), ((0, 1, 0), +1), ((0, 0, 1), +1)$, a linear model $f_w(x) = w^\top x$ and initial weight $w_0 = (-1, -2, -3)$.
When the step size is $\eta = 2$, the iterates are $w_1 = (1, 0, -1), w_2 = (1, 2, 1)$; and when the step size is $\eta = 3$, the iterates are $w_1 = (2, 1, 0), w_2 = (2, 1, 3)$.
The final solutions are dependent on the step size, so gradient descent does not exhibit implicit regularization with the hinge loss.
\end{example}

For our theoretical analysis, we focus on a linear model, where the models can be expressed by $f_w(x) = w^\top x$ and $w \in \RR^d$, in classification problems with strictly monotone losses.
We also make the following assumptions about the data. 
First, since we are mainly interested in the over-parameterized setting where $d \gg n$, we assume that the data is linearly separable, i.e., 
there exists  $w^* \in \RR^d$ s.t. $\sgn(\inp{w^*}{x_i}) = y_i$ for all $i\in[n]$.
We also assume that the inputs $x_i$'s are bounded, where $\max_i \norm{x_i} < C$, for some relevant norm $\norm{\cdot}$ which we will specify in Section~\ref{sec:main-result}. 

\subsection{Preliminaries on mirror descent}
The key component of mirror descent is a \textit{potential function}. 
In this work, we will focus on differentiable and strictly convex potentials $\psi$ defined on the entire domain $\R^n$.\footnote{In general, the mirror map is a convex function of Legendre type~(see, e.g., \citep[Sec. 26]{rockafellar1970convex}).} 
We call $\nabla \psi$ the corresponding \textit{mirror map}.
Given a potential, the natural notion of ``distance'' associated with the potential $\psi$ is given by the Bregman divergence.

\vspace{-0.5em}

\begin{definition}[Bregman divergence~\citep{bregman1967relaxation}]
For a strictly convex potential function  $\mir$, the Bregman divergence $\breg{\mir}{\cdot}{\cdot}$ associated to $\mir$ is defined as
\begin{align*}
    \breg{\mir}{x}{y}:= \mir(x)-\mir(y) -\inp{\nabla \mir(y)}{x-y},\qquad \forall x,y\in \R^n\,.
\end{align*}
\end{definition} 

\vspace{-0.25em}

An important case is the potential $\psi = \frac{\norm{\cdot}_2^2}{2}$, where $\norm{\cdot}_2$ denotes the Euclidean norm.
Then, the Bregman divergence becomes $D_\psi(x,y) = \frac{1}{2}\norm{x-y}_2^2$.
As an example for more complicated geometry, when the potential is $\psi(x) = \frac{1}{2} x^\top P x$ with positive definite $P \succ 0$, then the Bregman divergence becomes the Mahalanobis distance.
For more background on Bregman divergence and its properties, see, e.g., \citep[Section 2.2]{bauschke2017descent} and \citep[Section  II.A]{azizan2019stochastic}.

The mirror descent (MD) algorithm~\citep{nemirovskij1983problem} is a generalization of gradient descent over geometries beyond the Euclidean distance.
In mirror descent with potential $\psi$, we use Bregman divergence as a measure of distance:
\begin{align} \tag{\sf MD}\label{equ:md}
    w_{t+1} = \argmin_w \left\{\frac{1}{\eta}D_\psi(w, w_t) + \inp{\nabla L(w_t)}{w}\right\}
\end{align}

Equivalently, \ref{equ:md} can be written as  $\nabla\psi(w_{t+1}) = \nabla\psi(w_t) - \eta \nabla L(w_t)$. We refer readers to \cite[Figure 4.1]{bubeck2015convex} for a nice illustration of mirror descent.
Also, see \citep[Section 5.7]{juditsky2011first} for various examples of potentials depending on applications.

One property we will repeatedly use is the following~\citep{azizan2018stochastic}:

\begin{lemma}[\ref{equ:md} identity]
\label{thm:key-iden}
For any $w \in \RR^n$, the following identities hold for \eqref{equ:md}\footnote{For convenience, for a function $f$, we write $D_f(x, y) := f(x) - f(y) - \inp{\nabla f(y)}{x-y}$. Note that when $f$ is convex, $D_f(\cdot, \cdot) \ge 0$, and when $f$ is strictly convex, $D_f(\cdot, \cdot)$ is the Bregman divergence.}:
\begin{subequations}
\begin{align}
    &D_\psi(w, w_t) = D_\psi(w, w_{t+1}) + D_{\psi - \eta L}(w_{t+1}, w_t) + \eta D_{L}(w, w_t) - \eta L(w) + \eta L(w_{t+1})\,, \label{equ:key-iden-1}\\
     &\quad = D_\psi(w, w_{t+1}) + D_{\psi - \eta L}(w_{t+1}, w_{t}) - \eta \inp{\nabla L(w_t)}{w - w_t} - \eta L(w_{t}) + \eta L(w_{t+1})\, \label{equ:key-iden-2}.
\end{align}
\end{subequations}
\end{lemma}

Using Lemma~\ref{thm:key-iden}, we make several new observations and prove the following useful statements.

\begin{lemma}
\label{thm:decreasing-lose}
For sufficiently small step size $\eta$ such that $\psi - \eta L$ is convex, the loss is monotonically decreasing after each iteration of \eqref{equ:md}, i.e., $L(w_{t+1}) \le L(w_{t})$.
\end{lemma}

\begin{lemma}
\label{thm:to-infinity}
In a separable linear classification problem, if  $\eta$ is chosen sufficiently small s.t. $\psi - \eta L$ is convex, then we have $L(w_t) \to 0$ as $t \to \infty$. Hence, $\lim_{t\to \infty}\norm{w_t} = \infty$ for any norm $\norm{\cdot}$.
\end{lemma}

The formal proofs of these lemmas can be found in Appendix~\ref{sec:proof-basic-lemmas}.

\begin{remark}
\label{rem:step-size}
One can relax the condition in Lemma \ref{thm:decreasing-lose} and \ref{thm:to-infinity} such that for a sufficiently small step size $\eta$, $\psi - \eta L$ only has to be locally convex at the iterates $\{w_t\}_{t=0}^\infty$.
The relaxed condition allows us to analyze losses such as the exponential loss (see, e.g. footnote 2 of \cite{soudry2018implicit}).
This condition can be considered as the mirror descent counterpart to the standard smoothness assumption in the analysis of gradient descent (see \cite{lu2018relatively}).
In Appendix~\ref{sec:step-size}, we discuss in detail the existence of such a step size for the exponential loss.
\end{remark}

\subsection{Preliminaries on implicit regularization}
As we discussed above,  the weights vector $w_t$ diverges for mirror descent.
Here the main theoretical question is:

\begin{center}
What direction does \ref{equ:md} converge to? In other words, \\under some norm $\norm{\cdot}$, can we characterize $w_t / \norm{w_t}$ as $t\to \infty$?
\end{center}

To define a notion of ``norm'' respecting the geometry induced by potential $\psi$, we let $\norm{\cdot}_\psi$ be the Minkowski functional of $\psi$'s unit sub-level set:
\begin{equation}
    \label{equ:potential-norm}
    \norm{w}_\psi := \inf \{c > 0 : \psi(w / c) \le 1\}
\end{equation}
For a wide class of convex $\psi$, this definition indeed gives us a norm.
In Section~\ref{sec:main-result}, we shall give a sufficient condition for $\norm{\cdot}_\psi$ to be a norm.
With the definition of $\norm{\cdot}_\psi$ in mind, we introduce two special directions whose importance will be illustrated later.

\begin{definition}
The \textbf{regularization path} with respect to $\norm{\cdot}_\psi$ is defined as
\begin{equation}
    \bar{w}_\psi(B) = \argmin_{\norm{w}_\psi \le B} L(w)
\end{equation}
And if the limit $\lim_{B\to\infty} \bar{w}_\psi(B) / B$ exists, we call it the \textbf{generalized regularized direction} and denote it by $\reg{\psi}$.
\end{definition}

\begin{definition}
The \textbf{margin} $\gamma$ of the a linear classifier $w$ is defined as
$\gamma(w) = \min_{i \in [n]} y_i \inp{x_i}{w}. $
The \textbf{generalized max-margin direction} with respect to $\psi$ is defined as:
\begin{equation}
    \mmd{\psi} := \argmax_{\psi(w) \le 1} \left\{ \min_{i=1, \dots, n} y_i \inp{x_i}{w} \right\}
\end{equation}
And let $\mar{\psi}$ be the optimal value to the equation above.\footnote{Note that, when $\norm{\cdot}_\psi$ is a norm, the sets $\{\psi(\cdot) \le 1\}$ and $\{\norm{\cdot}_\psi \le 1\}$ are equal. In this paper, we will use these two formulations interchangeably.}
\end{definition}

\begin{wrapfigure}[16]{r}{0.47\textwidth}
\centering 
\includegraphics[width=0.9\textwidth]{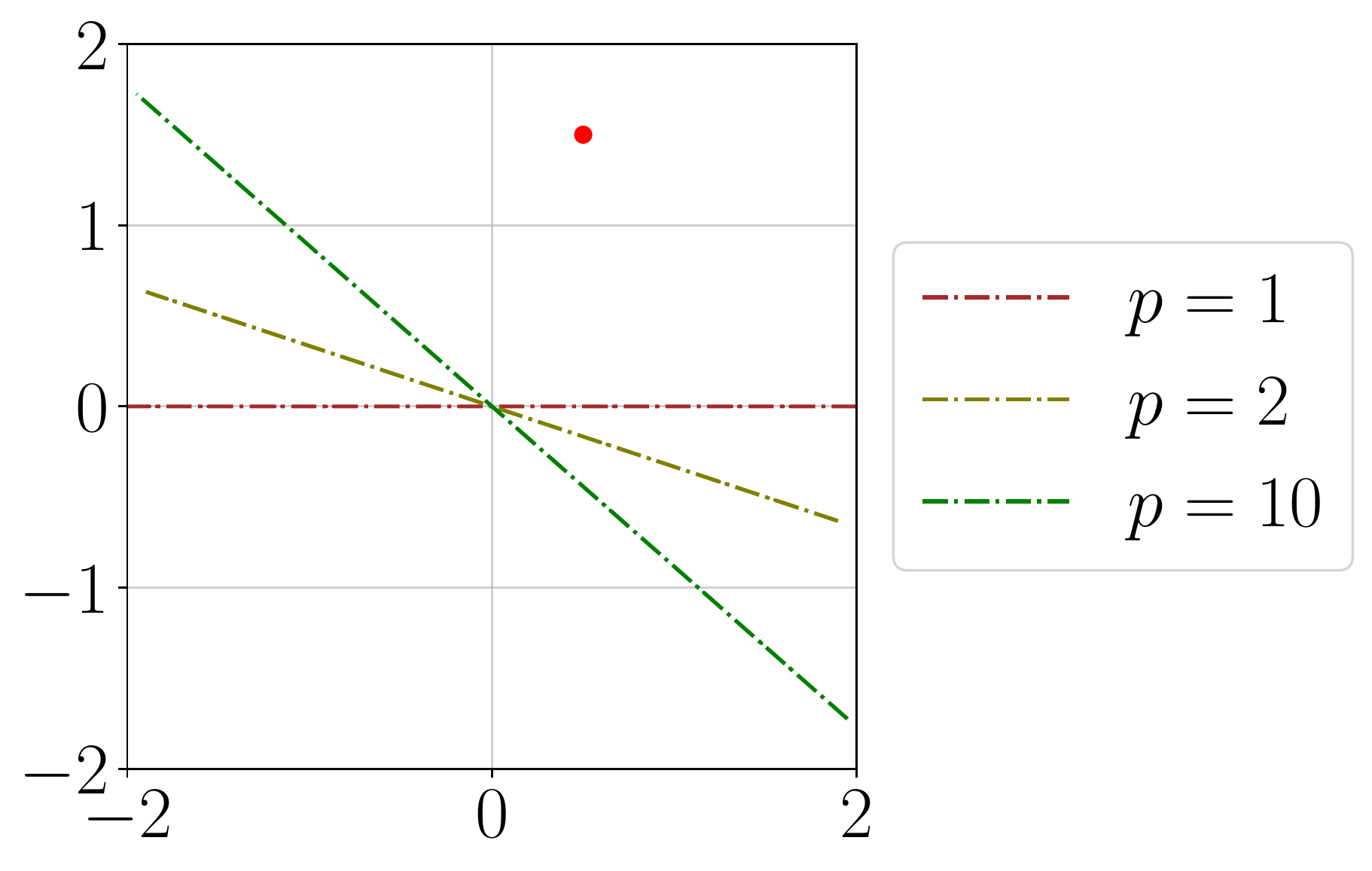}
\caption{The generalized maximum-margin solution to a single data point (denoted by \textcolor{red}{$\bullet$}) with respect to the $\ell_{1}, \ell_2$, and $\ell_{10}$ norms. For each generalized max-margin solution $u$, we plot the decision boundary $\{x \mid u^\top x = 0\}$.}
\label{tab:max-margin} 
\end{wrapfigure}

For simpler case where we have the $\ell_p$-norm (one possible corresponding potential is $\psi(\cdot) = \frac{1}{p} \norm{\cdot}_p^p$), we overload the notation with $\reg{p}$ and $\mmd{p}$.
Note that the superscripts in $\reg{\psi}$ and $\mmd{\psi}$ are not variables and we only use this notation to differentiate the two definitions.

To illustrate the effect of the potential $\psi$ on the maximum-margin solution, we consider a dataset consisting of a single point $((\frac{1}{2}, \frac{3}{2}), +1)$.
For $\mmd{2}$, we get the SVM solution whose decision boundary is orthogonal to the line connecting $(\frac{1}{2}, \frac{3}{2})$ to the origin.
And for $\mmd{1}$, we get a ``sparse'' max-margin solution that is zero in one coordinate.
Lastly, because $\norm{\cdot}_{10}$ is very close to $\norm{\cdot}_\infty$, the coordinates in the max-margin solution $\mmd{10}$ are very close to each other.

Prior results had shown that, in linear classification, gradient descent converges in direction to $\mmd{2}$, which is parallel to the hard-margin SVM solution w.r.t. $\ell_2$-norm: $\argmin_w \{\norm{w}_2 : \gamma(w) \ge 1\}$.

\begin{theorem}[\cite{soudry2018implicit}]
\label{thm:gd-maxmargin}
For separable linear classification problems with logistics/exponential loss, the gradient descent iterates with sufficiently small step size converge in direction to $\mmd{2}$, i.e., $\lim_{t\to\infty} \frac{w_t}{\norm{w_t}_2} = \mmd{2}$. 
\end{theorem}

\begin{theorem}[\cite{ji2020gradient}]
\label{thm:gd-regdir}
If the regularized direction $\reg{2}$ with respect to the $\ell_2$-norm exists, then the gradient descent iterates with sufficiently small step size converge to the regularized direction $\reg{2}$, i.e., $\lim_{t\to\infty} \frac{w_t}{\norm{w_t}_2} = \reg{2}$. 
\end{theorem}

As for mirror descent, an earlier version of this paper \citep{sun2022mirror} studied when the potential function is $\psi(\cdot) = \frac{1}{p} \norm{\cdot}_p^p$ and showed that $w_t / \norm{w_t}_p$ converges in direction. 
\begin{theorem}[\cite{sun2022mirror}]
\label{thm:pgd-regdir}
Given a separable linear classification problem with strictly monotone loss.
If the regularized direction $\reg{p}$ with respect to the $\ell_p$-norm exists, then the iterates of mirror descent with $\psi(\cdot) = \frac{1}{p} \norm{\cdot}_p^p$ and sufficiently small step size converge to the generalized regularized direction $\reg{p}$, i.e., $\lim_{t\to\infty} \frac{w_t}{\norm{w_t}_p} = \reg{p}$. 
\end{theorem}

In Section~\ref{sec:primal-bias-result}, we will generalize Theorems~\ref{thm:gd-maxmargin}-\ref{thm:pgd-regdir} to a more general setting of mirror descent with the class of homogeneous potential function.

\section{Mirror Descent with Homogeneous Potential}
\label{sec:main-result}

In this section, we investigate the implicit regularization property of mirror descent for classification problems with strictly monotone losses.
For the notion of ``direction'' to be well-defined, we impose the following properties on the potential function.
\begin{assumption}
\label{thm:assump-potential}
We assume that the potential function $\psi$ has the following properties:
\begin{list}{{\tiny $\blacksquare$}}{\leftmargin=1.5em}
\setlength{\itemsep}{-1pt}
    \item $\psi$ is twice differentiable and strictly convex, i.e. $\nabla^2 \psi \succ 0$.
    \item $\psi$ is positive definite in the sense that $\psi(\cdot) \ge 0$ and $\psi(x) = 0$ if and only if $x = 0$.
    \item For some constant $\beta > 1$, $\psi$ is $\beta$-absolutely homogeneous in the sense that for any scalar $c$ and vector $x \in \RR^d$, we have $\psi(cx) = |c|^\beta \psi(x)$.
\end{list}
\end{assumption}
Under this assumption, we can verify that $\norm{\cdot}_\psi$ as defined in \eqref{equ:potential-norm} is indeed a norm. 
In particular, if $\psi(\cdot) = \norm{\cdot}^\beta$ for some norm $\norm{\cdot}$, then we have $\norm{\cdot}_\psi = \norm{\cdot}$.
Also, it is not difficult to show that $\psi$ satisfies the following properties, so that the Bregman divergence is also absolutely homogeneous.
This would allow us to easily normalize the weight vector $w$ in our computations.
\begin{subequations}
\label{equ:breg-homo}
\begin{align}
 \inp{\nabla \psi (w)}{w} &= \beta \cdot \psi(w)\\
  \brg{c w}{c w'} &= |c|^\beta \brg{w}{w'} \quad \forall c\in \R. \label{equ:homo}
\end{align}
\end{subequations}
For a detailed discussion of potentials under Assumption~\ref{thm:assump-potential}, we refer the readers to Appendix~\ref{sec:potential-norm}.

We note that this assumption is quite general and covers several previously studied cases of mirror descent:
\begin{list}{{\tiny $\blacksquare$}}{\leftmargin=1.5em}
\setlength{\itemsep}{-1pt}
    \item When $\psi = \frac{1}{2} \norm{\cdot}_2^2$, we recover gradient descent.
    \item When $\psi = \frac{1}{2} \norm{\cdot}_q^2$, we have the so-called \textit{$p$-norm algorithm} (where $1/p + 1/q = 1$)~\citep{grove2001general, gentile2003robustness}.
    \item The potential $\mir(\cdot) = \frac{1}{p} \norm{\cdot}_p^p$ for $p > 1$ is particularly of practical interest because the mirror map $\nabla \psi$ updates becomes \textit{separable} in coordinates and thus can be implemented \textit{coordinate-wise} independent of other coordinates:
    \begin{align} \tag{\algname}\label{mdpp}
    \forall j \in [d],\quad \begin{cases}
        w_{t+1}[j] \leftarrow \left| w_t^+[j] \right|^{\frac{1}{p-1}} \cdot \sgn\left( w_t^+[j]\right)\\
        w_t^+[j]:= |w_t[j]|^{p-1}\sgn(w_t[j]) - \eta \nabla L(w_t)[j]
        \end{cases}
    \end{align}
    In comparison, the $p$-norm algorithm is not coordinate-wise separable since it requires computing $\norm{w_t}_q$ at each step (see, e.g., \citep[eq. (1)]{gentile2003robustness}).
    This potential $\mir(\cdot) = \frac{1}{p} \norm{\cdot}_p^p$ was first considered by~\citep{azizan2021stochastic} in the case of regression in which the loss function has unique minimizer.
   In an earlier version of this work \citep{sun2022mirror}, we analyzed this potential in the case of classification with strictly monotone losses and named mirror descent with this potential as \textit{$p$-norm GD}, or \algname in short, because it naturally generalizes gradient descent to $\ell_p$-norms.
\end{list}

\begin{remark}
In this paper, we do not consider the case where the potential is the negative entropy function (which recovers the multiplicative weights or the hedge algorithm from the online learning literature) because this potential function requires all the weights to be positive, which would be too restrictive in our problem setting.
\end{remark}

\begin{remark} \label{rmk:sep}
It is worth noting that the steepest descent algorithm, which follows the update rule
\[w_{t+1} = \argmin_w \left\{\frac{1}{2\eta}\norm{w}^2 + \inp{\nabla L(w_t)}{w}\right\}\]
for a general norm $\norm{\cdot}$, is not an instance of mirror descent since $\frac{1}{2}\norm{\cdot}^2$ is in general not a Bregman divergence.
Further, this update rule does not have a closed-form solution and thus cannot be solved efficiently.
In this section, we will see that mirror descent with potential $\psi(\cdot) = \norm{\cdot}^\beta$ for any $\beta > 1$ will induce the same implicit bias as steepest descent for strictly monotone losses.
\end{remark}

Finally, recall that in Section~\ref{sec:setting}, we discussed our assumptions on the dataset.
With the definition of $\norm{\cdot}_\psi$, we can now state these assumptions more precisely.
\begin{assumption}
\label{thm:assump-boundedness}
    The dataset $\{(x_i, y_i)\}_{i=1}^n \subset \RR^d \times \{\pm 1\}$ is linearly separable so there exists  $w^* \in \RR^d$ s.t. $\sgn(\inp{w^*}{x_i}) = y_i$ for all $i\in[n]$.
    Let $\norm{\cdot}_{\psi, *}$ be the dual norm to $\norm{\cdot}_\psi$.
    The input variables $x_i$ are bounded that there exists constant $C > 0$ where $\max(\norm{x_i}_2, \norm{x_i}_{\psi, *}) \le C$ for all $i \in [n]$.
\end{assumption}

\subsection{Main theoretical results}
\label{sec:primal-bias-result}

We extend Theorems~\ref{thm:gd-maxmargin}-\ref{thm:pgd-regdir} to the setting of mirror descent with potential functions satisfying Assumption~\ref{thm:assump-potential}.
Building upon the analysis in our previous work~\citep{sun2022mirror}, we resolve several major obstacles in the analysis of implicit regularization of linear classification where we apply a general class of MD to strictly monotone loss functions:
\begin{list}{{\tiny $\blacksquare$}}{\leftmargin=1.5em}
  \setlength{\itemsep}{-1pt}
    \item   We approach the classification setting with strictly monotone loss by considering the limit of a sequence of constrained optimization problems. Then each constrained problem has a unique and finite minimizer, and these solutions trace out the regularization path. Our analysis builds upon the techniques employed by~\cite{ji2020gradient}. In addition to generalizing regularized direction to a more general geometry, we derive stronger justification for using regularized direction by connecting the analysis of implicit bias under different types of loss functions.
    \item Our argument addresses the concern from \cite{gunasekar2018characterizing} that the implicit bias of regression and classification problems are ``fundamentally different.'' In \cite{gunasekar2018characterizing}, it was noted that gradient descent's implicit bias is dependent on the initialization for regression problems, but not for classification problems.
    In this section, we argue that it is sufficient to reframe the classification setting as a sequence of carefully chosen regression problems.\footnote{Recall that, in Section~\ref{sec:setting}, we define ``regression problem'' as the case where the loss function has a unique and attainable minimizer.} Then, we find that the dependence on the initialization vanishes after taking the limit.
    \item On a more technical note, one challenge in analyzing mirror descent lies in the cross terms of the form $\inp{\nabla\mir(w)}{w'}$, which lack direct geometric interpretations.
    We demonstrate that under Assumption~\ref{thm:assump-potential}, these terms can be simplified nicely and still induce a variety of desired geometric properties on the implicit bias.
\end{list}

\noindent We start our discussion with the following main result.
\begin{theorem}
\label{thm:primal-bias}
For a separable linear classification problem, if the regularized direction $\reg{\psi}$ exists, then with sufficiently small step size, the iterates of \ref{equ:md} with a potential function $\psi$ satisfying Assumption~\ref{thm:assump-potential} converge to $\reg{\psi}$ in direction:
\begin{equation}
    \lim_{t\to\infty} \frac{w_t}{\norm{w_t}_\psi} = \reg{\psi}.
\end{equation}
\end{theorem}

Next, we shall introduce the key ideas behind this theorem.
We motivate our use of regularized direction by first considering the regression setting and then highlighting the additional challenges we must overcome in the classification setting. 
For over-parameterized regression problems, there exists some weight vector $w$ such that $L(w) = 0$.
Then, we can apply Lemma~\ref{thm:key-iden} to get
\[D_\psi(w, w_{t}) = D_\psi(w, w_{t+1}) + D_{\psi - \eta L}(w_{t+1}, w_t) + \eta D_{L}(w, w_t) + \eta (L(w_{t+1}) - L(w))\]
By our choice that $L(w) = 0$, the equation above implies that $D_\psi(w, w_{t}) \ge D_\psi(w, w_{t+1})$ for sufficiently small step-size $\eta$.
This can be interpreted as \ref{equ:md} having a ``decreasing potential'' of the from $\brg{w}{\cdot}$ during each step.
Using this property, \citet{azizan2018stochastic} establishes the implicit bias results of mirror descent in the regression setting.

We now return to the classification case and note that for strictly monotone losses, there are no attainable minimizers.
Therefore, the choice of $w$ we make in the regression case would not be valid for classification problems.
Instead, we relax the ``decreasing potential'' property to hold for only a single step so that, at each time $t$, we choose a reference vector $\hat{w}_t$ satisfying $L(\hat{w}_t) \le L(w_{t+1})$.
The following result, which is a generalization of  \citep[Lemma 9]{ji2020gradient}, shows that we can let $\hat{w}_t$ be some scalar multiples of the regularized direction.

\begin{lemma}
\label{thm:approx-reg-dir-loss}
If the regularized direction $\reg{\psi}$ exists, then $\forall \alpha > 0$, there exists $r_\alpha$ such that for any $w$ with $\norm{w}_\psi > r_\alpha$, we have $L((1+\alpha)\norm{w}_\psi \reg{\psi}) \le L(w)$. 
\end{lemma}

We note that, due to Lemma~\ref{thm:to-infinity}, the condition in Lemma~\ref{thm:approx-reg-dir-loss} is met for any sufficiently large time $t$.
Then, at time $t$, we can pick Lemma~\ref{thm:key-iden}'s reference vector to be a ``moving target'' $c_t \reg{\psi}$ so that $L(c_t \reg{\psi}) \le L(w_{t+1})$.
Recall from the definition of the regularized direction, each choice of $c_t \reg{\psi}$ approximately corresponds to the solution of $\argmin_{\norm{w}_\psi \le c_t} L(w)$.
Hence, intuitively speaking, our analysis converts the classification problem into a sequence of regression problems by constructing a constrained optimization problem at each update step.
Let us formalize this idea. We begin with the following inequality:
\begin{align} \label{ineq:1}
\brg{c_t\reg{\psi}}{w_{t+1}} \le \brg{c_t\reg{\psi}}{w_t} - \eta L(w_{t+1}) + \eta L(w_t),
\end{align}
where $c_t$ is taken to be $\approx \norm{w_t}_\psi$.\footnote{To be more precise, we want $c_t = (1+\alpha) \norm{w_t}_p$; and reason behind this choice is self-evident after we present Corollary~\ref{thm:cross-term}.}

Now we modify \eqref{ineq:1} so that it can telescope over different iterations.
One way is to add $\brg{c_{t+1}\reg{\psi}}{w_{t+1}}$ on both sides of \eqref{ineq:1} and move $ \brg{c_t\reg{\psi}}{w_{t+1}}$ to the right-hand side as follows:
\begin{align*}
    &\brg{c_{t+1}\reg{\psi}}{w_{t+1}} \\
    \le{}& \brg{c_t\reg{\psi}}{w_t} - \eta L(w_{t+1}) + \eta L(w_t) + \brg{c_{t+1}\reg{\psi}}{w_{t+1}} - \brg{c_t\reg{\psi}}{w_{t+1}} \\
    ={}& \brg{c_t\reg{\psi}}{w_t} - \eta L(w_{t+1}) + \eta L(w_t) + \psi(c_{t+1} \reg{\psi}) - \psi(c_t \reg{\psi})   - \inp{\nabla\psi(w_{t+1})}{(c_{t+1} - c_t) \reg{\psi}}
\end{align*}
Summing over $t = 0, \dots, T-1$ gives us
\begin{equation}
\label{equ:tele-sum}
\begin{aligned}
    \brg{c_T\reg{\psi}}{w_T}
    &\le \brg{c_0\reg{\psi}}{w_0} - \eta L(w_T) + \eta L(w_0) + \psi(c_{T} \reg{\psi}) - \psi(c_0 \reg{\psi}) \\
    &\hspace{6em}- \sum_{t=0}^{T-1}\inp{\nabla\psi(w_{t+1})}{(c_{t+1} - c_t) \reg{\psi}}
\end{aligned}
\end{equation}

If we show that right-hand side of \eqref{equ:tele-sum} is bounded, then exploiting the homogeneity of $\psi$ (Assumption~\ref{thm:assump-potential}) and dividing both sides by $c_T^\beta$ would yield that $D_\psi(\reg{\psi}, w_T / \norm{w_T}_\psi) \to 0$ as $T \to \infty$.
Therefore, after normalization, all terms in \eqref{equ:tele-sum} dependent on the initialization $w_0$ would vanish.
We note that $L(w_T) \in O(1)$ due to Lemma~\ref{thm:decreasing-lose} and $\psi(c_T \reg{\psi}) = c_T^\beta$ because $\psi$ is homogeneous.
So, we turn our attention to the final term $\sum_t \inp{\nabla\psi(w_{t+1})}{(c_{t+1} - c_t) \reg{\psi}}$.
We first only consider the product $\inp{\nabla\psi(w_{t+1})}{\reg{\psi}}$.
Invoking the \ref{equ:md} update rule, we get:
\begin{align*}
 \inp{\nabla\psi(w_{t+1}) - \nabla\psi(w_{t})}{\reg{\psi}} = \inp{-\eta\nabla L(w_t)}{\reg{\psi}}.
\end{align*}

The inner product $\inp{\nabla L(w_t)}{\reg{\psi}}$ on the right-hand side is difficult to compute directly.
So, we exploit the property of $\reg{\psi}$ to bound this quantity.
We recall from the definition of regularized direction that $\reg{\psi}$ is the direction along which we achieve the smallest loss and hence $\nabla L(w_t)$ must point away from $\reg{\psi}$, i.e., it must be that $\inp{\nabla L(w_t)}{\reg{\psi}} \lesssim \inp{\nabla L(w_t)}{w_t}$ (up to lower-order terms).
The following result formalizes this intuition.
\begin{corollary}
\label{thm:cross-term}
For $w$ so that $\norm{w}_\psi > r_\alpha$, we have $\inp{\nabla L(w)}{w} \ge (1+\alpha)\norm{w}_\psi\inp{\nabla L(w)}{\reg{\psi}}$.  
\end{corollary}
\begin{proof}
This follows from the following inequality
\[\inp{\nabla L(w)}{w - (1+\alpha)\norm{w}\reg{\psi})} \ge L(w) - L((1+\alpha)\norm{w}\reg{\psi}) \ge 0\,,
\]
where the first inequality is due to the convexity of $L$ and the second inequality is due to Lemma~\ref{thm:approx-reg-dir-loss}.
\end{proof}

Therefore, we are left with
\begin{align*}
\inp{\nabla\psi(w_{t+1}) - \nabla\psi(w_{t})}{\reg{\psi}} \gtrsim \inp{-\eta\nabla L(w_t)}{w_t}, 
\end{align*}
Under our choice of homogeneous potential as detailed in Assumption~\ref{thm:assump-potential}, one can invoke Lemma~\ref{thm:key-iden} and Equation \eqref{equ:breg-homo} to lower bound the quantity $\inp{-\eta\nabla L(w_t)}{w_t}$ in terms of $\psi(w_{t+1})$ and $\psi(w_t)$, and this step is detailed in Lemma~\ref{thm:cross-term-diff} in Appendix~\ref{sec:cross-term-diff}.
It follows that the quantity $\inp{\nabla\psi(w_{t+1})}{\reg{\psi}}$ can be bounded with a telescoping sum, where we can show that $\inp{\nabla\psi(w_{t+1})}{\reg{\psi}} \in \Omega(\norm{w_t}_\psi^{\beta-1})$.
Then, the final term in \eqref{equ:tele-sum} turns into another telescoping sum in $c_t$'s and $\norm{w_t}$'s.
Unwinding the above process, we show that 
\[\sum_{t=0}^{T-1}\inp{\nabla\psi(w_{t+1})}{(c_{t+1} - c_t) \reg{\psi}} \in \Omega(\norm{w_T}_\psi^\beta),\]
which cancels out the quantity $\psi(c_T \reg{\psi})$ in \eqref{equ:tele-sum}.
After normalization, it must be the case that $\brg{\reg{\psi}}{w_t / \norm{w_t}_\psi}$ converges to zero in the limit as $t \to \infty$.
Putting everything together, we obtain Theorem~\ref{thm:primal-bias}.
A formal proof of this result can be found in Appendix~\ref{sec:proof-primal-bias}.

The final missing piece to Theorem~\ref{thm:primal-bias} would be the existence of the generalized regularized direction.
In general, finding the limit direction $\reg{\psi}$ would be difficult.
Fortunately, we can sometimes appeal to the generalized max-margin direction, which can be computed by solving a convex optimization problem.
The following result is a generalization of \citep[Proposition 10]{ji2020gradient} and shows that for common losses in classification, the regularized direction and the max-margin direction are the same, hence proving the existence of the former.

\begin{proposition}
\label{thm:reg-max-dir}
    If we have a loss with exponential tail, e.g. $\lim_{z\to\infty} \ell(z) e^{az} = b$, then under a strictly convex potential $\psi$, the generalized regularized direction $\reg{\psi}$ exists and it is equal to the generalized max-margin direction $\mmd{\psi}$.
\end{proposition}
The proof of this result can be found in Appendix~\ref{sec:proof-reg-max-dir}.
Note that many commonly used losses in classification, e.g., logistic loss, have exponential tails.

\subsection{Asymptotic convergence rate}
\label{sec:asymp-result}

In this section, we characterize the rate of convergence for Theorem~\ref{thm:primal-bias}.
As an immediate consequence of the proof of Theorem~\ref{thm:primal-bias}, one can show the following result in the case of linearly separable data.
\begin{corollary}
    \label{thm:convg-rate}
    Under the same setting as Theorem~\ref{thm:primal-bias}, the iterates of \ref{equ:md} follows the rate of convergence:
    \[ \brg{\reg{\psi}}{\frac{w_t}{\norm{w_t}_\psi}} \in O\left(\norm{w_t}_\psi^{-(\beta-1)}\right).\]
\end{corollary}
To fully understand the convergence rate, we need to characterize the asymptotic behavior of $\norm{w_t}_\psi$.
In the following result, we quantify $\norm{w_t}_\psi$ for the exponential loss.
We note that similar conclusions can be drawn for other losses with an exponential tail, e.g. logistics loss.
For the sake of simplicity, our analysis in this section focuses on the exponential loss.
\\
Recall that from Assumption~\ref{thm:assump-boundedness}, $\max_i \norm{x_i}_{\psi, *} \le C$, and the max-margin direction $\mmd{\psi}$ satisfies $y_i \inp{x_i}{\mmd{\psi}} \ge \hat{\gamma}_\psi \, \forall i \in [n]$.
Then, we have the following bound on $\norm{w_t}_\psi$.

\begin{lemma}
    \label{thm:norm-rate}
    For exponential loss, the iterates of \ref{equ:md} satisfies $\norm{w_t}_\psi \in \Theta(\log t)$.
    In particular, we have
    \[\norm{w_t}_\psi \ge \frac{1}{C} (\log t - \beta \log\log t) + O(1) \text{ and } \limsup_{t\to\infty} \frac{\norm{w_t}_\psi}{\log t} \le \hat{\gamma}_\psi^{-1} \frac{\beta}{\beta-1}.\]
\end{lemma}
The proof of this lemma can be found in Appendix~\ref{sec:proof-asymp-result}.

It follows that mirror descent with homogeneous potential has a poly-logarithmic rate of convergence.
\begin{corollary}
    \label{thm:final-convg-rate}
    For exponential loss, the iterates of \ref{equ:md} have convergence rate 
    \[ \brg{\reg{\psi}}{\frac{w_t}{\norm{w_t}_\psi}} \in O\left(\frac{1}{\log^{\beta-1}(t)}\right). \]
\end{corollary}

As we previously discussed, the Bregman divergence $D_\psi(\cdot, \cdot)$ generalizes the Euclidean distance to the geometry induced by the potential function $\psi$.
Therefore, the quantity $\brg{\reg{\psi}}{w_t / \norm{w_t}_\psi}$ the most natural measure of the angle between the generalized regularized direction $\reg{\psi}$ and the MD iterates $w_t$.
For the case of gradient descent, our result recovers the rate derived in~\citep{soudry2018implicit}.

\subsection{Accelerated convergence with variable step size}
\label{sec:normalized-asymp}
The result in the previous section shows that \ref{equ:md} with a fixed step size converges to the maximum-margin solution with a poly-logarithmic rate.
In this section, we demonstrate an accelerated convergence rate with adaptive step sizes.
In particular, we consider a form of normalized mirror descent:
\[\nabla\psi(w_{t+1}) = \nabla\psi(w_t) - \eta_t \frac{\nabla L(w_t)}{\norm{\nabla L(w_t)}_{\psi, *}}.\]

Although the quantity $\norm{\nabla L(w)}_{\psi, *}$ may not be simple to compute, we claim that it differs from $L(w)$ by at most a constant.
Recall that $\mar{\psi} = \max_{\norm{w}_{\psi} = 1} \{\min_i y_i w^\top x_i\}$.
Hence,
\[\norm{\nabla L(w)}_{\psi, *} = \max_{\norm{v}_{\psi} = 1} \sum_{i=1}^n \exp(-y_i w^\top x_i) y_i v^\top x_i \ge \mar{\psi} \sum_{i=1}^n \exp(-y_i w^\top x_i) = \mar{\psi} L(w).\]
For an upper bound on $\norm{\nabla L(w)}_{\psi, *}$, we note that, by Assumption~\ref{thm:assump-boundedness}, $x_i$'s are bounded:
\[\norm{\nabla L(w)}_{\psi, *} = \norm{\sum_{i=1}^n \exp(-y_i w^\top x_i) y_i x_i}_{\psi, *} \le \sum_{i=1}^n \exp(-y_i w^\top x_i) \norm{x_i}_{\psi, *} \le C \cdot L(w).\]

From this observation, we will replace $\norm{\nabla L(w_t)}_{\psi, *}$ with the more readily available quantity $L(w_t)$.
By letting $\eta_t = \frac{\eta_0}{\sqrt{t+1}}$, we now consider the update rule
\begin{equation}\tag{\sf N-MD}
    \label{equ:normalized-md}
    \nabla\psi(w_{t+1}) = \nabla\psi(w_t) - \frac{\eta_0}{\sqrt{t+1}} \frac{\nabla L(w_t)}{L(w_t)},
\end{equation}
which is a direct mirror descent extension of the normalized gradient descent algorithm studied in~\citep{nacson2019convergence}.
In Appendix~\ref{sec:step-size}, we will discuss the choice of step sizes when we use the exponential loss.

To prove the rate of convergence, we follow a similar strategy as we did in the case of fixed step size.
First, we show that the magnitude of the iterates $w_t$ can be upper bounded.
\begin{lemma}
\label{thm:normalized-norm-upper}
    For exponential loss, the iterates of the normalized mirror descent \eqref{equ:normalized-md} satisfies 
    \[ \limsup_{t\to\infty} \frac{\norm{w_t}_\psi}{\sqrt{t}} \le \mar{\psi}^{-1} \frac{\beta}{\beta-1}.\]
\end{lemma}
And with this lemma, we can proceed to show a lower bound on the magnitude of $w_t$.
\begin{lemma}
\label{thm:normalized-norm-rate}
    For exponential loss, if the potential function satisfies Assumption~\ref{thm:assump-potential} with $\beta < 3$ and is continuously twice differentiable, and the initial loss is sufficiently small that $L(w_0) \le \frac{1}{2n}$, then the iterates of the normalized mirror descent \eqref{equ:normalized-md} satisfies $\norm{w_t}_\psi \in \Omega(t^{(3-\beta)/2 - \zeta})$ for any $\zeta > 0$.
\end{lemma}
These two lemmas together represent a normalized mirror descent counterpart of Lemma~\ref{thm:norm-rate}.
Then, we can apply Lemmas~\ref{thm:normalized-norm-upper} and~\ref{thm:normalized-norm-rate} to the same proof technique as Theorem~\ref{thm:primal-bias}.
Compared to the fixed step size case, the main difficulty here lies in that we no longer obtain a telescoping sum like \eqref{equ:tele-sum} when the step size is not constant.
We overcome this difficulty by controlling the terms that do not telescope are of lower order than $D_\psi(c_T \reg{\psi}, w_T)$ using the asymptotic growth rate on $\norm{w_t}_\psi$ for a certain range of $\beta$.
In particular, we show that such terms vanish after normalization.
Hence, the normalized mirror descent algorithm converges to the same implicit bias as \ref{equ:md}, but at a much faster rate.
\begin{theorem}
\label{thm:normalized-md-rate}
    For a separable linear classification problem with exponential loss, if the potential function satisfies Assumption~\ref{thm:assump-potential} with $1 < \beta < \frac{1}{2}(3+\sqrt{5})$ and is continuously twice differentiable, and the initial loss is sufficiently small that $L(w_0) \le \frac{1}{2n}$, then normalized mirror descent~\eqref{equ:normalized-md} converges to the generalized maximum-margin direction at a rate
    \[D_\psi\left(\mmd{\psi}, \frac{w_t}{\norm{w_t}_\psi}\right) \in O\left(t^{-(3\beta - 1 - \beta^2)/2 + \beta\zeta}\right),\]
    for any $\zeta > 0$.
\end{theorem}
In the case of gradient descent, the rate becomes $O(t^{-1/2 + 2\zeta})$.
We note that \cite{nacson2019convergence} presented their convergence rate for normalized gradient descent in terms of the difference in margin: $\mar{2} - \gamma(w_t / \norm{w_t}_2) \in O(\log t / \sqrt{t})$.
While their result is not directly comparable to our result here, it intuitively matches the rate we derived through Lemma~\ref{thm:normalized-norm-rate} and Theorem~\ref{thm:normalized-md-rate}.

We also note that our result requires a warm-up period to find a small initial loss $L(w_0) \le 1/(2n)$\footnote{This is not necessary when $\beta=2$, but our proof requires initial loss $L(w_0)$ be small for other values of $\beta$.}.
One way to achieve this is by warm-starting with  \ref{equ:md} and switching to \ref{equ:normalized-md} after the condition is met.
From what we had shown about the convergence rate of \ref{equ:md} with a fixed step size (see the proof of Lemma~\ref{thm:norm-rate}), the warm-up period is at most $O(n)$ steps.
In practice, we observe that such a warm-start scheme is not necessary.

The proofs for this section can be found in Appendix~\ref{sec:proof-normalized-asymp}.
For experiments illustrating faster convergence with normalized MD, please see Section~\ref{sec:normalized-exp}.
 
\section{Experiments}
\label{sec:experiments}

In this section, we perform various experiments to complement our theoretical results in Section~\ref{sec:main-result} and to illustrate the performance of MD in practical settings.
We will apply mirror descent (MD) with various homogeneous potentials of the form $\psi(\cdot) = \frac{1}{p} \norm{\cdot}_p^\beta$, so that, according to our results in Section~\ref{sec:main-result}, the iterates converge to the generalized maximum-margin solutions with respect to the $\ell_p$-norm.
As we previously discussed, when $\beta = p$, the MD update rule can efficiently computed because it is \textit{coordinate-wise separable} (see the update rule \eqref{mdpp}).
Hence, our experiments also highlight the efficacy of \algname, which is MD with potential function $\psi(\cdot) = \frac{1}{p} \norm{\cdot}_p^p$.
We naturally pick $p = 2$, which corresponds to gradient descent.
Because the mirror map $\nabla \psi$ would not be invertible at $p = 1$ and $\infty$, we choose $p = 1.1$ as a surrogate for $\ell_1$, and $p = 10$ as a surrogate for $\ell_\infty$.
We also consider $p = 1.5, 3, 6$ to interpolate these points.
In addition to applying \algname, we also present several experiments with more general potential functions where $\beta \neq p$.
This section will present a summary of our experimental results; the complete experimental setup and full results can be found in Appendices~\ref{sec:experiment-detail} and \ref{sec:add-experiments}.

\subsection{Linear classification}
\label{sec:linear-classifier}
\begin{figure}
    \centering
    \begin{subfigure}[b]{0.3\textwidth}
        \centering
        \includegraphics[width=\textwidth]{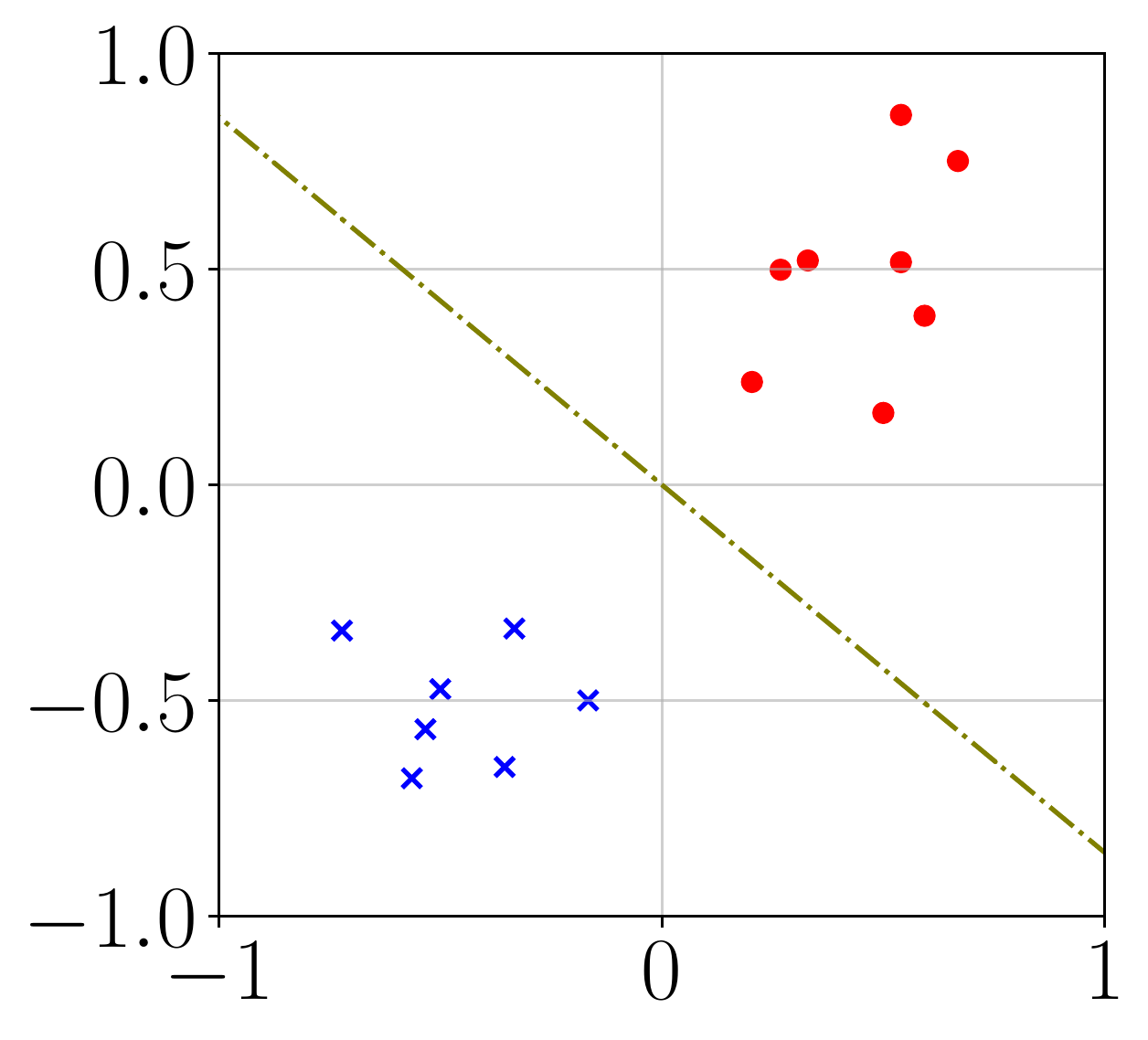}
    \end{subfigure}
    \begin{subfigure}[b]{0.32\textwidth}
        \includegraphics[width=\textwidth]{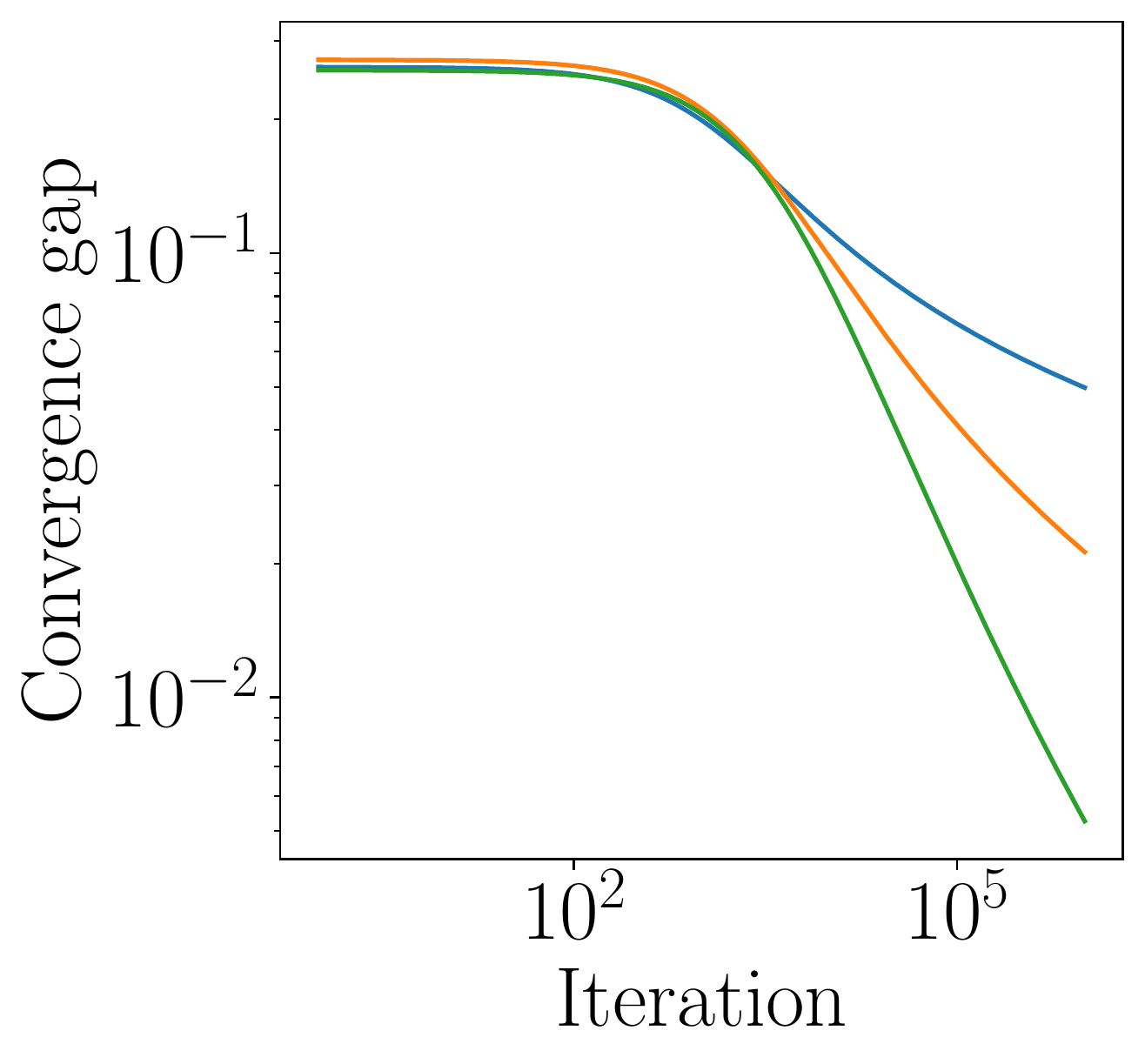}
    \end{subfigure}
    \begin{subfigure}[b]{0.3\textwidth}
        \includegraphics[width=\textwidth]{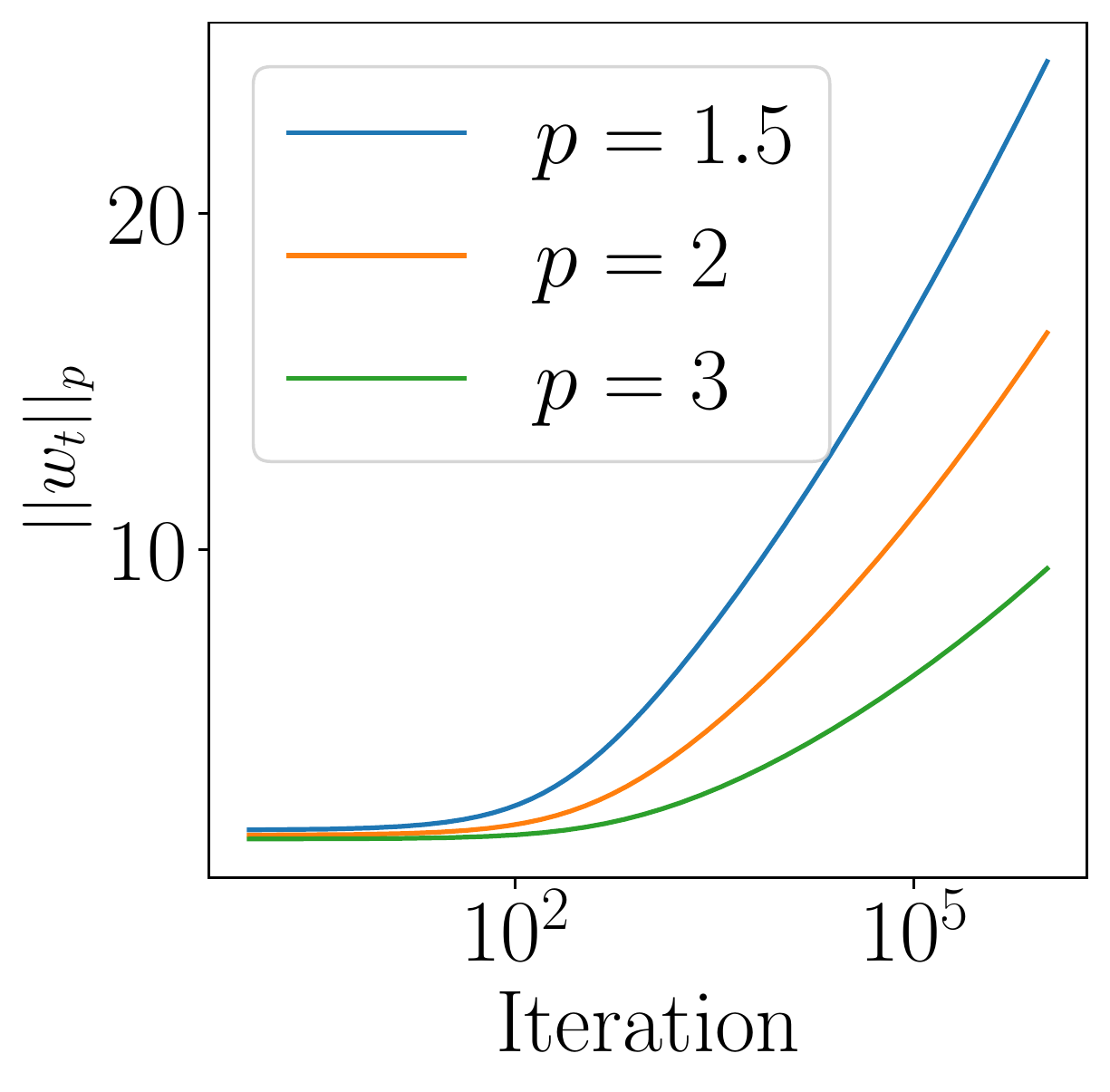}
    \end{subfigure}
    \caption{An example of \algname (MD with potential $\psi(\cdot) = \frac{1}{p} \norm{\cdot}_p^p$) on randomly generated data with exponential loss and $p = 1.5, 2, 3$. 
    \textbf{(1)} The left plot is a scatter plot of the data: \textcolor{blue}{$\times$}'s and \textcolor{red}{$\bullet$}'s denote the two different labels ($y_i = \pm 1$). The dotted line is the $\ell_2$ max-margin classifier. For clarity, other $\ell_p$ max-margin classifiers are omitted from the plot.
    \textbf{(2)} The middle plot shows the rate which the quantity $\brg{\reg{p}}{w_t / \norm{w_t}_t}$ converges to 0.
    \textbf{(3)} The right plot shows how fast the $p$-norm of $w_t$ grows.
    We can observe that the asymptotic behaviors of these plots are consistent with Corollary~\ref{thm:final-convg-rate}.
    }
    \label{fig:synthetic-data}
    \vspace{-0.5em}
\end{figure}

\begin{figure}
    \centering
    \begin{subfigure}[b]{0.45\textwidth}
        \centering
        \includegraphics[width=\textwidth]{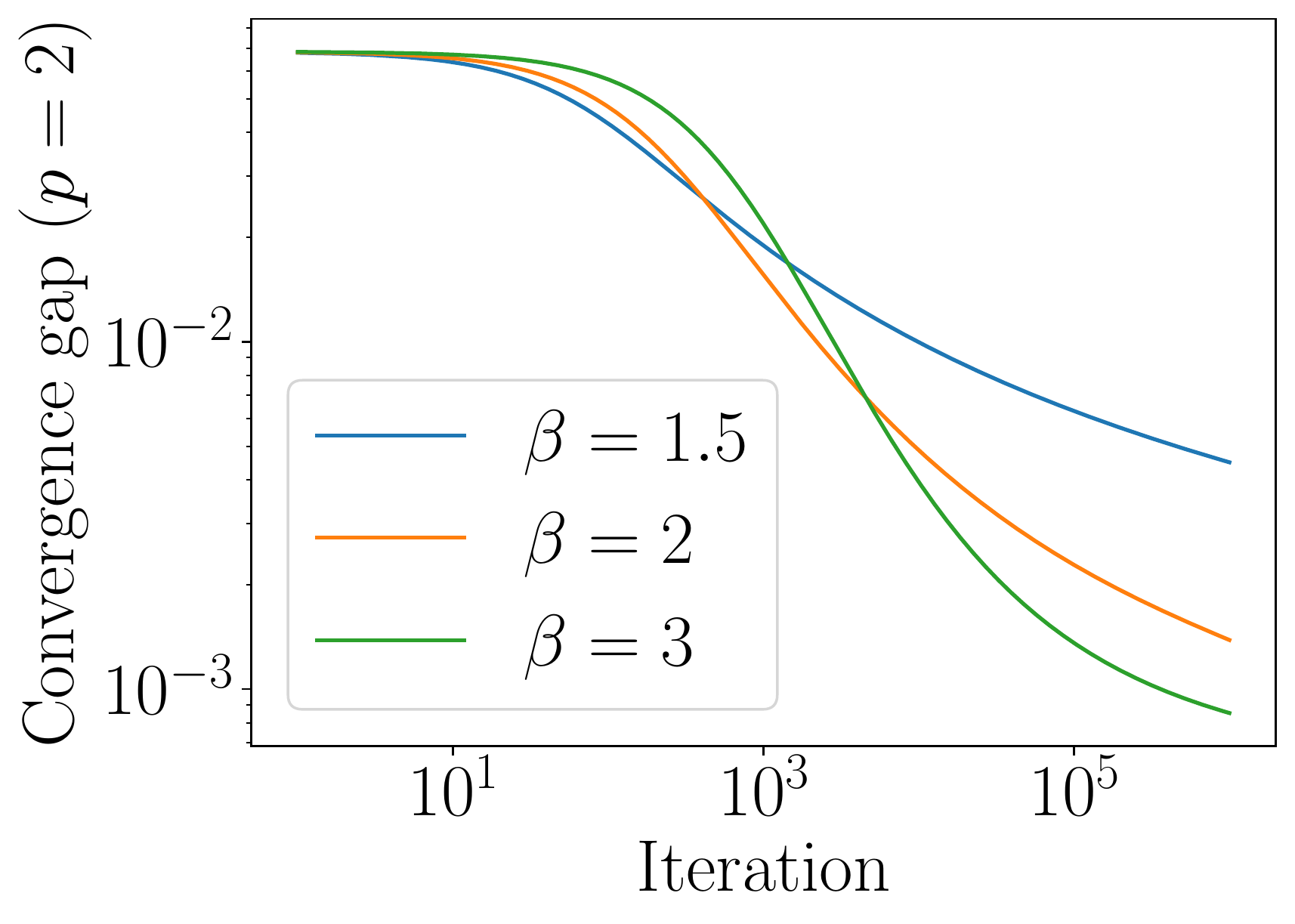}
    \end{subfigure}
    \begin{subfigure}[b]{0.45\textwidth}
        \includegraphics[width=\textwidth]{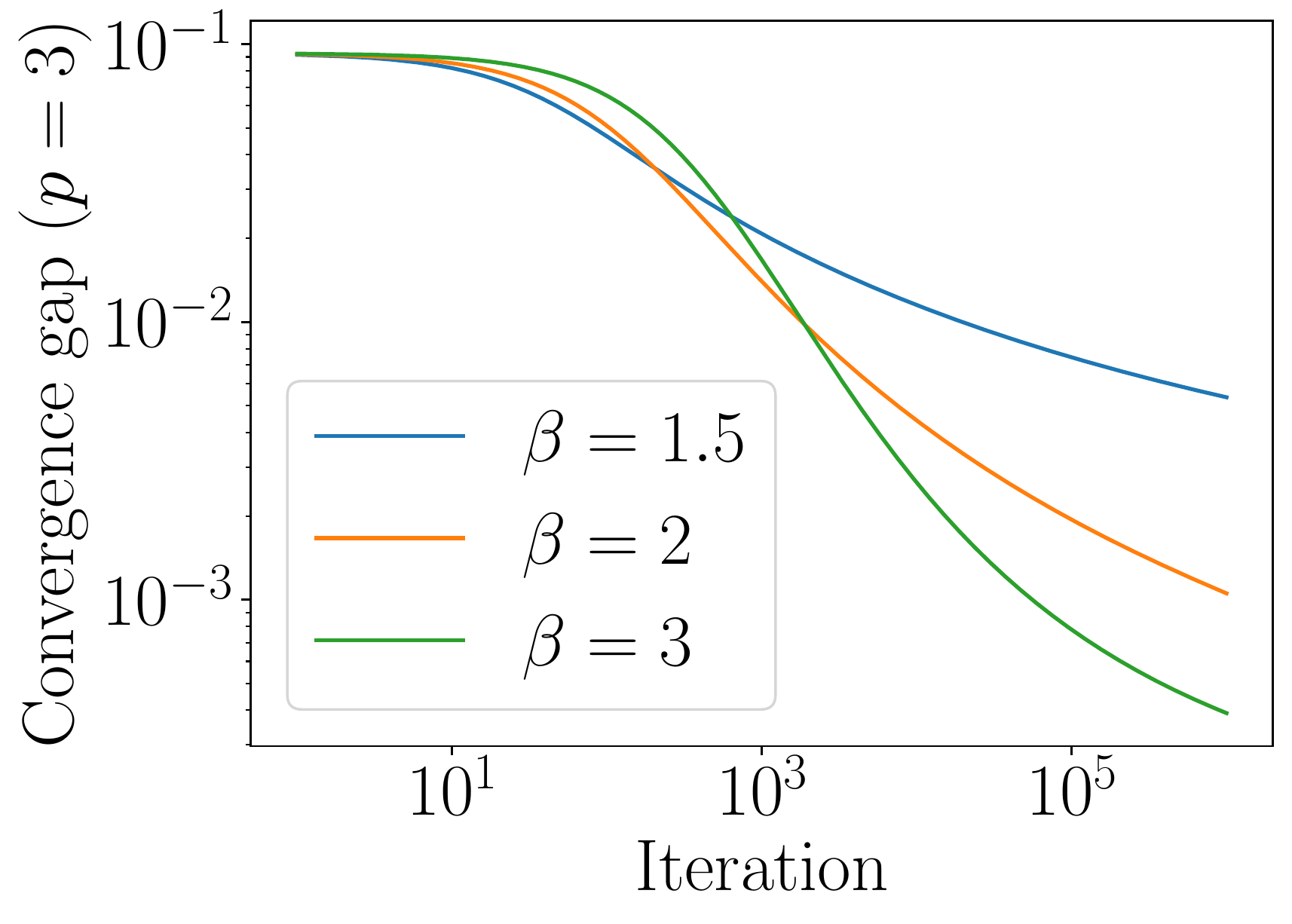}
    \end{subfigure}
    \caption{An example of MD with potential $\psi(\cdot) = \frac{1}{p} \norm{\cdot}_p^\beta$) on the same dataset as in Figure~\ref{fig:synthetic-data}.
    To verify the conclusion of Corollary~\ref{thm:final-convg-rate}, we plot the quantity $\brg{\reg{p}}{w_t / \norm{w_t}_t}$ for $p = 2$ (\textbf{left} figure) and $p = 3$ (\textbf{right} figure). We see that the rate of convergence is faster for higher values of the exponent $\beta$, which is consistent with Corollary~\ref{thm:final-convg-rate}.
    }
    \label{fig:synthetic-data-md}
    \vspace{-0.5em}
\end{figure}

\paragraph{Visualization of the convergence of MD.}
To visualize the results of Theorem~\ref{thm:primal-bias} and Corollary~\ref{thm:final-convg-rate}, we randomly generated a linearly separable set of 15 points in $\R^2$.
We then employ MD on this dataset with exponential loss $\ell(z) = \exp(-z)$ and one million iterations at a fixed step size $\eta = 10^{-3}$.
For the experiments involving \algname, we pick $p = 1.5, 2, 3$.
And for experiments on more general MD potential, we consider $p = 2, 3$ and $\beta = 1.5, 2, 3$.

In the illustrations of Figure~\ref{fig:synthetic-data}, the mirror descent iterates $w_t$ have unbounded norm and converge in direction to $\mmd{p}$. 
These results are consistent with Lemma~\ref{thm:to-infinity} and with Theorem~\ref{thm:primal-bias}.
Moreover, as predicted by Corollary~\ref{thm:final-convg-rate}, the exact rate of convergence for $\brg{\mmd{p}}{w_t / \norm{w_t}_t}$ is poly-logarithmic with respect to the number of iterations.
And in the third plot of Figure~\ref{fig:synthetic-data}, the norm of the iterates $w_t$ grows at a logarithmic rate, which is the same as the prediction by Lemma~\ref{thm:norm-rate}.

In Corollary~\ref{thm:final-convg-rate}, the convergence rate of MD is dependent on the homogeneity parameter $\beta$ of the potential function, where larger exponent $\beta$ leads to faster convergence.
We see that this is consistent with our observation in the second plot of Figure~\ref{fig:synthetic-data} and Figure~\ref{fig:synthetic-data-md}.
In the second plot of Figure~\ref{fig:synthetic-data}, we see that \algname enjoys faster convergence for larger $p$, and in Figure~\ref{fig:synthetic-data-md}, we see that for the same value of $p$, larger exponent $\beta$ led to faster convergence to the same generalized maximum-margin direction with respect to $\ell_p$.
 
\paragraph{Implicit bias of MD in linear classification.}
We now verify the conclusions of Theorem~\ref{thm:primal-bias}.
To this end, we recall that $\mmd{p}$ is parallel to the generalized SVM solution $\argmin_w \{\norm{w}_p : \gamma(w) \ge 1\}$.
Hence, we can exploit the linearity and rescale any classifier so that its margin is equal to $1$.
If the prediction of Theorem~\ref{thm:primal-bias} holds, then for each fixed value of $p$ and independent of the value of $\beta$, the classifier generated by MD with a potential $\psi(\cdot) = \frac{1}{p} \norm{\cdot}_p^\beta$ should have the smallest $\ell_p$-norm after rescaling.

To ensure that $\mmd{p}$ are sufficiently different for different values of $p$, we simulate an over-parameterized setting by randomly selecting 15 points in $\RR^{100}$.
We used a fixed step size of $\eta = 10^{-4}$ and ran one million iterations for different $p$'s.

\begin{table}
\begin{floatrow}[2]
\ttabbox{
{\small
    \begin{tabular}{l| c|c|c|c}
    \hline
    & $\ell_{1.1}$ & $\ell_{2}$ & $\ell_{3}$& $\ell_{10}$ \\
    \hline\hline
    $p=1.1$ & \textbf{5.477} & 1.637 & 1.093 & 0.696 \\
    $p=2$ & 6.161 & \textbf{1.224} & 0.684 & 0.382 \\
    $p=3$ & 7.299 & 1.296 & \textbf{0.667} & 0.309 \\
    $p=10$ & 9.032 & 1.515 & 0.740 & \textbf{0.280} \\
    \hline
    \end{tabular}
}
}
{
    \caption{Size of the linear classifiers generated by \algname (after rescaling) in $\ell_{1.1}, \ell_2, \ell_3$ and $\ell_{10}$ norms.}
    \label{tab:linear-bias}
}
\ttabbox{
{\small
    \begin{tabular}{l| c|c|c|c}
    \hline
    & $\ell_{1.1}$ & $\ell_{2}$ & $\ell_{3}$& $\ell_{10}$ \\
    \hline\hline
    $p=1.1$ & \textbf{5.136} & 1.864 & 1.276 & 0.795 \\
    $p=2$ & 6.161 & \textbf{1.224} & 0.684 & 0.382 \\
    $p=3$ & 7.305 & 1.275 & \textbf{0.652} & 0.301 \\
    $p=10$ & 9.132 & 1.477 & 0.712 & \textbf{0.266} \\
    \hline
    \end{tabular}
}
}
{
    \caption{Size of the linear classifiers generated by MD with potential $\psi(\cdot) = \frac{1}{p} \norm{\cdot}_p^2$ (after rescaling) in $\ell_{1.1}, \ell_2, \ell_3$ and $\ell_{10}$ norms.}
    \label{tab:linear-bias-md}
}
\end{floatrow}
\end{table}

Tables~\ref{tab:linear-bias} and~\ref{tab:linear-bias-md} shows the results for $p = 1.1, 2, 3$ and 10; under each norm, we highlight the smallest classifier in \textbf{boldface}.
In Table~\ref{tab:linear-bias}, we applied \algname. And in Table~\ref{tab:linear-bias-md}, we applied MD with a fixed exponent $\beta = 2$, which is known as the \textit{$p$-norm algorithm} in the literature~\citep{grove2001general, gentile2003robustness}.
Among the four classifiers we presented, \algname with $p = 1.1$ has the smallest $\ell_{1.1}$-norm.
And similar conclusions hold for $p = 2, 3, 10$.
Although MD converges to $\mmd{p}$ at a very slow rate, we can observe a very strong implicit bias of \algname classifiers toward their respective $\ell_p$ geometry in a highly over-parameterized setting.
This suggests we should be able to take advantage of the implicit regularization in practice and at a moderate computational cost.
For a more complete result with additional values of $p$, we refer the readers to Appendix~\ref{sec:add-experiment-synthetic}.

\subsection{Experiments with normalized MD}
\label{sec:normalized-exp}

\begin{figure}
    \centering
    \begin{subfigure}[b]{0.3\textwidth}
        \centering
        \includegraphics[width=\textwidth]{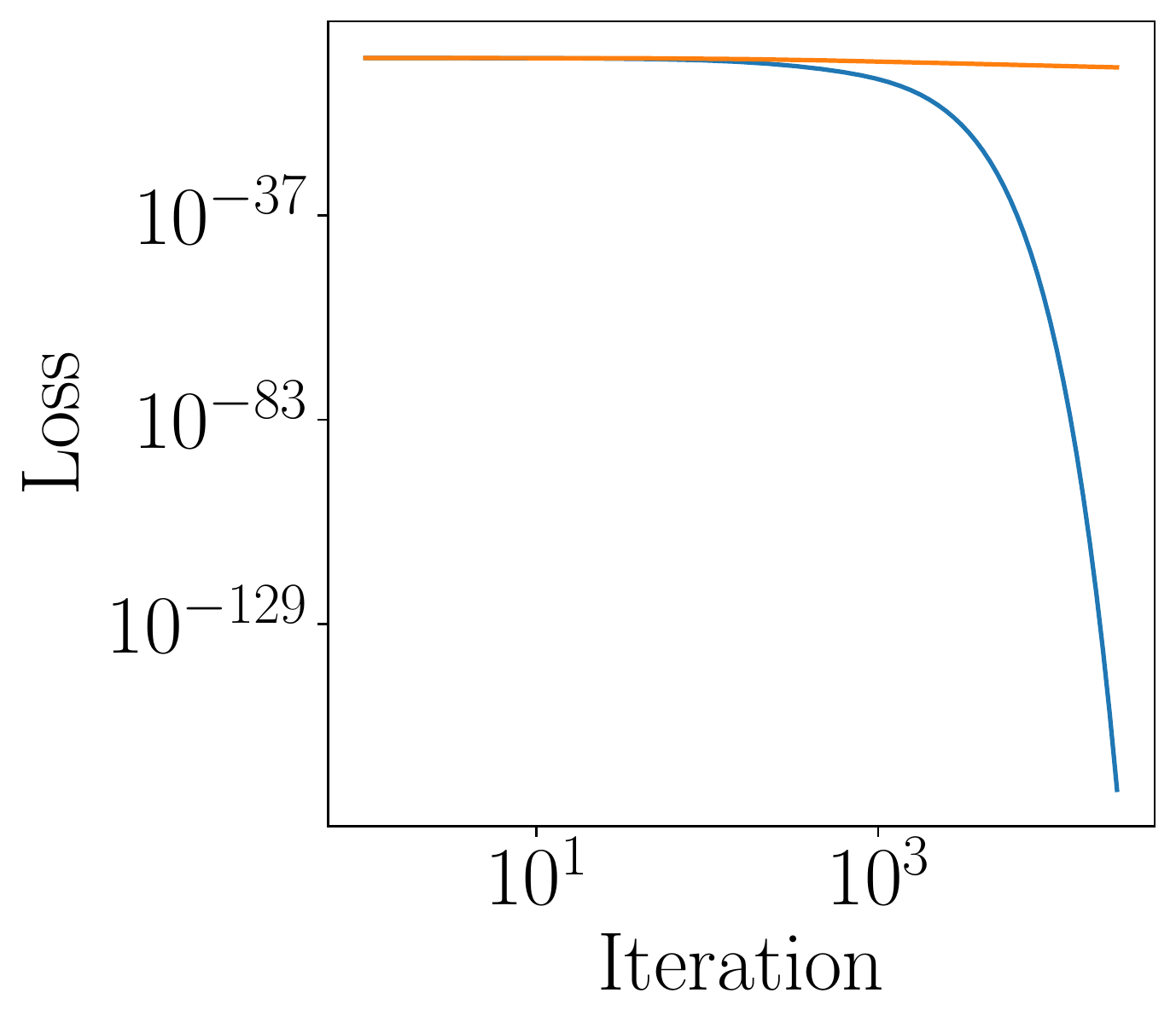}
    \end{subfigure}
    \begin{subfigure}[b]{0.3\textwidth}
        \includegraphics[width=\textwidth]{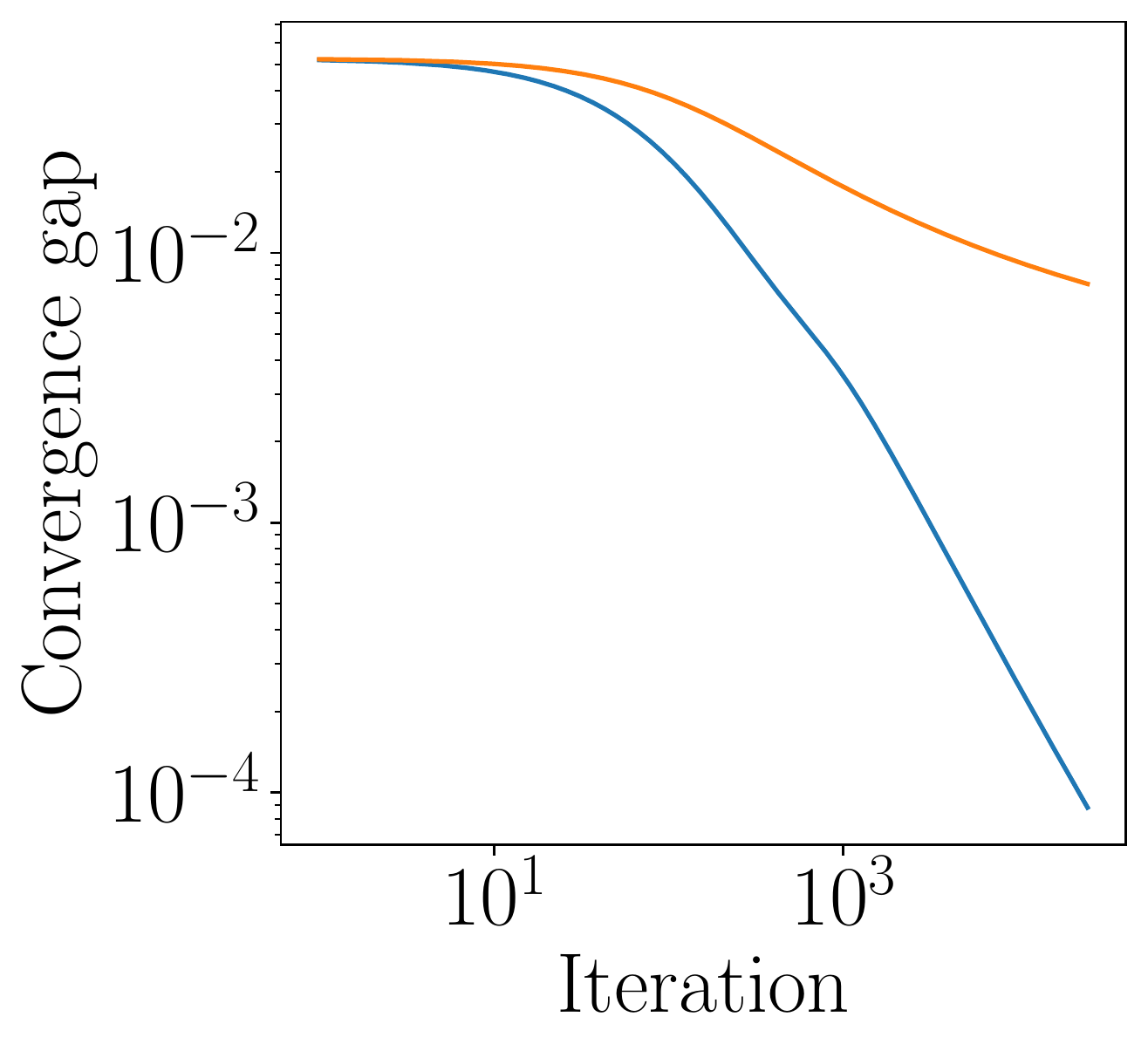}
    \end{subfigure}
    \begin{subfigure}[b]{0.3\textwidth}
        \includegraphics[width=\textwidth]{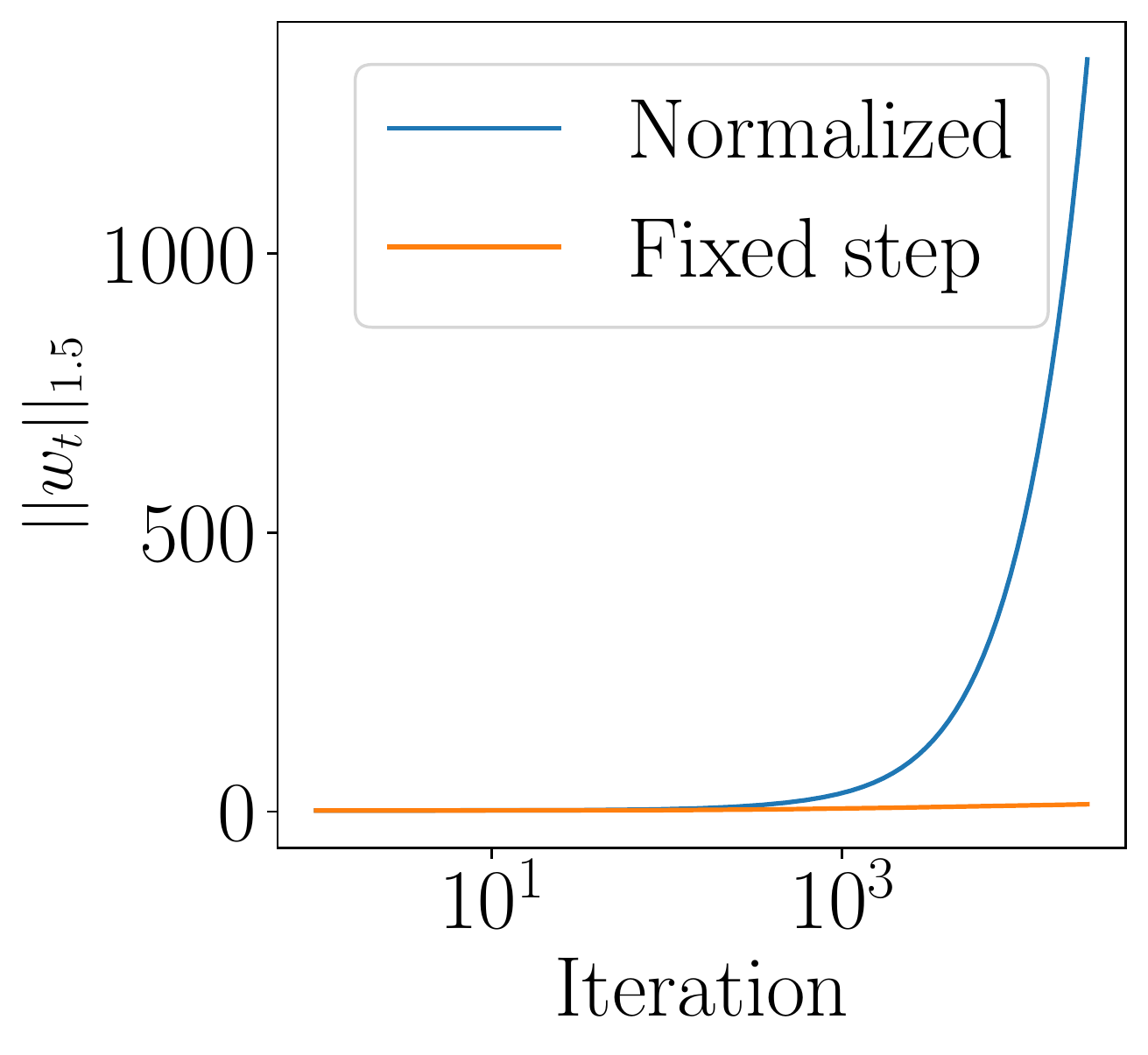}
    \end{subfigure}
    \caption{An example of \algname (MD with potential $\psi(\cdot) = \frac{1}{p} \norm{\cdot}_p^p$) and normalized \algname on randomly generated data with exponential loss and $p = 1.5$. 
    \textbf{(1)} The left plot is the empirical loss.
    \textbf{(2)} The middle plot shows the rate which the quantity $\brg{\reg{p}}{w_t / \norm{w_t}_t}$ converges to 0.
    \textbf{(3)} The right plot shows how fast the $p$-norm of $w_t$ grows.
    }
    \label{fig:synthetic-data-normalized}
    \vspace{-0.5em}
\end{figure}

We now demonstrate a faster rate of convergence with the normalized mirror descent update \eqref{equ:normalized-md} from Section~\ref{sec:normalized-asymp}.
Because the $\sqrt{t+1}$ term in the update rule \eqref{equ:normalized-md} would dominate at small $t$, we rescale the denominator by a factor of $\lambda > 0$ so that
\begin{equation}
    \label{equ:normalized-md-rescaled}
    \nabla\psi(w_{t+1}) = \nabla\psi(w_t) - \frac{\eta_0}{\sqrt{1 + \lambda t}} \frac{\nabla L(w_t)}{L(w_t)},
\end{equation}
Here, we present experimental results for \algname (along with its normalized counterpart) with $p = 1.5$.
Additional results for $p = 2$ (which is equivalent to gradient descent) and $p = 2.5$ are deferred to Appendix~\ref{sec:add-experiment-normalized}.

\paragraph{Linear classification.}
For a clean visualization of the rate of convergence, we again use the 2-dimensional synthetic dataset from Section~\ref{sec:linear-classifier}.
We reuse the same choice of hyper-parameters for mirror descent with a fixed step size.
As for normalized mirror descent update \eqref{equ:normalized-md-rescaled}, we use a base step size $\eta_0 = 10^{-3}$ and scale $\lambda = 10^{-3}$.
Since normalized MD converges much more quickly, we only run for 25000 iterations.

As shown in Figure~\ref{fig:synthetic-data-normalized}, the normalized \algname algorithm converges to the generalized maximum margin much faster than \algname with fixed step sizes, which is consistent with the predictions made by Corollary~\ref{thm:final-convg-rate} and Theorem~\ref{thm:normalized-md-rate}.
Also, we see that the empirical loss decreases very rapidly for normalized \algname, which means it is able to find classifiers with very large margins in just a few iterations.

\begin{figure}
    \centering
    \begin{subfigure}[b]{0.45\textwidth}
        \centering
        \includegraphics[width=\textwidth]{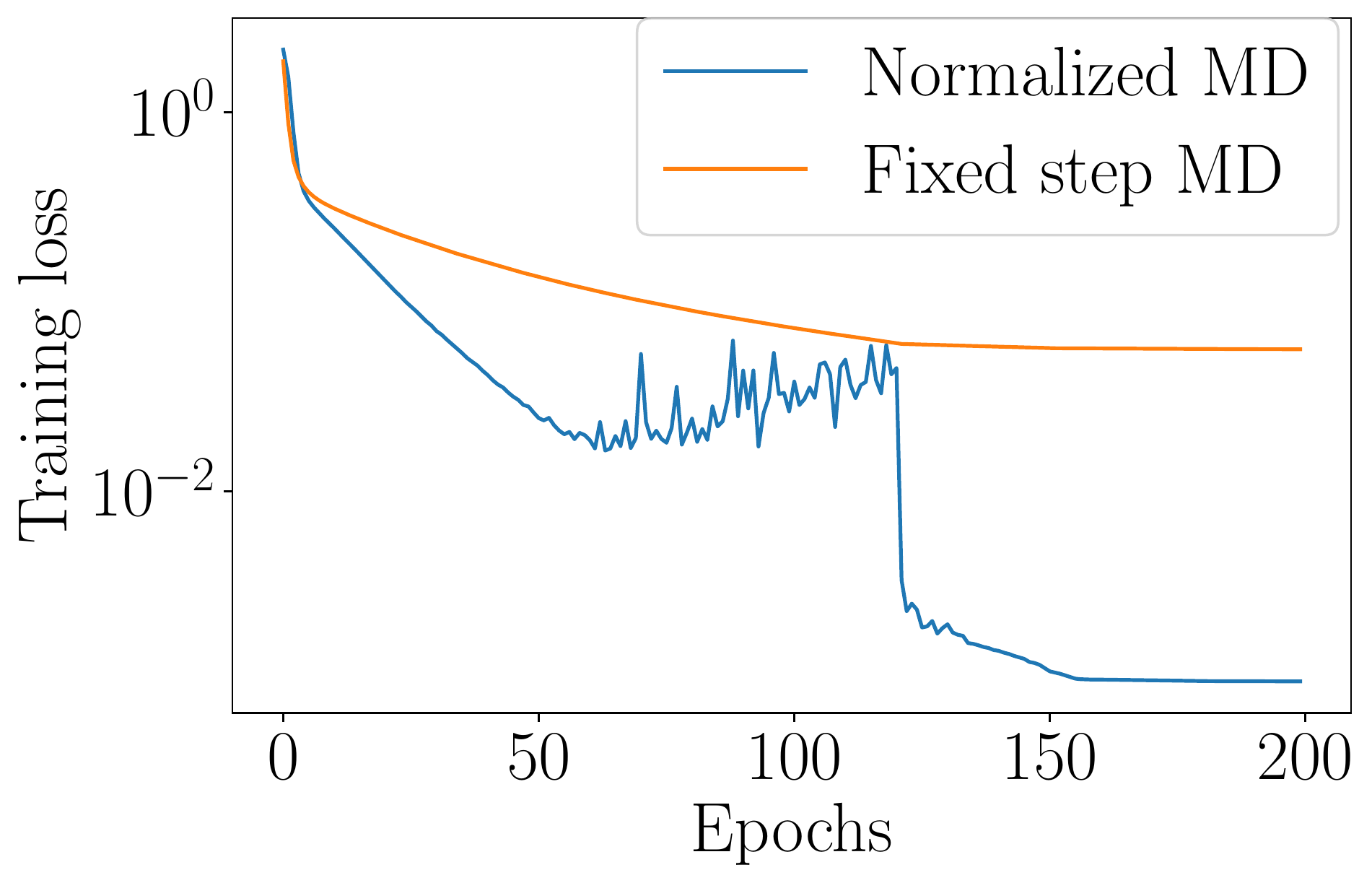}
    \end{subfigure}
    \begin{subfigure}[b]{0.45\textwidth}
        \includegraphics[width=\textwidth]{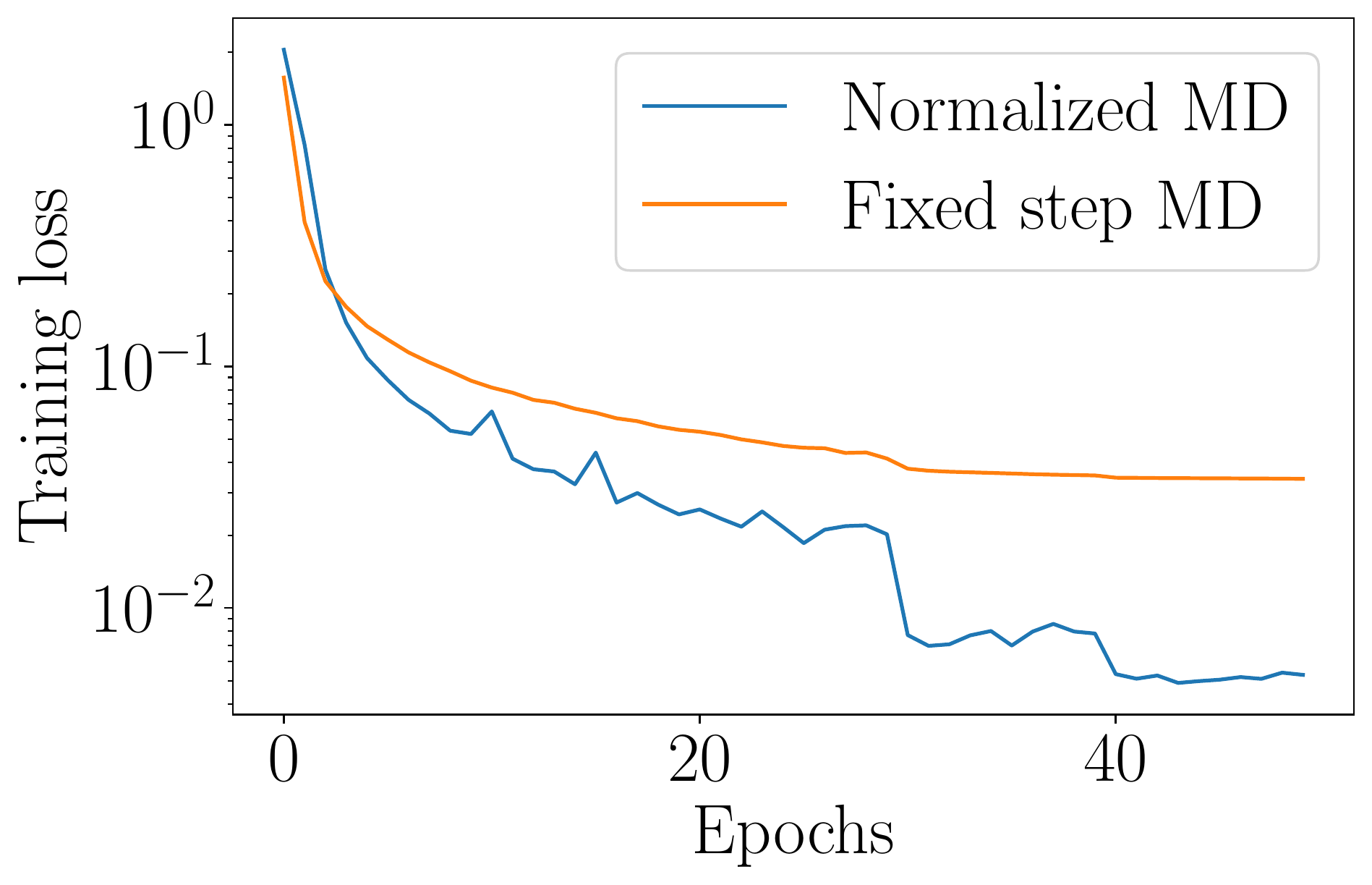}
    \end{subfigure}
    \caption{Training loss of \algname and normalized \algname on the MNIST dataset and $p = 1.5$. \\
    \textbf{(1)} The left plot involves a fully connected network.
    \textbf{(2)} The right plot involves a conv-net.
    }
    \label{fig:mnist-data}
    \vspace{-0.5em}
\end{figure}

\paragraph{Image classification on MNIST.}
For a more involved example, we apply \algname to the MNIST dataset~\citep{lecun1998gradient}.
For this task, we use two different architectures: 1) a 2-layer fully connected network with 300 hidden neurons and ReLU activation, and 2) a convolutional network with two convolution layers and batch-norm.
We train the fully connected network for 200 epochs and the convolution network for 50 epochs.
The detailed specification of this experiment can be found in Appendix~\ref{sec:experiment-detail}.

As shown in Figure~\ref{fig:mnist-data}, normalized \algname again achieves faster convergence.
Finally, we note that the lower loss achieved by normalized \algname translates to better generalization performance.
For the fully connected network, normalized \algname has a test accuracy of 98.37\%, whereas standard \algname has a test accuracy of 97.65\%.
And for the convolutional network, normalized \algname has a test accuracy of 99.19\%, whereas standard \algname has a test accuracy of 98.68\%.
These experiments demonstrate that normalized MD enjoys a faster convergence rate and its practical utility in performing better at test time.
 
\subsection{Deep neural networks} 
\label{sec:cifar}
Going beyond linear models, we now investigate \algname{} in deep-learning settings in its impact on the structure of the learned model and potential implications on the generalization performance.
As we had discussed in Section \ref{sec:main-result}, {\em the implementation of \algname{} is straightforward}; to illustrate the simplicity of implementation, we provide code snippets in
Appendix~\ref{sec:practicality}. 
Thus, we are able to effectively experiment with the behaviors of \algname in neural network training.
Specifically, we perform a set of experiments on the CIFAR-10 dataset~\citep{krizhevsky2009learning}.
We use the \textit{stochastic}\footnote{Instead of applying \algname update with the whole dataset, we use a randomly drawn mini-batch at each iteration.} version of \algname with different values of $p$.
We choose a variety of networks: \textsc{VGG}~\citep{simonyan2014very}, \textsc{ResNet}~\citep{he2016deep}, \textsc{MobileNet}~\citep{sandler2018mobilenetv2} and \textsc{RegNet}~\citep{radosavovic2020designing}.

\begin{figure}[!t]
    \centering
    \begin{subfigure}[b]{0.48\textwidth}
        \centering
        \includegraphics[width=\textwidth]{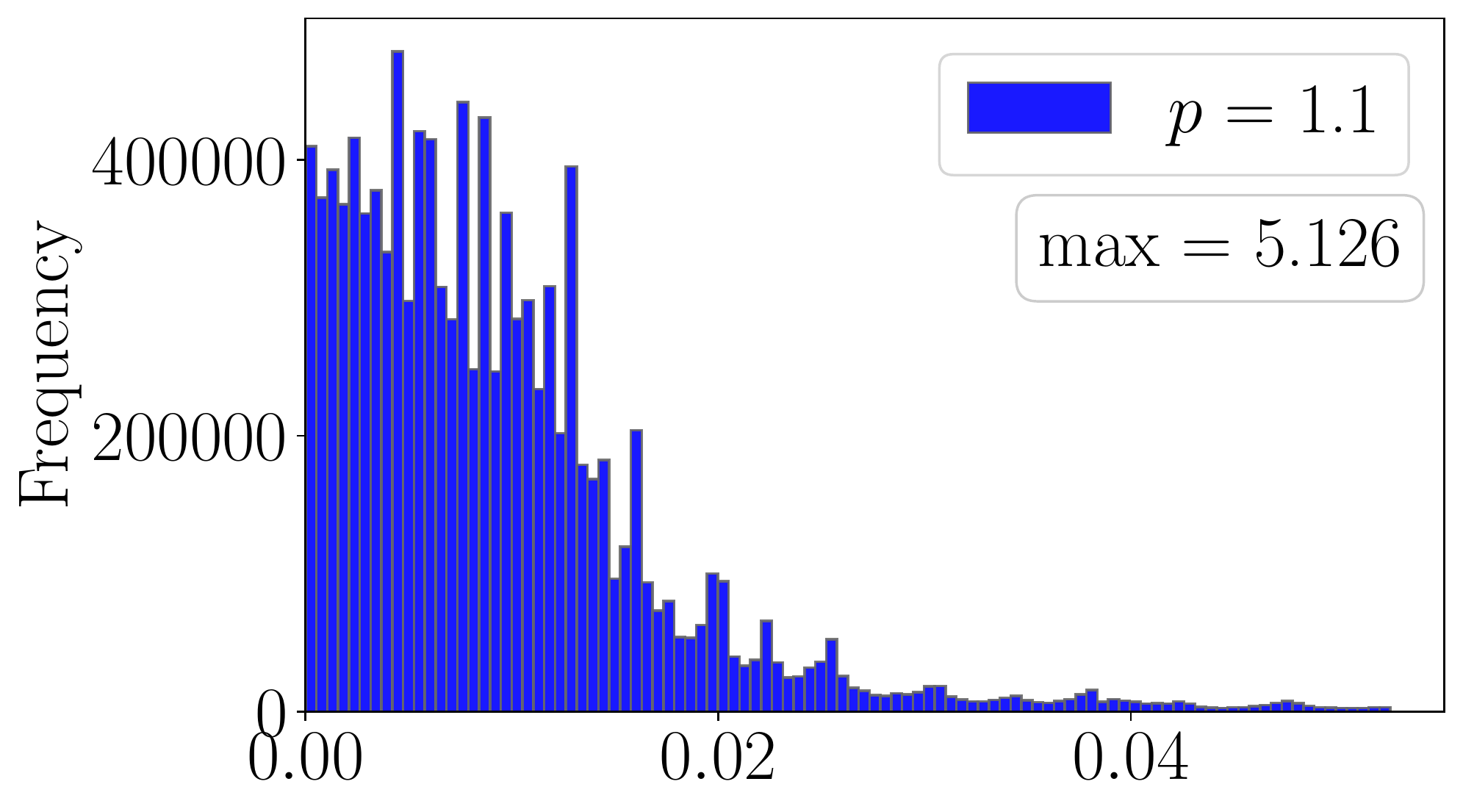}
    \end{subfigure}
    ~
    \begin{subfigure}[b]{0.45\textwidth}
        \includegraphics[width=\textwidth]{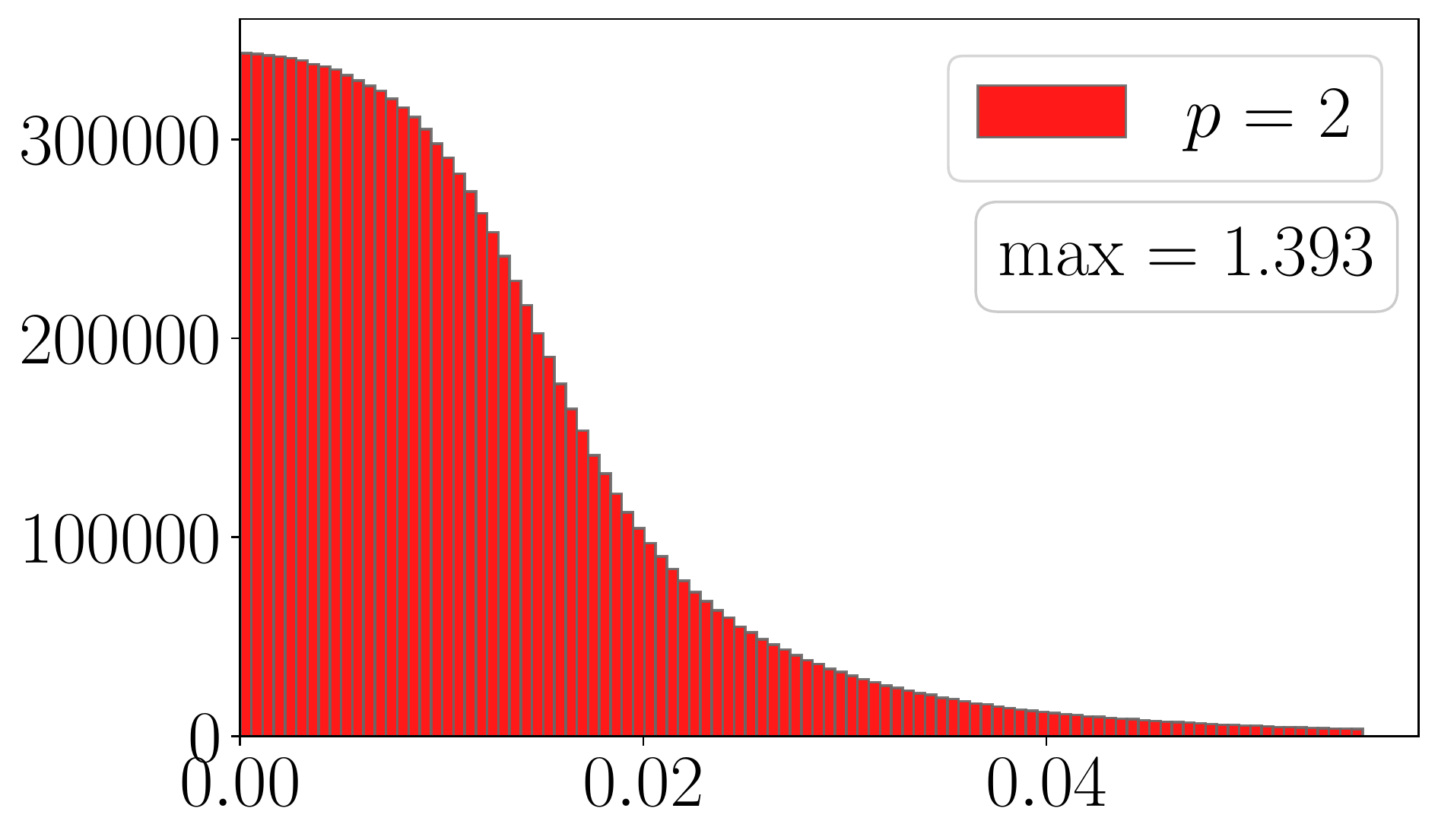}
    \end{subfigure}
    \begin{subfigure}[b]{0.48\textwidth}
        \includegraphics[width=\textwidth]{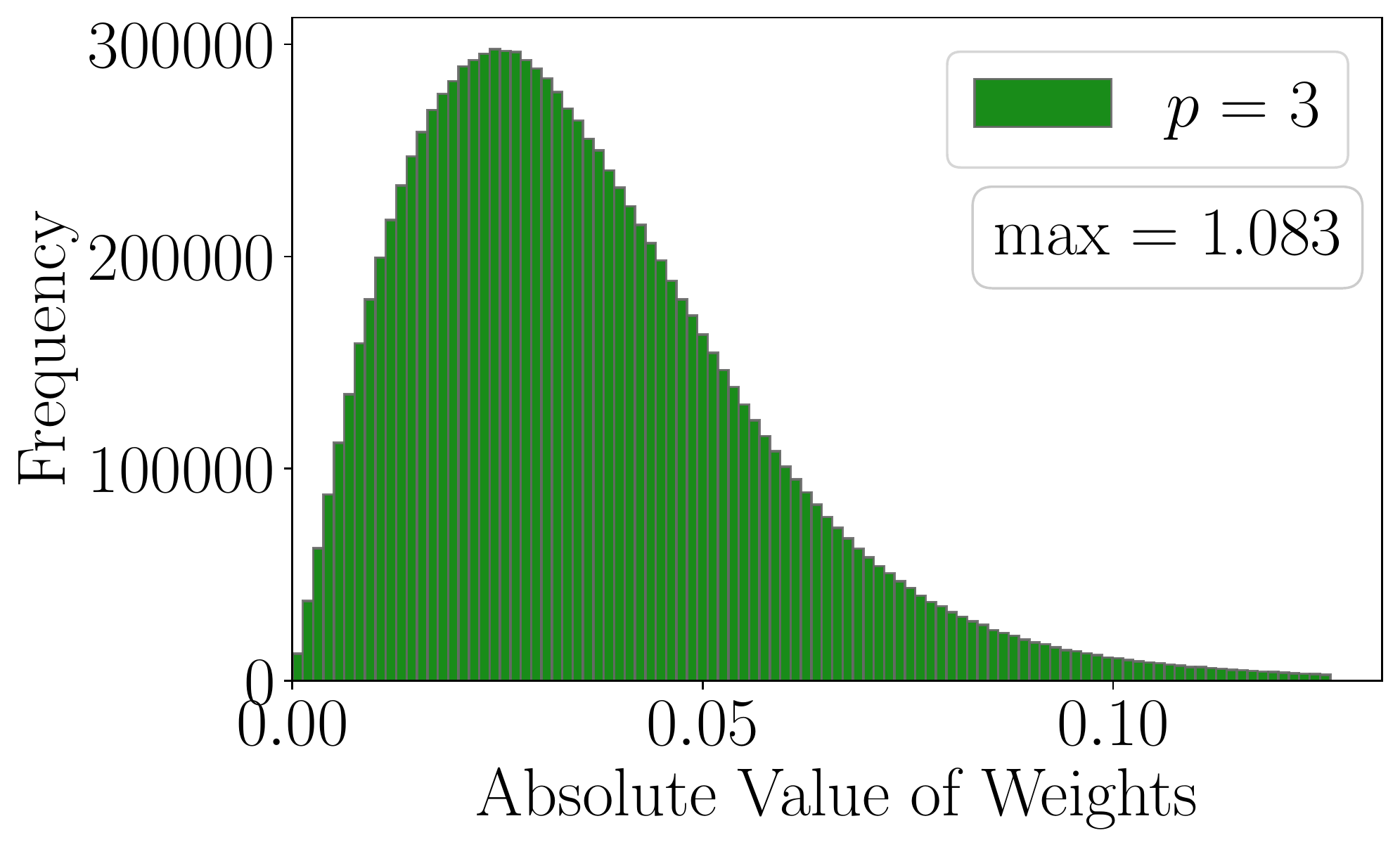}
    \end{subfigure}
    ~
    \begin{subfigure}[b]{0.45\textwidth}
        \includegraphics[width=\textwidth]{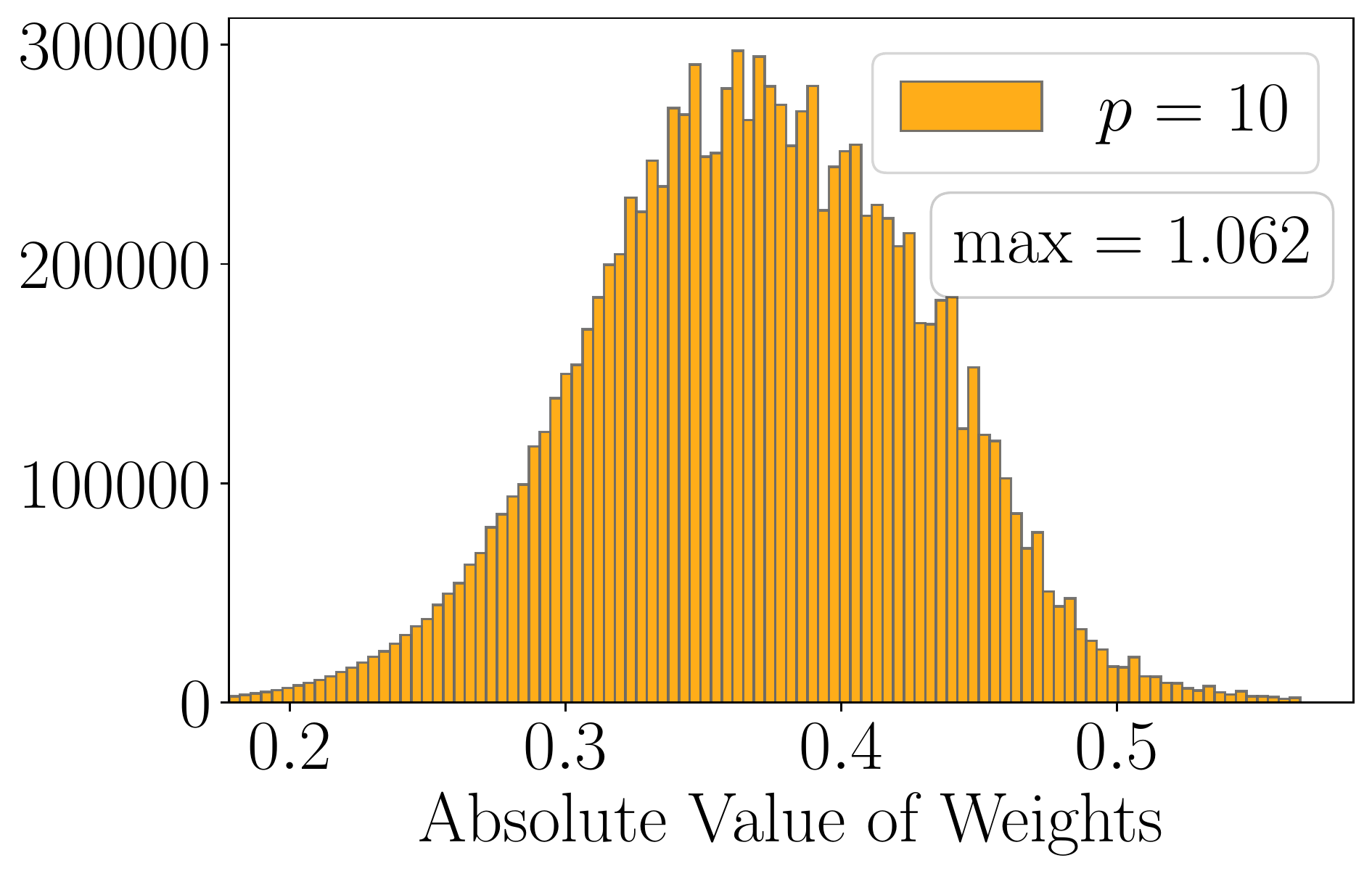}
    \end{subfigure}
    \caption{The histogram of weights in \textsc{ResNet-18} models trained with \algname (MD with potential $\psi(\cdot) = \frac{1}{p} \norm{\cdot}_p^p$) for the CIFAR-10 dataset. 
    For clarity, we cropped out the tails and each plot has 100 bins after cropping.
    The trends of these histograms reflect the implicit biases of \algname: the distribution of $p = 1.1$ has the most number of weights around zero, and the maximum weight is smallest when $p = 10$.
    }
    \label{fig:cifar10-hist}
\end{figure}

\paragraph{Implicit bias of \algname in deep neural networks.}
Since the notion of margin is not well-defined in this highly nonlinear setting, we instead visualize the impacts of \algname's implicit regularization on the histogram of weights (in absolute value) in the trained model.

In Figure~\ref{fig:cifar10-hist}, we report the weight histograms of \textsc{ResNet-18} models trained under \algname with $p = 1.1, 2, 3$ and $10$.
Depending on $p$, we observe interesting differences between the histograms.
Note that the deep network is most sparse when $p=1.1$ as most weights clustered around $0$.
Moreover, comparing the maximum weights, one can see that the case of $p = 10$ achieves the smallest value. 
Another observation is that the network becomes denser as $p$ increases; for instance, there are more weights away from zero for the cases $p = 3, 10$. 
These overall tendencies are also observed for other deep neural networks; see Appendix~\ref{sec:add-experiment-cifar-bias}.

\begin{table}[]
    \centering
    \setlength{\tabcolsep}{5pt}
    \begin{tabular}{l| c|c|c|c}
         \hline
         &  \hspace{1.25em} \textsc{VGG-11} \hspace{1.25em} & \hspace{0.75em} \textsc{ResNet-18} \hspace{0.75em} & \textsc{MobileNet-v2} & \textsc{RegNetX-200mf}  \\
         \hline \hline
         $p = 1.1$ & \pmval{88.19}{.17} & \pmval{92.63}{.12} & \pmval{91.16}{.09}& \pmval{91.21}{.18}  \\
         $p = 2$ (SGD) & \pmval{90.15}{.16} & \bpmval{93.90}{.14} & \pmval{91.97}{.10}& \pmval{92.75}{.13} \\
         $p = 3$ & \bpmval{90.85}{.15} & \bpmval{94.01}{.13} & \bpmval{93.23}{.26}& \bpmval{94.07}{.12} \\
         $p = 10$ & \pmval{88.78}{.37} & \pmval{93.55}{.21} & \pmval{92.60}{.22}& \pmval{92.97}{.16} \\
         \hline
    \end{tabular}
    \caption{CIFAR-10 test accuracy (\%) of \algname on various deep neural networks. For each deep network and value of $p$, the average $\pm$ \textcolor{gray}{std. dev.} over 5 trials are reported. And the best-performing value(s) of $p$ for each individual deep network is highlighted in \textbf{boldface}.}
    \label{tab:generalization-cifar10}
\end{table}

\paragraph{Generalization performance.}
We next investigate the generalization performance of networks trained with different $p$'s.  
To this end, we adopt a fixed selection of hyper-parameters and then train four deep neural network models to 100\% training accuracy with  \algname with different $p$'s.
As Table~\ref{tab:generalization-cifar10} shows, interestingly the networks trained by \algname with $p = 3$ consistently outperform other choices of $p$'s; notably, for \textsc{MobileNet} and \textsc{RegNet}, the case of $p=3$ outperforms the others by more than 1\%.
Somewhat counter-intuitively, the sparser network trained by \algname with $p = 1.1$ does not exhibit better generalization performance but rather shows worse generalization than other values of $p$.
Although these observations are not directly predicted by our theoretical results, we believe that understanding them would establish an important step toward understanding the generalization of over-parameterized models.
For additional experimental results, we refer the readers to Appendix~\ref{sec:add-experiment-cifar-generalization}.

\paragraph{\textsc{ImageNet} experiments.} We also perform a similar set of experiments on the {\sc ImageNet} dataset~\citep{russakovsky2015imagenet}, and these results can be found in Appendix \ref{sec:add-experiment-imagenet}.
 
\section{Conclusion and Future Work}
\label{sec:conclusion}
In this paper, we provided a unifying view of controlling implicit regularization using mirror descent. More specifically, we analyzed the implicit regularization of the mirror descent algorithm for linear classification problems with strictly monotone losses (e.g. logistics and exponential losses) and showed that for the general class of homogeneous potential functions, mirror descent converges in direction to the generalized regularized/max-margin direction.
This result, along with prior literature \cite{gunasekar2018characterizing} and \cite{azizan2018stochastic}, shows that mirror descent can induce implicit regularization with respect to a general geometry for both regression and classification problems.
Besides gradient descent, no other algorithm is known to exhibit implicit regularization in both settings.
Hence, this work completes the analysis of the first optimization algorithm that can control the implicit regularization for both general geometry and different classes of loss functions.
Finally, we ran several experiments to corroborate our theoretical findings and to illustrate the practical applications of mirror descent.
The experiments are conducted in various settings: (i) linear models in both low and high dimensions, (ii) real-world data with highly over-parameterized nonlinear models.

We conclude this paper with several important future directions:
\begin{list}{{\tiny $\blacksquare$}}{\leftmargin=1.5em}
\setlength{\itemsep}{-1pt}
    \item As discussed in Section \ref{sec:cifar}, different choices of $p$ for the \algname algorithm result in different generalization performance. It would be interesting to develop a theory that explains under what conditions can we show that certain MD potentials lead to better generalization performance.
    \item Another interesting line of work is to extend this analysis to more sophisticated optimization algorithms such as Adam or RMSprop. While we proved the convergence for a very simple adaptive step size strategy in Section~\ref{sec:normalized-asymp}, it would be interesting to see if such analysis can be strengthened to cover the algorithms most commonly used in practice.
\end{list}

\newpage

\acks{The authors thank Christos Thrampoulidis for insightful discussions, his valuable feedback, and his involvement in an earlier version of this work.
The authors thank former MIT UROP students Tiffany Huang and Haimoshri Das for contributing to the experiments in Section~\ref{sec:cifar}.
The authors acknowledge the MIT SuperCloud and Lincoln Laboratory Supercomputing Center for providing computing resources that have contributed to the results reported within this paper.
This work was supported in part by MathWorks, the MIT-IBM Watson AI Lab, and Amazon. }

\vskip 0.2in
\bibliography{reference}

\newpage

\appendix

\section{Extended Version of Table~\ref{table:main}}
\label{sec:table-full}
\begin{table}[h]
\centering
\renewcommand{\arraystretch}{1.25}
\setlength\tabcolsep{8pt}

\begin{tabular}{ |c |c|c| }
\hline  & Regression (e.g. square loss) & Classification (e.g. logistic loss) \\
\hline\hline  
\multirow{4}{*}{\begin{tabular}{c}Gradient Descent\\(i.e. $\mir(\cdot) = \frac{1}{2}\norm{\cdot}_2^2$)\end{tabular}}  & $\argmin_w \norm{w-w_0}_2$ & $\argmin_w \frac{1}{2}\norm{w}_2^2$  \\
 & $\st~~w \text{ fits all data} $ & $\st~~w \text{ classifies all data} $    \\
 & 
\multirow{2}{*}{\citep[Thm 6.1]{engl1996regularization}} & \cite{soudry2018implicit}   \\
 & & \cite{ji2019implicit} \\
   \hline
 \multirow{4}{*}{\begin{tabular}{c}Mirror Descent\\ \end{tabular}}  & $\argmin_w D_\psi(w, w_0)$ & $\argmin_w \psi(w)$  \\
 & $\st~~w \text{ fits all data} $ & $\st~~w \text{ classifies all data} $    \\
 & 
\cite{gunasekar2018characterizing} & \multirow{2}{*}{ \large \high{This work}}  \\
& \cite{azizan2018stochastic} & \\
  \hline 
\end{tabular}

\caption{{\bf Conceptual summary of our results.} In the case of well-specified linear regression, there is a complete theory of implicit regularization with respect to a general geometry; it is shown that mirror descent converges to the interpolating solution that is closest to the initialization. 
However, such characterization in the separable linear classification setting is missing in the literature.
In this paper, we prove the implicit regularization of mirror descent with a class of homogeneous potentials and extend the result of gradient descent beyond the $\ell_2$-norm.
Compare to Table~\ref{table:main}, here we consider an arbitrary initialization $w_0$ in the regression setting.}
\label{table:main-full}
\end{table} 
\section{Proofs for Section~\ref{sec:priliminaries}}
\label{sec:proof-basic-lemmas}
\subsection{Proof of Lemma~\ref{thm:key-iden}}
The overall proof follows \citep{azizan2021stochastic}.
We make several modifications to make it better applicable to the classification setting.
Note that in the classification setting, there is no $w \in \RR^d$ that satisfies $L(w) = 0$.

\begin{proof}
We start with the definition of Bregman divergence:
\[\brg{w}{w_{t+1}}  = \psi(w) - \psi(w_{t+1}) - \inp{\nabla\psi(w_{t+1})}{w-w_{t+1}}.\]
Now, we plugin the \ref{equ:md} update rule $\nabla\psi(w_{t+1}) = \nabla\psi(w_t) - \eta\nabla L(w_t)$:
\[\brg{w}{w_{t+1}} = \psi(w) - \psi(w_{t+1}) - \inp{\nabla\psi(w_{t})}{w - w_{t+1}} + \eta \inp{\nabla L(w_t)}{w - w_{t+1}}. \]
We again invoke the definition of Bregman divergence so that:
\begin{align*}
    \brg{w}{w_{t+1}}  &= \psi(w) - \psi(w_{t+1}) - \inp{\nabla\psi(w_{t+1})}{w-w_{t+1}}, \\
    \brg{w_{t+1}}{w_t}  &= \psi(w_{t+1}) - \psi(w_t) - \inp{\nabla\psi(w_t)}{w_{t+1}-w_t}.
\end{align*}

It follows that
\begin{equation}
\label{equ:proof-key-iden-1}
    \begin{aligned}
        \brg{w}{w_{t+1}} 
        &= \psi(w) - \psi(w_{t}) - \inp{\nabla\psi(w_{t})}{w - w_{t}} \\
        &\hspace{7.25em}+ \inp{\nabla\psi(w_{t})}{w_{t+1} - w_t} - \psi(w_{t+1}) + \psi(w_t) \\
        &\hspace{7.25em}+ \eta \inp{\nabla L(w_t)}{w - w_{t+1}} \\
        &= \brg{w}{w_t} - \brg{w_{t+1}}{w_t} + \eta \inp{\nabla L(w_t)}{w - w_{t+1}}
    \end{aligned}
\end{equation}

Next, we consider the term $\inp{\nabla L(w_t)}{w - w_{t-1}}$:
\begin{equation}
    \label{equ:proof-key-iden-2}
    \begin{aligned}
        \inp{\nabla L(w_t)}{w - w_{t-1}}
        &= \inp{\nabla L(w_t)}{w - w_t} - \inp{\nabla L(w_t)}{w_{t+1} - w_t} \\
        &\hspace{8em}+ L(w_{t+1}) - L(w_t) - L(w_{t+1}) + L(w_t) \\
        &= \inp{\nabla L(w_t)}{w - w_t} + D_L(w_{t+1}, w_t) - L(w_{t+1}) + L(w_t),
    \end{aligned}
\end{equation}
where the last step holds because $L$ is convex.

Combining \eqref{equ:proof-key-iden-1} and \eqref{equ:proof-key-iden-2} yields:
\begin{align*}
    &\brg{w}{w_t} \\
    ={}& \brg{w}{w_{t+1}} + \brg{w_{t+1}}{w_t} -\eta\big( \inp{\nabla L(w_t)}{w - w_t} + D_L(w_{t+1}, w_t) - L(w_{t+1}) + L(w_t)\big)\\
    ={}& \brg{w}{w_{t+1}} + \breg{\psi-\eta L}{w_{t+1}}{w_t} - \eta  \inp{\nabla L(w_t)}{w - w_t} + \eta L(w_{t+1}) - \eta L(w_t),
\end{align*}
where in the last step, we note that Bregman divergence is additive in its potential.
This gives us \eqref{equ:key-iden-2}.
And for \eqref{equ:key-iden-1}, we use the definition of Bregman divergence again, i.e. $D_L(w, w_t) = L(w) - L(w_t) - \inp{\nabla L (w_t)}{w - w_t}$:
\begin{align*}
    \brg{w}{w_t}
    &= \brg{w}{w_{t+1}} + \breg{\psi-\eta L}{w_{t+1}}{w_t} - \eta  \inp{\nabla L(w_t)}{w - w_t} \\
    &\hspace{8em}  + \eta L(w) - \eta L(w_t)  + \eta L(w_{t+1}) - \eta L(w) \\
    &= \brg{w}{w_{t+1}} + \breg{\psi-\eta L}{w_{t+1}}{w_t} + \eta  D_L(w, w_t) - \eta L(w)  + \eta L(w_{t+1})\,.
\end{align*}
This completes the proof of Lemma~\ref{thm:key-iden}.
\end{proof}

\subsection{Proof of Lemma~\ref{thm:decreasing-lose}}
\begin{proof}
This is an application of Lemma~\ref{thm:key-iden} with $w = w_t$:
\begin{align*}
    0 &= D_\psi(w_t, w_{t+1}) + D_{\psi - \eta L}(w_{t+1}, w_t) - \eta L(w_t) + \eta L(w_{t+1}) \\
    \implies \eta L(w_t) &= D_\psi(w_t, w_{t+1}) + D_{\psi - \eta L}(w_{t+1}, w_t) + \eta L(w_{t+1}) \ge \eta L(w_{t+1})
\end{align*}
where we used the fact that Bregman divergence with a convex potential function is non-negative.
\end{proof}

\subsection{Proof of Lemma~\ref{thm:to-infinity}}
\begin{proof}
By Lemma~\ref{thm:decreasing-lose}, $L(w_t)$ is decreasing with respect to $t$, therefore the limit exists.
Suppose the contrary that $\lim_{t\to\infty} L(w_t) = \varepsilon > 0$.
Since the data is separable, we can pick $w$ so that $L(w) \le \varepsilon / 2$.
Applying Lemma~\ref{thm:key-iden}, the following holds for all $t$:
\begin{align*}
    D_\psi(w, w_{t+1}) 
    &= D_\psi(w, w_{t}) - D_{\psi - \eta L}(w_{t+1}, w_t) - \eta D_{L}(w, w_t) + \eta L(w) - \eta L(w_{t+1}) \\
    &\le D_\psi(w, w_{t}) + \eta\varepsilon/2 - \eta\varepsilon = D_\psi(w, w_{t}) - \eta\varepsilon/2
\end{align*}
Hence, $D_\psi(w, w_{t}) \le D_\psi(w, w_0) - t\eta\varepsilon / 2$.
This implies that $\limsup_{t\to\infty} D_\psi(w, w_{t}) = -\infty$, contradiction.
\end{proof} 

\section{Properties of Potential Functions under Assumption~\ref{thm:assump-potential}}
\label{sec:potential-norm}

In this section, we shall establish several useful properties for potential functions satisfying Assumption~\ref{thm:assump-potential}.
First, we will show that $\norm{\cdot}_\psi$, as defined in \eqref{equ:potential-norm}, is a valid norm induced by the potential $\psi$.
\begin{list}{{\tiny $\blacksquare$}}{\leftmargin=1.5em}
\setlength{\itemsep}{-1pt}
    \item Since $\psi$ is positive definite, $\norm{w}_\psi = 0$ only when $w = 0$.
    \item By applying homogeneity and the definition of $\norm{\cdot}_\psi$, we have 
    \begin{align*}
        \norm{s w}_\psi 
        &= \inf \{c > 0: \psi(sw/c) \le 1\} \\
        &= \inf \left\{c > 0: \left| \frac{s}{|s|} \right|^\beta \psi\left(\frac{w}{c/|s|}\right) \le 1\right\} \\
        &= \inf \left\{c > 0: \psi\left(\frac{w}{c/|s|}\right) \le 1\right\} = |s| \cdot \norm{w}_\psi
    \end{align*}
    \item To show the triangle inequality, we consider any vectors $w_1, w_2 \in \RR^d$ and let $a = \norm{w_1}_\psi$ and $b = \norm{w_2}_\psi$.
    Because $\psi$ is convex, we have
    \[\psi\left(\frac{w_1+w_2}{a+b}\right) = \psi\left(\frac{a}{a+b} \cdot \frac{w_1}{a} + \frac{b}{a+b} \cdot \frac{w_2}{b}\right) \le \frac{a}{a+b} \cdot \psi\left(\frac{w_1}{a}\right) + \frac{b}{a+b} \cdot \psi\left(\frac{w_2}{b}\right) = 1.\]
    Therefore, $\norm{w_1 + w_2}_\psi \le a + b$, as desired.
\end{list}
Therefore, $\norm{\cdot}_\psi$ is a norm.
Note that, due to continuity, we have $\psi(w / \norm{w}_\psi) = 1$.
It follows that, by homogeneity, we can write $\psi(\cdot) = \norm{\cdot}_\psi^\beta$.

Next, we show that convexity and $\beta$-absolute homogeneity imply strict convexity.
So, we can in fact relax the conditions in Assumption~\ref{thm:assump-potential} where $\psi$ is only required to be convex.
Another consequence of this fact is that the potential $\psi(\cdot) = \norm{\cdot}^\beta$ for any norm $\norm{\cdot}$ satisfies the conditions in Assumption~\ref{thm:assump-potential}.
For any $\lambda \in (0, 1)$, we have
\[\norm{\lambda x + (1-\lambda)y}_\psi^\beta \le (\lambda \norm{x}_\psi + (1-\lambda) \norm{y}_\psi)^\beta < \lambda \norm{x}_\psi^\beta + (1-\lambda) \norm{y}_\psi^\beta\]
The first inequality is due to the triangle inequality and the second inequality holds because the map $z \mapsto |z|^\beta$ is strictly convex whenever $\beta > 1$.
Therefore, $\psi$ is strictly convex.

By appealing to the limit definition of gradient, we can show \eqref{equ:breg-homo}, so that the Bregman divergence is also homogeneous.
\begin{align*}
    \inp{\nabla \psi (w)}{w} 
    &= \lim_{h \to 0} \frac{\psi(w + h w) - \psi(w)}{h} \\
    &= \lim_{h \to 0} \frac{(1+h)^\beta - 1}{h} \psi(w) = \beta \cdot \psi(w) \\
    \brg{c w}{c w'} 
    &= \psi(cw) - \psi(cw') - \inp{\nabla \psi(cw')}{c (w - w')} \\
    &= |c|^\beta \psi(w) - |c|^\beta \psi(w') - \lim_{h \to 0} \frac{\psi(c w' + h c(w-w')) - \psi(cw')}{h} \\
    &= |c|^\beta \psi(w) - |c|^\beta \psi(w') - \lim_{h \to 0} |c|^\beta \frac{\psi(w' + h (w-w')) - \psi(w')}{h} \\
    &= |c|^\beta \psi(w) - |c|^\beta \psi(w') - |c|^\beta \inp{\nabla \psi(w')}{w - w'} \\
    &= |c|^\beta \brg{w}{w'} \quad \forall c\in \R.
\end{align*}

Finally, through the triangle inequality, we can show that $\nabla\psi(w)$ and $w$ are ``parallel'' in the sense that
\begin{equation}\label{equ:mirror-parallel}
    |\inp{\nabla\psi(w)}{w}| \ge |\inp{\nabla\psi(w)}{v}|, \forall \, v \text{ s.t.} \norm{v}_\psi = \norm{w}_\psi.
\end{equation}
\begin{proof}
From the limit definition, we have
\[\inp{\nabla\psi(w)}{v} = \lim_{h \to 0} \frac{\psi(w + h v) - \psi(w)}{h}.\]
For any $h > 0$, we have
\begin{align*}
    \psi(w + hv) = \norm{w + hv}_\psi^\beta \le  (\norm{w}_\psi + h \norm{w}_\psi)^\beta = \psi((1+h)w),
\end{align*}
and similarly,
\[\psi(w + hv) = \norm{w + hv}_\psi^\beta \ge  (\norm{w}_\psi - h \norm{w}_\psi)^\beta = \psi((1-h)w).\]
Therefore, we have that
\[\lim_{h \to 0} \frac{\psi(w - h w) - \psi(w)}{h} \le \lim_{h \to 0} \frac{\psi(w + h v) - \psi(w)}{h} \le \lim_{h \to 0} \frac{\psi(w + h w) - \psi(w)}{h},\]
and hence
\[\inp{\nabla\psi(w)}{-w} \le \inp{\nabla\psi(w)}{v} \le \inp{\nabla\psi(w)}{w} \implies |\inp{\nabla\psi(w)}{v}| \le |\inp{\nabla\psi(w)}{w}|.\]
\end{proof} 
\section{Discussion of Step Sizes}
\label{sec:step-size}
In this section, we provide the details for Remark~\ref{rem:step-size}. In particular, 
we discuss the existence of small step size $\eta$ such that $\psi - \eta L$ is convex.
We break down this discussion into two parts; first, we analyze the case of standard mirror descent update~\eqref{equ:md} with fixed step size, and secondly, we consider the case of normalized mirror descent~\eqref{equ:normalized-md} with time-varying step sizes.
For concreteness, we focus on the exponential loss $\ell(z) = \exp(-z)$.

\subsection{MD with fixed step size}
 
We first use Assumption~\ref{thm:assump-boundedness} to directly bound the Hessian $\nabla^2 L$:
\[\norm{\nabla^2 L(w)}_2 = \norm{\frac{1}{n} \sum_{i=1}^n \exp(-y_iw^\top x_i) x_i x_i^\top}_2 \le \sum_{i=1}^n \exp(-y_iw^\top x_i) C^2 = C^2 L(w)\]
It follows that if $\psi$ is $\mu$-strongly convex, then $\eta < \frac{\mu}{C^2 L(w_0)}$ ensures that $\psi - \eta L$ is convex at $w_0$.
By applying Lemma~\ref{thm:decreasing-lose} and induction, we can conclude that $\psi - \eta L$ is convex at all iterates $\{w_t\}_{t=0}^\infty$.
For instance, gradient descent falls under this case because it has a 1-strongly convex potential $\psi(\cdot) = \frac{1}{2} \norm{\cdot}_2^2$.

However, for any $\beta \neq 2$, $\psi$ is not strongly convex, and we shall consider the cases where $\beta \in (1, 2)$ and $\beta > 2$ separately.
When $\beta \in (1, 2)$, we invoke the following fact:
\begin{lemma}
\label{thm:hessian-scaling}
    Consider a function $\psi : \RR^d \to \RR$ satisfying Assumption~\ref{thm:assump-potential},
    \begin{list}{{\tiny $\blacksquare$}}{\leftmargin=1.5em}
    \setlength{\itemsep}{-1pt}
        \item If there exists $m > 0$ so that $\inf_{\norm{w}_2=1}\norm{\nabla^2\psi(w)}_2 \ge m$, then $\norm{\nabla^2\psi(w)}_2 \ge m \norm{w}_2^{\beta - 2}$.
        \item If there exists $M > 0$ so that $\sup_{\norm{w}_2=1}\norm{\nabla^2\psi(w)}_2 \le M$, then $\norm{\nabla^2\psi(w)}_2 \le M \norm{w}_2^{\beta - 2}$.
    \end{list}
    The matrix norm used here is the operator norm induced by the $\ell_2$ vector norm.
\end{lemma}
It follows that $\norm{\nabla^2\psi(w)}_2 \in \Omega(\norm{w}_2^{\beta-2})$ as long as it is uniformly positive on the unit circle.
Then, since $\norm{\nabla^2 L(w)}_2$ decays exponentially with respect to $\norm{w}_2$, there exists $\eta > 0$ so that $\psi - \eta L$ is convex for all iterates $\{w_t\}_{t=0}^\infty$ and this $\eta$ can be computed from a finite horizon over the iterates.
For instance, the \algname potential $\psi(\cdot) = \frac{1}{p} \norm{\cdot}_p^p$ satisfies the desired condition on $\nabla^2\psi$ whenever $p \in (1, 2)$.

For $\beta > 2$, we alternatively consider the potential $\psi'(\cdot) = \frac{\varepsilon}{2} \norm{\cdot}_2^2 + \psi(\cdot)$ for some small $\varepsilon > 0$.
Intuitively, this modified potential should induce the same implicit bias as the original potential since their asymptotic tails are the same.
Formally, because Bregman divergence is additive, we can plug $\psi'$ into the proof of Theorem~\ref{thm:primal-bias} and show that mirror descent with potential $\psi'$ converges to $\reg{\psi}$.
Since $\psi'$ is strongly convex, the existence of step size $\eta$ is assured.
In practice, we find that directly applying mirror descent with potential $\psi$ already works well.

\subsection{Normalized MD with variable step sizes}
We claim that there exists sufficiently small $\eta_0 > 0$ so that $\eta_t = \frac{\eta_0}{\sqrt{t+1}} L(w_t)^{-1}$ satisfies $\psi - \eta_t L$ being locally convex at all iterates $\{w_t\}_{t=0}^\infty$.
Recall that $\norm{\nabla^2 L(w)}_2 \le C^2 L(w)$.
If $\psi$ is $\mu$-strongly convex, then $\eta_0 < \frac{\mu}{C^2}$ guarantees that $\psi - \eta_t L$ is convex at the iterates $\{w_t\}_{t=0}^\infty$.
Therefore, our analysis from the previous section still applies for $\beta \ge 2$.

And for $\beta \in (1, 2)$, we apply a more careful analysis.
We first leverage the upper bound on $\norm{w}_\psi$ from Lemma~\ref{thm:normalized-norm-upper}.
From Assumption~\ref{thm:assump-potential}, we can show that $\inf_{\norm{w}_2 = 1} \psi(w)$ is positive.
Therefore,
\[ \norm{w_t}_2 \le \left(\inf_{\norm{w}_2 = 1} \psi(w)\right)^{-1/\beta} \norm{w_t}_\psi \in O\left(\left(\inf_{\norm{w}_2 = 1} \psi(w)\right)^{-1/\beta} \sqrt{t}\right) = O(\sqrt{t}). \]
Combining with Lemma~\ref{thm:hessian-scaling}, we have that $\norm{\nabla^2 \psi(w_t)}_2 \in \Omega(\sqrt{t}^{\beta-2}) \subseteq \omega(1/\sqrt{t})$ for $\beta \in (0, 1)$.
Since we have 
\[ \eta_t \norm{\nabla^2 L(w_t)}_2 \le \frac{\eta_0}{\sqrt{t+1}} C^2 \in O(1/\sqrt{t}),\]
there exists sufficiently small $\eta_0 > 0$ where $\psi - \eta_t L$ is convex at respective iterates $\{w_t\}_{t=0}^\infty$.

\subsection{Proof of Lemma~\ref{thm:hessian-scaling}}
\begin{proof}
    We directly compute the derivatives through their limit definition.
    For any vector $v$,
    \begin{align*}
        \nabla \psi(w)^\top v
        &= \lim_{h\to 0} \frac{\psi(w + hv_1) - \psi(w)}{h} \\
        &= \norm{w}_2^\beta \lim_{h \to 0} \frac{\psi(w/\norm{w}_2 + h v/\norm{w}_2) - \psi(w/\norm{w}_2)}{h} \\
        &= \norm{w}_2^\beta \nabla \psi(w/\norm{w}_2)^\top (v/\norm{w}_2) \\
        &= \norm{w}_2^{\beta-1} \nabla \psi(w/\norm{w}_2)^\top v
    \end{align*}
    Therefore, $\nabla \psi(w) = \norm{w}_2^{\beta-1}\nabla \psi(w/\norm{w}_2)$.

    As for the second derivative, for any vectors $v_1, v_2$, we have
    \begin{align*}
        v_2^\top \nabla^2 \psi(w) v_1
        &= \lim_{h\to 0} \frac{\nabla \psi(w + hv_2)^\top v_1 - \nabla \psi(w)^\top v_1}{h} \\
        &= \norm{w}_2^{\beta-1} \lim_{h \to 0} \frac{\nabla\psi(w/\norm{w}_2 + h v_2/\norm{w}_2)^\top v_1 - \nabla\psi(w/\norm{w}_2)^\top v_1}{h} \\
        &= \norm{w}_2^{\beta-1} (v_2 / \norm{w}_2)^\top \nabla^2 \psi(w/\norm{w}_2) v_1 \\
        &= \norm{w}_2^{\beta-2} v_2^\top \nabla^2 \psi(w/\norm{w}_2) v_1
    \end{align*}
    Therefore, $\nabla^2 \psi(w) = \norm{w}_2^{\beta-2}\nabla^2 \psi(w/\norm{w}_2)$. Now, this Lemma immediately follows from the computation above.
\end{proof} 
\section{Proofs for Section~\ref{sec:primal-bias-result}}

\subsection{Proof of Lemma~\ref{thm:approx-reg-dir-loss}}

\begin{proof}
Let $\bar{\gamma}$ be the margin of $\reg{\psi}$.
Under separability, we know $\bar{\gamma} > 0$.
Recall the definition of the regularization path.
There exists sufficiently large $r_\alpha$ so that 
\[ \norm{\frac{\bar{w}(\norm{w}_\psi)}{\norm{w}_\psi} - \reg{\psi}}_\psi \le \frac{\alpha \bar{\gamma}}{C} \]
whenever $\norm{w}_\psi \ge r_\alpha$.
Recall that, from Assumption~\ref{thm:assump-boundedness}, $C \ge \max_{i = 1, \dots, n} \norm{x_i}_{\psi, *}$.
Then, for all $i \in [n]$, we have
\begin{align*}
    y_i \inp{\bar{w}(\norm{w}_\psi)}{x_i}
    &= y_i \inp{\bar{w}(\norm{w}_\psi) - \norm{w}_\psi \reg{\psi}}{x_i} + y_i \inp{\norm{w}_\psi\reg{}}{x_i} \\
    &\le \alpha \bar{\gamma} \norm{w}_\psi \norm{x_i}_{\psi, *} / C + y_i \inp{\norm{w}_\psi\reg{\psi}}{x_i} \\
    &\le \alpha \bar{\gamma} \norm{w}_\psi + y_i \inp{\norm{w}_\psi\reg{\psi}}{x_i} \\
    &\le y_i \inp{(1+\alpha) \norm{w}_\psi \reg{\psi}}{x_i}
\end{align*}
Since the loss $L$ is decreasing, we have
\[L((1+\alpha)\norm{w}_\psi\reg{\psi}) \le L(\bar{w}(\norm{w}_\psi)) \le L(w)\,, \]
as desired.
\end{proof}

\subsection{Lower bounding the mirror descent updates}
\label{sec:cross-term-diff}
\begin{lemma}
\label{thm:cross-term-diff}
For any potential $\psi$ satisfying Assumption~\ref{thm:assump-potential}, the mirror descent update satisfies the following inequality:
\begin{equation}
    (\beta-1) \psi(w_{t+1}) - (\beta-1) \psi(w_t) + \eta L (w_{t+1}) - \eta L (w_{t}) \le \inp{-\eta\nabla L (w_{t})}{w_t}
\end{equation}
\end{lemma}
\begin{proof}
This result follows from Lemma~\ref{thm:key-iden} with $w = 0$:
\begin{align*}
    \brg{0}{w_t}
    &= \brg{0}{w_{t+1}} + \breg{\psi-\eta L}{w_{t+1}}{w_t} + \eta \brgl{0}{w_t}+ \eta L(w_{t+1}) - \eta L(0) \\
    &\ge \brg{0}{w_{t+1}} + \eta \brgl{0}{w_t}+ \eta L(w_{t+1}) - \eta L(0) \\
    &= \brg{0}{w_{t+1}} + \eta(L(0) - L(w_t) - \inp{\nabla L(w_t)}{-w_t}) + \eta L(w_{t+1}) - \eta L(0) \\
    &= \brg{0}{w_{t+1}} + \eta \inp{\nabla L(w_t)}{w_t} + \eta L(w_{t+1}) - \eta L(w_t)
\end{align*}
Rearranging the terms yields
\[ \brg{0}{w_{t+1}} - \brg{0}{w_{t}} + \eta L (w_{t+1}) - \eta L (w_{t}) \le \inp{-\eta\nabla L (w_{t})}{w_t} \]
We conclude the proof by noting that for any $w \in \RR^d$, 
\[\brg{0}{w}  = \psi(0) - \psi(w) - \inp{\nabla\psi(w)}{-w} = \inp{\nabla\psi(w)}{w} - \psi(w) = (\beta-1) \psi(w)\,, \]
where the last equality follows from the homogeneity of Bregman divergence \eqref{equ:breg-homo}.
\end{proof}

\subsection{Proof of Theorem~\ref{thm:primal-bias}}
\label{sec:proof-primal-bias}
\begin{proof}
Consider arbitrary $\alpha \in (0, 1)$ and define $r_\alpha$ according to Lemma~\ref{thm:approx-reg-dir-loss}.
Since $\lim_{t\to\infty}\norm{w_t}_\psi = \infty$, we can find $t_0$ so that $\norm{w_t}_\psi > \max(1, r_\alpha)$ for all $t \ge t_0$.
Let $c_t = (1+\alpha)\norm{w_t}_\psi$.

Substitute $w = c_t \reg{\psi}$ into Lemma~\ref{thm:key-iden}, we get
\[\brg{c_t\reg{\psi}}{w_{t+1}} \le \brg{c_t\reg{\psi}}{w_t} + \eta\inp{\nabla L(w_t)}{c_t \reg{\psi} - w_t} - \eta L(w_{t+1}) + \eta L(w_t).\]
By Corollary~\ref{thm:cross-term}, we have $\inp{\nabla L(w_t)}{c_t \reg{\psi} - w_t} \le 0$.
Therefore,
\[\brg{c_t\reg{\psi}}{w_{t+1}} \le \brg{c_t\reg{\psi}}{w_t} - \eta L(w_{t+1}) + \eta L(w_t).\]

It follows that
\begin{align*}
    &\brg{c_{t+1}\reg{\psi}}{w_{t+1}} \\
    \le{}& \brg{c_t\reg{\psi}}{w_t} - \eta L(w_{t+1}) + \eta L(w_t) + \brg{c_{t+1}\reg{\psi}}{w_{t+1}} - \brg{c_t\reg{\psi}}{w_{t+1}} \\
    ={}& \brg{c_t\reg{\psi}}{w_t} - \eta L(w_{t+1}) + \eta L(w_t) + \psi(c_{t+1} \reg{\psi}) - \psi(c_t \reg{\psi}) - \inp{\nabla\psi(w_{t+1})}{(c_{t+1} - c_t) \reg{\psi}}
\end{align*}
Summing over $t = t_0, \dots, T-1$ gives us
\begin{align}
    \brg{c_T\reg{\psi}}{w_T}
    &\le \brg{c_{t_0}\reg{\psi}}{w_{t_0}} - \eta L(w_T) + \eta L(w_{t_0}) + \psi(c_{T} \reg{\psi}) - \psi(c_{t_0} \reg{\psi}) \nonumber\\
    &\quad\quad- \sum_{t=t_0}^{T-1}\inp{\nabla\psi(w_{t+1})}{(c_{t+1} - c_t) \reg{\psi}} \label{equ:succ-cross-term}
\end{align}

Now we want to establish a lower bound on the last term of \eqref{equ:succ-cross-term}.
To do so, we inspect the change in $\nabla\psi(w_t)$ from each successive mirror descent update:
\begin{subequations}
\begin{align}
    &\inp{\nabla\psi(w_{t+1}) - \nabla\psi(w_t)}{\reg{\psi}} \\
    ={}& \inp{-\eta\nabla L(w_{t})}{\reg{\psi}} \\
    \ge{}& \frac{1}{(1+\alpha)\norm{w_t}_\psi} \inp{-\eta\nabla L(w_t)}{w_t} \label{equ:cross-term-expansion-L1}\\
    \ge{}& \frac{1}{(1+\alpha)\norm{w_t}_\psi} \left((\beta-1) \norm{w_{t+1}}_\psi^\beta - (\beta-1) \norm{w_{t}}_\psi^\beta + \eta L(w_{t+1}) - \eta L(w_{t})\right) \label{equ:cross-term-expansion-L2} \\
    \ge{}& \frac{1}{(1+\alpha)\norm{w_t}_\psi} \left((\beta-1) \norm{w_{t+1}}_\psi^\beta - (\beta-1) \norm{w_{t}}_\psi^\beta\right) + \eta L(w_{t+1}) - \eta L(w_{t}) \label{equ:cross-term-expansion-L3}
\end{align}
\end{subequations}
where we applied Corollary~\ref{thm:cross-term} on \eqref{equ:cross-term-expansion-L1} and Lemma~\ref{thm:cross-term-diff} on \eqref{equ:cross-term-expansion-L2}.

Now we bound \eqref{equ:cross-term-expansion-L3}.
We claim the following identity and defer its derivation to Section~\ref{sec:main-thm-aux-pow}.
\begin{equation}
    \label{equ:p-norm-diff}
    (\beta-1)(\norm{w_{t+1}}_\psi^\beta - \norm{w_t}_\psi^\beta) \ge \beta (\norm{w_{t+1}}_\psi^{\beta-1} - \norm{w_{t}}_\psi^{\beta-1}) \norm{w_t}_\psi.
\end{equation}

We are left with
\begin{equation*} \inp{\nabla\psi(w_{t+1}) - \nabla\psi(w_t)}{\reg{\psi}} \ge \beta \cdot \frac{\norm{w_{t+1}}_\psi^{\beta-1} - \norm{w_{t}}_\psi^{\beta-1}}{1+\alpha} + \eta L(w_{t+1}) - \eta L(w_{t}).
\end{equation*}

Summing over $t = t_0, \dots, T-1$ gives us
\begin{equation}\label{equ:mirror-cross-term}
\inp{\nabla\psi(w_{T}) - \nabla\psi(w_{t_0})}{\reg{\psi}} \ge \beta \cdot \frac{\norm{w_T}_\psi^{\beta-1} - \norm{w_0}_\psi^{\beta-1}}{1+\alpha} + \eta L(w_{T}) - \eta L(w_{t_0}).
\end{equation}

With \eqref{equ:mirror-cross-term}, we can bound the last term of \eqref{equ:succ-cross-term} as follows:
\begin{align}
    \sum_{t=t_0}^{T-1}\inp{\nabla\psi(w_{t+1})}{(c_{t+1} - c_t) \reg{\psi}} \nonumber
    &\ge \sum_{t=t_0+1}^{T} \frac{\beta \norm{w_t}_\psi^{\beta-1} + O(1)}{1+\alpha}(c_t - c_{t-1}) \nonumber\\
    &= \sum_{t=t_0+1}^{T} \beta (\norm{w_t}_\psi^{\beta-1} + O(1))(\norm{w_t}_\psi - \norm{w_{t-1}}_\psi) \nonumber\\
    &\ge \sum_{t=t_0+1}^{T} (\norm{w_t}_\psi^\beta - \norm{w_{t-1}}_\psi^\beta) + O(1) \cdot (\norm{w_T}_\psi - \norm{w_{t_0}}_\psi)\nonumber\\
    &= \norm{w_T}_\psi^\beta + O(\norm{w_T}_\psi) \label{equ:cross-telescoping}
\end{align}
where we defer the computation on the last inequality to Section~\ref{sec:main-thm-aux-pow}.

We now apply the inequality in \eqref{equ:cross-telescoping} to \eqref{equ:succ-cross-term}.
Note that $\psi(c_T \reg{}) = (1+\alpha)^\beta\norm{w_T}_\psi^\beta$.
We now have the following:
\begin{equation}
\label{equ:breg-upper}
    \brg{(1+\alpha)\norm{w_T}_\psi \reg{\psi}}{w_T} \le \norm{w_T}_\psi^\beta ((1+\alpha)^\beta - 1) + O(\norm{w_T}_\psi).
\end{equation}
After applying homogeneity of Bregman divergence, and define $\varepsilon \in (0, 1)$ so that $\alpha = \frac{\varepsilon}{1-\varepsilon}$, we have
\begin{equation}
\label{equ:breg-norm}
    \brg{\reg{\psi}}{(1-\varepsilon)\frac{w_T}{\norm{w_T}}_\psi} \le \frac{\norm{w_T}_\psi^\beta (1 - (1-\varepsilon)^\beta)}{\norm{w_T}_\psi^\beta} + o(1).
\end{equation}
Let $\tilde{w}_T = \frac{w_T}{\norm{w_T}_\psi}$.
We note that Bregman divergence in fact satisfies the Law of Cosines (see, e.g., \cite{azizan2019stochastic}):
\begin{lemma}[Law of Cosines]
\label{thm:breg-loc}
\begin{equation*}
    \brg{w}{w'} = \brg{w}{w''} + \brg{w''}{w'} - \inp{\nabla\psi(w') - \nabla\psi(w'')}{w - w''}
\end{equation*}
\end{lemma}

Therefore,

\begin{equation}
    \label{equ:final-limit}
    \begin{aligned}
    \brg{\reg{\psi}}{\tilde{w}_T} 
    &\le \frac{\norm{w_T}_\psi^\beta (1 - (1-\varepsilon)^{\beta})}{\norm{w_T}_\psi^\beta} + \brg{(1-\varepsilon)\tilde{w}_T}{\tilde{w}_T} \\
    &\hspace{6em} - \inp{\nabla\psi(\tilde{w}_T) - \nabla\psi((1-\varepsilon)\tilde{w}_T)}{\reg{\psi} - (1-\varepsilon)\tilde{w}_T} + o(1) \\
    &\le (1 - (1-\varepsilon)^{\beta}) + ((1-\varepsilon)^\beta - 1) + \beta\varepsilon + 2 \beta (1 - (1-\varepsilon)^{\beta-1}) + o(1)
    \end{aligned}
\end{equation}
And we defer the computation for the last inequality to Section~\ref{sec:main-thm-aux-pow}.
Taking the limit as $T \to \infty$, we have that
\begin{equation}
    \begin{aligned}
    \limsup_{T\to\infty} \brg{\reg{\psi}}{\frac{w_T}{\norm{w_T}_\psi}}
    &\le \beta\varepsilon + 2 \beta (1 - (1-\varepsilon)^{\beta-1})
    \end{aligned}
\end{equation}
Note that the RHS vanishes in the limit as $\varepsilon \to 0$.
Since the choice of $\varepsilon$ is arbitrary, we have $w_T/\norm{w_T}_\psi \to \reg{\psi}$ as $T \to\infty$.

\end{proof}

\begin{remark}
Because Bregman divergence is additive, we can extend this theorem to the case where the potential can be written in the form of $\psi = \psi_1 + \psi_2$, where $\psi_1$ and $\psi_2$ satisfy the conditions of Assumption~\ref{thm:assump-potential} with homogeneity constants $\beta_1$ and $\beta_2$, respectively, and $\beta_1 < \beta_2$.
In the end, all terms associated with $\psi_1$ would vanish because they are of lower order and we will be left with $D_{\psi_2}(\reg{\psi_2}, \frac{w_T}{\norm{w_T}_{\psi_2}}) \to 0$.
\end{remark}

\subsection{Auxiliary Computation for Section~\ref{sec:proof-primal-bias}}
\label{sec:main-thm-aux-pow}

To show \eqref{equ:p-norm-diff}, we claim that for $\delta \ge -1$ and $\beta > 1$, we have 
\[ \frac{\beta-1}{\beta} ((1+\delta)^\beta - 1) \ge (1+\delta)^{\beta-1} - 1. \]
Note that we have equality when $\delta=0$, and now we consider the first derivative:
\[\frac{d}{d\delta} \left\{\frac{\beta-1}{\beta} ((1+\delta)^\beta - 1) - (1+\delta)^{\beta-1} + 1 \right\} = (\beta-1)\delta(1+\delta)^{\beta-2},\]
which is negative when $\delta \in [-1, 0)$ and positive when $\delta > 0$, so this identity holds.
Now, \eqref{equ:p-norm-diff} follows from setting $\delta = (\norm{w_{t+1}}_\psi - \norm{w_t}_\psi)/\norm{w_{t}}_\psi$ and then multiplying by $\beta \cdot \norm{w_t}_\psi^\beta$ on both sides.

To finish showing \eqref{equ:cross-telescoping}, we claim that for $\delta \ge -1$ and $\beta > 1$, we have 
\[ \frac{1}{\beta}((1+\delta)^\beta - 1) \le \delta (1+\delta)^{\beta-1}. \]
Note that we have equality when $\delta=0$, and now we consider the first derivative:
\[\frac{d}{d\delta} \left\{\frac{1}{\beta}((1+\delta)^\beta - 1) - \delta (1+\delta)^{\beta-1}\right\} = -(\beta-1)\delta(1+\delta)^{\beta-2},\]
which is positive when $\delta \in [-1, 0)$ and negative when $\delta > 0$, so this identity holds.
Now, the last inequality of \eqref{equ:cross-telescoping} follows by setting $\delta = (\norm{w_{t}}_\psi - \norm{w_{t-1}}_\psi)/\norm{w_{t-1}}_\psi$ and then multiply by $\beta \cdot \norm{w_t}_\psi^\beta$ on both sides.

Finally, we simplify the RHS of \eqref{equ:final-limit} by taking advantage of the fact that $\tilde{w}_T$ is normalized:
\begin{align*}
    \brg{(1-\varepsilon)\tilde{w}_T}{\tilde{w}_T} 
    &= (1-\varepsilon)^\beta \psi(\tilde{w}_T) - \psi(\tilde{w}_T) + \inp{\nabla\psi(\tilde{w}_T)}{\varepsilon \tilde{w}_T} \\
    &= ((1-\varepsilon)^\beta - 1) + \beta\varepsilon.
\end{align*}
And note that for any vector $v$ and $\varepsilon \in (0, 1)$, we have
\begin{align*}
    \inp{\nabla\psi((1-\varepsilon)\tilde{w}_T)}{v}
    &= \frac{1}{1-\varepsilon} \inp{\nabla\psi((1-\varepsilon)\tilde{w}_T)}{(1-\varepsilon)v} \\
    &= \frac{1}{1-\varepsilon} \lim_{h\to 0} \frac{\psi((1-\varepsilon)\tilde{w}_T + h(1-\varepsilon)v) - \psi((1-\varepsilon)\tilde{w}_T)}{h} \\
    &= (1-\varepsilon)^{\beta-1} \lim_{h\to 0} \frac{\psi(\tilde{w}_T + hv) - \psi((1-\varepsilon)\tilde{w}_T)}{h} \\
    &= (1-\varepsilon)^{\beta-1} \inp{\nabla\psi(\tilde{w}_T)}{v}.
\end{align*}
It follows that
\begin{align*}
    &\left|\inp{\nabla\psi(\tilde{w}_T) - \nabla\psi((1-\varepsilon)\tilde{w}_T)}{\reg{\psi} - (1-\varepsilon)\tilde{w}_T}\right|\\
    ={}& (1-(1-\varepsilon)^{p-1}) \left|\inp{\nabla\psi(\tilde{w}_T)}{\reg{\psi} - (1-\varepsilon)\tilde{w}_T}\right|\\
    \le{}& (1 - (1-\varepsilon)^{\beta-1}) \left( |\inp{\nabla\psi(\tilde{w}_T)}{\reg{\psi}}| + |\inp{\nabla\psi(\tilde{w}_T)}{(1-\varepsilon)\tilde{w}_T}| \right)\\
    \le{}& (1 - (1-\varepsilon)^{\beta-1}) \left( |\inp{\nabla\psi(\tilde{w}_T)}{\tilde{w}_T}| + |\inp{\nabla\psi(\tilde{w}_T)}{(1-\varepsilon)\tilde{w}_T}| \right) \\
    \le{}& (1 - (1-\varepsilon)^{\beta-1}) \beta (2-\varepsilon) \le 2 \beta (1 - (1-\varepsilon)^{\beta-1})
\end{align*}
where the second to last line follows from \eqref{equ:mirror-parallel}.

\subsection{Proof of Proposition~\ref{thm:reg-max-dir}}
\label{sec:proof-reg-max-dir}
\begin{proof}
We first show that $\mmd{\psi}$ is unique.
Suppose the contrary that there are two distinct unit $\norm{\cdot}_\psi$-norm vectors $u_1 \neq u_2$ both achieving the maximum-margin $\mar{\psi}$.
Then $u_3 = (u_1 + u_2)/2$ satisfies
\[ \forall i, \; y_i \inp{u_3}{x_i} = \frac{1}{2} y_i \inp{u_1}{x_i} + \frac{1}{2} y_i \inp{u_2}{x_i} \ge \mar{\psi} \]
Therefore, $u_3$ has margin of at least $\mar{\psi}$.
Since $\psi$ is strictly convex, we must have $\norm{u_3}_\psi < 1$.
Therefore, the margin of $u_3 / \norm{u_3}_\psi$ is strictly greater than $\mar{\psi}$, contradiction.

Define $B' > 0$ so that $\ell(z) e^{az} \in [b/2, 2b]$ for any $z = B\mar{\psi}$ where $B > B'$.
Note that
\[L(B\mmd{\psi}) = \sum_{i=1}^n \ell(y_i\inp{B\mmd{\psi}}{x_i}) \le n \cdot \ell(B\mar{\psi}) \le 2 b n \cdot \exp(-aB\mar{\psi})\]

Suppose the contrary that the regularized direction does not converge to $\mmd{\psi}$, then there must exist $\varepsilon \in (0, \mar{\psi}/2)$ so that there are arbitrarily large values of $B$ satisfying
\[\min_{i=1, \dots, n} y_i \inp{\frac{\bar{w}(B)}{B}}{x_i} \le \mar{\psi} - \varepsilon. \]
And this implies
\[ L(\bar{w}(B)) \ge \ell(B(\mar{\psi} - \varepsilon)) \ge \frac{b}{2} \exp(-a B\mar{\psi}) \exp(a B\varepsilon)\]

Then, for sufficiently large $B > B'$, we have $\exp(aB\varepsilon) > 4 n \Rightarrow L(\bar{w}(B)) > L(B\mmd{\psi})$, contradiction.
Therefore, the regularized direction exists and $\reg{\psi} = \mmd{\psi}$.
\end{proof}
 
\section{Proofs for Section~\ref{sec:asymp-result}}
\label{sec:proof-asymp-result}
\subsection{Proof of Corollary~\ref{thm:convg-rate}}
\begin{proof}
    This is an immediate consequence of \eqref{equ:breg-upper} and \eqref{equ:breg-norm}.
\end{proof}

\subsection{Proof of Lemma~\ref{thm:norm-rate}}
For the following proof, we assume without loss of generality that $y_i = 1$ by replacing every instance of $(x_i, -1)$ with $(-x_i, 1)$.

\begin{proof}
    For the upper bound, we consider a reference vector $w^\star = \mar{\psi}^{-1} \mmd{\psi}$.
    By the definition of the max-margin direction, the margin of $w^\star$ is 1 and $\norm{w^\star}_\psi = \mar{\psi}^{-1}$.
    From Lemma~\ref{thm:key-iden}, we have
    \begin{align*} D_\psi(w^\star \log T, w_t) = D_\psi(w^\star \log T, w_{t+1}) + D_{\psi - \eta L}(w_{t+1}, w_{t}) &- \inp{\nabla L(w_t)}{w^\star \log T - w_t} \\
    &- \eta L(w_{t}) + \eta L(w_{t+1}). 
    \end{align*}
    We first bound the quantity $\inp{\nabla L(w_t)}{w^\star \log T - w_t}$ by expanding the definition of exponential loss:
    \begin{align*}
        &\inp{\nabla L(w_t)}{w^\star \log T - w_t} \\
        ={}& \sum_{i=1}^n \inp{\nabla_w \exp(-\inp{w}{x_i}) \bigg|_{w=w_t}}{w^\star \log T - w_t} \\
        ={}& \sum_{i=1}^n \inp{\exp(-\inp{w_t}{x_i})x_i}{w_t - w^\star \log T} \\
        ={}& \sum_{i=1}^n \exp(-\inp{w^\star \log T}{x_i}) \exp(-\inp{w_t - w^\star \log T}{x_i})  \inp{x_i}{w_t - w^\star \log T} \\
        \le{}& \sum_{i=1}^n \frac{1}{T} \cdot \frac{1}{e} = \frac{n}{eT}
    \end{align*}
    where the last line follows from the definition of $w^\star$ and the fact that for any $x \in \RR$, we have $e^{-x} x \le 1/e$.
    It follows that
    \[ D_\psi(w^\star \log T, w_t) \ge D_\psi(w^\star \log T, w_{t+1}) - \frac{n}{eT} - \eta L(w_{t}) + \eta L(w_{t+1}). \]
    
    Summing over $t = 0, \dots, T-1$ gives us
    \[ D_\psi(w^\star \log T, w_0) \ge D_\psi(w^\star \log T, w_T) - \frac{n}{e} - \eta L(w_0) + \eta L(w_T). \]
    Due to the homogeneity of Bregman divergence with respect to the $\beta$-absolutely homogeneous potential $\psi$, we can divide by a factor of $\log^\beta T$ on both sides:
    \begin{equation}
    \label{equ:lim-reference-scale}
        D_\psi\left(w^\star, \frac{w_0}{\log T}\right) \ge D_\psi\left(w^\star, \frac{w_T}{\log T}\right) - o(1).
    \end{equation}
    As $T\to\infty$, the left-hand side converges to $\brg{w^\star}{0}  = \psi(w^\star) = \mar{\psi}^{-\beta}$.
    Let $\tilde{w} = w_T / \log T$, we expand the right-hand side as
    \begin{align*}
        \brg{w^\star}{\tilde{w}}
        &= \psi(w^\star) - \psi(\tilde{w}) - \inp{\nabla\psi(\tilde{w})}{w^\star - \tilde{w}} \\
        &= \mar{\psi}^{-\beta} + (\beta-1) \norm{\tilde{w}}_\psi^\beta - \inp{\nabla\psi(\tilde{w})}{w^\star} \\
        &\ge \mar{\psi}^{-\beta} + (\beta-1) \norm{\tilde{w}}_\psi^\beta - \frac{\mar{\psi}^{-1}}{\norm{\tilde{w}}_\psi}|\inp{\nabla\psi(\tilde{w})}{\tilde{w}}| \\
        &= \mar{\psi}^{-\beta} + (\beta-1) \norm{\tilde{w}}_\psi^\beta - \beta \mar{\psi}^{-1}\norm{\tilde{w}}_\psi^{\beta-1}
    \end{align*}
    where the inequality follows from \eqref{equ:mirror-parallel}.
    
    If $\norm{w_T / \log T}_\psi > \mar{\psi}^{-1} \cdot \frac{\beta}{\beta-1}$ for arbitrarily large $T$, then $\brg{w^\star}{w_T / \log T} > \mar{\psi}^{-\beta}$ for those $T$.
    This in turn contradicts inequality \eqref{equ:lim-reference-scale}.
    Therefore, we must have 
    \[ \limsup_{T\to\infty} \frac{\norm{w_T}_\psi}{\log T} \le \mar{\psi}^{-1} \frac{\beta}{\beta-1}. \]
    
    Now we can turn our attention to the lower bound.
    Let $m_t = \gamma(w_t)$ be the margin of the mirror descent iterates.
    Then, 
    \[ L(w_t) = \frac{1}{n} \sum_{i=1}^n \exp(-\inp{w_t}{x_i}) \ge \frac{1}{n} \exp(-m_t).\]
    Due to Lemma~\ref{thm:to-infinity}, we also know that $m_t \xrightarrow{t\to\infty} \infty$.
    
    By the definition of the max-margin direction, we know that $\gamma(\norm{w_t}_\psi \mmd{\psi}) \ge m_t$.
    Then by linearity of margin, there exists $w^\star$ so that $\gamma(w^\star) = (1+\frac{\log(2n)}{m_t}) m_t$ and $\norm{w^\star}_\psi \le (1+\frac{\log(2n)}{m_t}) \norm{w_t}_\psi$.
    It follows that
    \[ L(w^\star) = \frac{1}{n} \sum_{i=1}^n \exp(-\inp{w^\star}{x_i}) \le \exp(-\gamma(w^\star)) = \frac{1}{2n} \exp(-m_t).\]
    
    Under the assumption that the step size $\eta$ is sufficiently small so that $\psi - \eta L$ is convex on the iterates, we can apply the standard convergence rate of mirror descent ~\citep[Theorem 3.1]{lu2018relatively}:
    \[ L(w_t) - L(w^\star) \le \frac{1}{\eta t} \brg{w^\star}{w_0} \]
    From our choice of $w^\star$, we have
    \begin{align*}
        \frac{1}{2n} \exp(-m_t)
        &\le \frac{1}{\eta t} \brg{w^\star}{w_0} \\
        &= \frac{1}{\eta t} (\psi(w^\star) - \psi(w_0) - \inp{\nabla\psi(w_0)}{w^\star - w_0})
    \end{align*}
    After dropping the lower order terms and recalling the upper bounds on $\norm{w^\star}_\psi$ and $\norm{w_t}_\psi$, we have
    \[ \frac{1}{2n} \exp(-m_t) \le O(1) \cdot \frac{1}{\eta t} \cdot \left(1 + \frac{\log (2n)}{m_t}\right)^\beta \left(\mar{\psi}^{-1} \frac{\beta}{\beta-1} \log t\right)^\beta\]
    Since $m_t$ is unbounded, the quantity $1 + \frac{\log (2n)}{m_t}$ is upper bounded by a constant.
    Taking the logarithm on both sides yields
    \[ m_t \ge \log t - \beta \log\log t + O(1)\]
    
    Finally, we use the definition of margin to conclude that $ m_t \le \inp{w_t}{x_i} \le C \cdot \norm{w_t}_\psi$.
    Therefore, 
    \[ \norm{w_t}_\psi \ge \frac{1}{C} (\log t - \beta \log\log t) + O(1).\]
\end{proof} 
\section{Proofs for Section~\ref{sec:normalized-asymp}}
\label{sec:proof-normalized-asymp}
\subsection{Proof of Lemma~\ref{thm:normalized-norm-upper}}
This proof follows the same technique as Lemma~\ref{thm:norm-rate}.
For simplicity, we assume without loss of generality that $y_i = 1$ for all $i$ by replacing every data point of the form $(x_i, -1)$ with $(-x_i, 1)$.
\begin{proof}
    We consider a reference vector $w^\star = \mar{\psi}^{-1} \mmd{\psi}$.
    By the definition of the max-margin direction, the margin of $w^\star$ is 1 and $\norm{w^\star}_\psi = \mar{\psi}^{-1}$.   
    From Lemma~\ref{thm:key-iden}, we have
    \begin{align*} D_\psi(w^\star \sqrt{T}, w_t) = D_\psi(w^\star \sqrt{T}, w_{t+1}) + D_{\psi - \eta_t L}(w_{t+1}, w_{t}) &- \inp{\nabla L(w_t)}{w^\star \sqrt{T} - w_t} \\
    &- \eta_t L(w_{t}) + \eta_t L(w_{t+1}). 
    \end{align*}
    We first bound the quantity $\inp{\nabla L(w_t)}{w^\star \sqrt{T} - w_t}$ by expanding the definition of exponential loss:
    \begin{align*}
        &\inp{\nabla L(w_t)}{w^\star \sqrt{T} - w_t} \\
        ={}& \sum_{i=1}^n \inp{\nabla_w \exp(-\inp{w}{x_i}) \bigg|_{w=w_t}}{w^\star \sqrt{T} - w_t} \\
        ={}& \sum_{i=1}^n \inp{\exp(-\inp{w_t}{x_i})x_i}{w_t - w^\star \sqrt{T}} \\
        ={}& \sum_{i=1}^n \exp\left(-\inp{w^\star \sqrt{T}}{x_i}\right) \exp\left(-\inp{w_t - w^\star \sqrt{T}}{x_i}\right)  \inp{x_i}{w_t - w^\star \sqrt{T}} \\
        \le{}& \sum_{i=1}^n \exp(-\sqrt{T}) \cdot \frac{1}{e} \in o\left(\frac{1}{T}\right)\,,
    \end{align*}
    where the last line follows from the definition of $w^\star$ and the fact that for any $x \in \RR$, we have $e^{-x} x \le 1/e$.
    It follows that
    \begin{align*} 
        D_\psi(w^\star \sqrt{T}, w_t) 
        &\ge D_\psi(w^\star \sqrt{T}, w_{t+1}) - o(1/T) - \eta_t L(w_{t}) + \eta_t L(w_{t+1}) \\
        &\ge D_\psi(w^\star \sqrt{T}, w_{t+1}) - o(1/T) - \frac{\eta_0}{\sqrt{t+1}}\,.
    \end{align*}
    
    Summing over $t = 0, \dots, T-1$ gives us
    \[ D_\psi(w^\star \sqrt{T}, w_0) \ge D_\psi(w^\star \sqrt{T}, w_T) - O(\sqrt{T}). \]
    Due to the homogeneity of Bregman divergence with respect to the $\beta$-absolutely homogeneous potential $\psi$, we can divide by a factor of $T^{\beta/2}$ on both sides:
    \begin{equation}
    \label{equ:normalized-lim-reference-scale}
        D_\psi\left(w^\star, \frac{w_0}{\sqrt{T}}\right) \ge D_\psi\left(w^\star, \frac{w_T}{\sqrt{T}}\right) - o(1).
    \end{equation}
    As $T\to\infty$, the left-hand side converges to $\brg{w^\star}{0}  = \psi(w^\star) = \mar{\psi}^{-\beta}$.
    Let $\tilde{w} = w_T / \sqrt{T}$, we expand the right-hand side as
    \begin{align*}
        \brg{w^\star}{\tilde{w}}
        &= \psi(w^\star) - \psi(\tilde{w}) - \inp{\nabla\psi(\tilde{w})}{w^\star - \tilde{w}} \\
        &= \mar{\psi}^{-\beta} + (\beta-1) \norm{\tilde{w}}_\psi^\beta - \inp{\nabla\psi(\tilde{w})}{w^\star} \\
        &\ge \mar{\psi}^{-\beta} + (\beta-1) \norm{\tilde{w}}_\psi^\beta - \frac{\mar{\psi}^{-1}}{\norm{\tilde{w}}_\psi}|\inp{\nabla\psi(\tilde{w})}{\tilde{w}}| \\
        &= \mar{\psi}^{-\beta} + (\beta-1) \norm{\tilde{w}}_\psi^\beta - \beta \mar{\psi}^{-1}\norm{\tilde{w}}_\psi^{\beta-1}\,,
    \end{align*}
    where the inequality follows from \eqref{equ:mirror-parallel}.
    
    If $\norm{w_T / \sqrt{T}}_\psi > \mar{\psi}^{-1} \cdot \frac{\beta}{\beta-1}$ for arbitrarily large $T$, then $\brg{w^\star}{w_T / \sqrt{T}} > \mar{\psi}^{-\beta}$ for those $T$.
    This in turn contradicts inequality \eqref{equ:normalized-lim-reference-scale}.
    Therefore, we must have 
    \[ \limsup_{T\to\infty} \frac{\norm{w_T}_\psi}{\sqrt{T}} \le \mar{\psi}^{-1} \frac{\beta}{\beta-1}\,, \]
    as desired.
\end{proof}

\subsection{Proof of Lemma~\ref{thm:normalized-norm-rate}}
\begin{proof}
    Given sufficiently small $\eta_0$, we have that $\psi - \eta_t L$ is convex.
    From relatively smoothness~\citep[Proposition 1.1]{lu2018relatively}, we have
    \[L(w_{t+1}) \le L(w_t) + \inp{\nabla L(w_t)}{w_{t+1}-w_t} + \frac{1}{\eta_t} D_\psi(w_{t+1}, w_t).\]
    By applying the mirror descent update rule~\eqref{equ:md} to the RHS, the following holds for any $w$:
    \begin{equation} 
    \label{equ:rel-smooth-descent}
    L(w_{t+1}) \le L(w_t) + \inp{\nabla L(w_t)}{w-w_t} + \frac{1}{\eta_t} D_\psi(w, w_t).
    \end{equation}
    Let $\Delta w$ be the vector satisfying 
    \[ -\inp{\nabla L(w_t)}{\Delta w} = \norm{\nabla L(w_t)}_{\psi, *} \norm{\Delta w}_\psi = \norm{\nabla L(w_t)}_{\psi, *}^2 = \norm{\Delta w}_\psi^2.\]
    Then, we choose the vector $w$ so that $w - w_t = \eta_t (t+1)^{-c/2} \Delta w$ for some constant $c > 0$ which we will determine later.

    Next, we bound the Bregman divergence with Lagrange's remainder:
    \[\psi(w) \le \psi(w_t) + \inp{\nabla\psi(w_t)}{w-w_t} + \underbrace{\frac{1}{2} \sup_{\lambda \in (0, 1)} (w - w_t)^\top \nabla^2\psi(w_t + \lambda(w-w_t)) (w-w_t)}_R\,.\]
    Note that by construction, $\norm{\Delta w_t}_\psi = \norm{\nabla L(w_t)}_{\psi, *} \le C \cdot L(w_t)$.
    Hence, it holds that
    \[\norm{w - w_t}_\psi = \eta_t (t+1)^{-c/2} \norm{\Delta w}_\psi \le C \cdot \eta_0\,.\]
Next, when $L(w_0) \le \frac{1}{2n}$, we have
    \begin{align*}
        & L(w_t)  = \frac{1}{n} \sum_{i=1}^n \exp(-w_t^\top x_i) \ge \frac{1}{n} \exp(-\min_i w_t^\top x_i), \\
        \implies & -\log L(w_t) \le \min_i w_t^\top x_i + \log n \le C \cdot \norm{w_t}_\psi + \log n, \\
        \implies & \norm{w_t}_\psi \ge \log 2 / C.
    \end{align*}
    Also, from Assumption~\ref{thm:assump-potential}, we can show that for any vector $w$, the norms $\norm{w}_2$ and $\norm{w}_\psi$ differ by up to a constant factor:
    \begin{equation}
        \label{equ:gen-norm-l2}
        \left(\inf_{\norm{w}_2 = 1} \psi(w)\right)^{-1/\beta} \norm{w}_\psi \le \norm{w}_2 \le \left(\sup_{\norm{w}_2 = 1} \psi(w)\right)^{1/\beta} \norm{w}_\psi\,.
    \end{equation}
    Therefore, for sufficiently small $\eta_0$, we have $\norm{w - w_t}_2 \le 2 \norm{w_t}_2$.
    Applying Lemma~\ref{thm:hessian-scaling} yields
    \begin{align*}
        R \le \frac{1}{2} \norm{w - w_t}_2^2 (2\norm{w_t}_2)^{\beta-2} \sup_{\norm{v}_2 = 1} \norm{\nabla^2 \psi(v)}_2.
    \end{align*}
    Invoking \eqref{equ:gen-norm-l2} again, there exists a constant $B > 0$ so that
    \begin{align*}
        R \le \frac{B}{2} \norm{w - w_t}_\psi^2 \norm{w_t}_\psi^{\beta-2} = \frac{B}{2} \eta_t^2 (t+1)^{-c}\norm{\Delta w}_\psi^2 \norm{w_t}_\psi^{\beta-2}\,.
    \end{align*}
    Because $D_\psi(w, w_t) \le R$, we return to \eqref{equ:rel-smooth-descent} and conclude that
    \begin{align*}
        L(w_{t+1})
        & \le L(w_t) - \eta_t (t+1)^{-c/2} \norm{\nabla L(w_t)}_{\psi, *}^2 + \frac{B}{2} \eta_t (t+1)^{-c}\norm{\Delta w}_\psi^2 \norm{w_t}_\psi^{\beta-2} \\
        & = L(w_t) - \eta_t (t+1)^{-c/2} \norm{\nabla L(w_t)}_{\psi, *}^2 + \frac{B}{2} \eta_t (t+1)^{-c}\norm{\nabla L(w_t)}_{\psi, *}^2 \norm{w_t}_\psi^{\beta-2} \\
        & = L(w_t) \left(1 - \eta_0 (t+1)^{-(c+1)/2} \delta_t + \frac{B}{2} \eta_0 (t+1)^{-(2c+1)/2} \delta_t \norm{w_t}_\psi^{\beta-2} \right) \\
        & \le L(w_t) \exp\left(- \eta_0 (t+1)^{-(c+1)/2} \delta_t + \frac{B}{2} \eta_0 (t+1)^{-(2c+1)/2} \delta_t \norm{w_t}_\psi^{\beta-2} \right)\,,
    \end{align*}
    where we let $\delta_t = \left(\frac{\norm{\nabla L(w_t)}_{\psi, *}}{L(w_t)}\right)^2$.
    Since $C \cdot \norm{w}_\psi \ge \min_i w^\top x_i \ge -\log L(w)$, we propagate the previous inequality through $t = 0, \dots, T-1$ to get
    \begin{align*}
        \norm{w_T}_\psi
        &\ge -\frac{1}{C} \log L(w_T) \\
        &\ge \frac{1}{C} \left(\sum_{t=0}^{T-1} \left(\eta_0 (t+1)^{-(c+1)/2} \delta_t - \frac{B}{2} \eta_0 (t+1)^{-(2c+1)/2} \delta_t \norm{w_t}^{\beta-2} \right) - \log L(w_0)\right) \\
        &\ge \Theta\left(\sum_{t=0}^{T-1} t^{-(c+1)/2} - t^{-(2c+1)/2} \norm{w_t}_\psi^{\beta-2} \right),
    \end{align*}
    where the second inequality holds as $\mar{\psi} L(w) \le \norm{\nabla L(w)}_{\psi, *} \le C \cdot L(w) \implies \delta_t$ bounded.

    Recall that Lemma~\ref{thm:normalized-norm-upper} gives us $\norm{w_t}_\psi \in O(t^{-1/2})$.
    Picking $c = \max(\beta-2, 0) + 2\zeta$ for arbitrarily small $\zeta > 0$ gives us
    \[\norm{w_t}_\psi \in \Omega\left(\sum_{t=0}^{T-1} t^{(1-\beta)/2 - \zeta} -  t^{(1-\beta)/2 - 2\zeta}\right) \subseteq \Omega\left(t^{(3-\beta)/2-\zeta}\right),\]
    as desired.
\end{proof}

\subsection{Proof of Theorem~\ref{thm:normalized-md-rate}}
This proof mostly follows the same steps as that of Theorem~\ref{thm:primal-bias}.
\begin{proof}
Consider arbitrary $\alpha \in (0, 1)$ and define $r_\alpha$ according to Lemma~\ref{thm:approx-reg-dir-loss}.
Since $\lim_{t\to\infty}\norm{w_t}_\psi = \infty$, we can find $t_0$ so that $\norm{w_t}_\psi > \max(1, r_\alpha)$ for all $t \ge t_0$.
Let $c_t = (1+\alpha)\norm{w_t}_\psi$.

Substitute $w = c_t \reg{\psi}$ into Lemma~\ref{thm:key-iden}, we get
\[\brg{c_t\reg{\psi}}{w_{t+1}} \le \brg{c_t\reg{\psi}}{w_t} + \eta_t \inp{\nabla L(w_t)}{c_t \reg{\psi} - w_t} - \eta L(w_{t+1}) + \eta L(w_t).\]
By Corollary~\ref{thm:cross-term}, we have $\inp{\nabla L(w_t)}{c_t \reg{\psi} - w_t} \le 0$.
Therefore,
\begin{align*}
    \brg{c_t\reg{\psi}}{w_{t+1}} 
    &\le \brg{c_t\reg{\psi}}{w_t} - \eta_t L(w_{t+1}) + \eta_t L(w_t) \\
    &\le \brg{c_t\reg{\psi}}{w_t} + \eta_0 (t+1)^{-1/2}.
\end{align*}

It follows that
\begin{align*}
    &\brg{c_{t+1}\reg{\psi}}{w_{t+1}} \\
    \le{}& \brg{c_t\reg{\psi}}{w_t} + \eta_0 (t+1)^{-1/2} + \brg{c_{t+1}\reg{\psi}}{w_{t+1}} - \brg{c_t\reg{\psi}}{w_{t+1}} \\
    ={}& \brg{c_t\reg{\psi}}{w_t} + \eta_0 (t+1)^{-1/2} + \psi(c_{t+1} \reg{\psi}) - \psi(c_t \reg{\psi}) - \inp{\nabla\psi(w_{t+1})}{(c_{t+1} - c_t) \reg{\psi}}
\end{align*}
Summing over $t = t_0, \dots, T-1$ gives us
\begin{align}
    \brg{c_T\reg{\psi}}{w_T}
    &\le \brg{c_{t_0}\reg{\psi}}{w_{t_0}} + \sum_{t=t_0}^{T-1} \eta_0 (t+1)^{-1/2} + \psi(c_{T} \reg{\psi}) - \psi(c_{t_0} \reg{\psi}) \nonumber\\
    &\quad\quad\quad- \sum_{t=t_0}^{T-1}\inp{\nabla\psi(w_{t+1})}{(c_{t+1} - c_t) \reg{\psi}} \nonumber\\
    &= \brg{c_{t_0}\reg{\psi}}{w_{t_0}} + O(\sqrt{T}) + \psi(c_{T} \reg{\psi}) - \psi(c_{t_0} \reg{\psi}) \nonumber\\
    &\quad\quad\quad- \sum_{t=t_0}^{T-1}\inp{\nabla\psi(w_{t+1})}{(c_{t+1} - c_t) \reg{\psi}}
    \label{equ:normalized-succ-cross-term}
\end{align}

Now we want to establish a lower bound on the last term of \eqref{equ:normalized-succ-cross-term}.
To do so, we inspect the change in $\nabla\psi(w_t)$ from each successive mirror descent update:
\begin{subequations}
\begin{align}
    &\inp{\nabla\psi(w_{t+1}) - \nabla\psi(w_t)}{\reg{\psi}} \\
    ={}& \inp{-\eta_t\nabla L(w_{t})}{\reg{\psi}} \\
    \ge{}& \frac{1}{(1+\alpha)\norm{w_t}_\psi} \inp{-\eta_t\nabla L(w_t)}{w_t} \label{equ:normalized-cross-term-expansion-L1}\\
    \ge{}& \frac{1}{(1+\alpha)\norm{w_t}_\psi} \left((\beta-1) \norm{w_{t+1}}_\psi^\beta - (\beta-1) \norm{w_{t}}_\psi^\beta + \eta_t L(w_{t+1}) - \eta_t L(w_{t})\right) \label{equ:normalized-cross-term-expansion-L2} \\
    \ge{}& \frac{1}{(1+\alpha)\norm{w_t}_\psi} \left((\beta-1) \norm{w_{t+1}}_\psi^\beta - (\beta-1) \norm{w_{t}}_\psi^\beta\right) - \frac{\eta_t L(w_{t})}{\norm{w_t}_\psi} \label{equ:normalized-cross-term-expansion-L3}
\end{align}
\end{subequations}
where we applied Corollary~\ref{thm:cross-term} on \eqref{equ:normalized-cross-term-expansion-L1} and Lemma~\ref{thm:cross-term-diff} on \eqref{equ:normalized-cross-term-expansion-L2}.

Now we bound \eqref{equ:normalized-cross-term-expansion-L3}.
We claim the following identity and defer its derivation to Section~\ref{sec:main-thm-aux-pow}.
\begin{equation}
    (\beta-1)(\norm{w_{t+1}}_\psi^\beta - \norm{w_t}_\psi^\beta) \ge \beta (\norm{w_{t+1}}_\psi^{\beta-1} - \norm{w_{t}}_\psi^{\beta-1}) \norm{w_t}_\psi.
\end{equation}

From Lemma~\ref{thm:normalized-norm-rate}, we have that for any $\zeta > 0$, $\norm{w_t}_\psi \in \Omega(t^{(3-\beta)/2 - \zeta})$.
Therefore,
\[\frac{\eta_t L(w_t)}{\norm{w_t}_\psi} = \frac{\eta_0 (t+1)^{-1/2}}{\norm{w_t}_\psi} \in O(t^{-2 + \beta/2 + \zeta}).\]
We are left with
\begin{equation*} \inp{\nabla\psi(w_{t+1}) - \nabla\psi(w_t)}{\reg{\psi}} \ge \beta \cdot \frac{\norm{w_{t+1}}_\psi^{\beta-1} - \norm{w_{t}}_\psi^{\beta-1}}{1+\alpha} - O(t^{-2 + \beta/2 + \zeta}).
\end{equation*}

Summing over $t = t_0, \dots, T-1$ gives us
\begin{equation}\label{equ:normalized-mirror-cross-term}
\inp{\nabla\psi(w_{T}) - \nabla\psi(w_{t_0})}{\reg{\psi}} \ge \beta \cdot \frac{\norm{w_T}_\psi^{\beta-1} - \norm{w_0}_\psi^{\beta-1}}{1+\alpha} - O(T^{-1 + \beta/2 + \zeta}).
\end{equation}

With \eqref{equ:normalized-mirror-cross-term}, we can bound the last term of \eqref{equ:succ-cross-term} as follows:
\begin{align}
    & \sum_{t=t_0}^{T-1}\inp{\nabla\psi(w_{t+1})}{(c_{t+1} - c_t) \reg{\psi}} \nonumber \\
    \ge{}& \sum_{t=t_0+1}^{T} \frac{\beta \norm{w_t}_\psi^{\beta-1} - O(t^{-1 + \beta/2 + \zeta})}{1+\alpha}(c_t - c_{t-1}) \nonumber\\
    ={}& \sum_{t=t_0+1}^{T} \beta (\norm{w_t}_\psi^{\beta-1} - O(t^{-1 + \beta/2 + \zeta}))(\norm{w_t}_\psi - \norm{w_{t-1}}_\psi) \nonumber\\
    \ge{}& \sum_{t=t_0+1}^{T} (\norm{w_t}_\psi^\beta - \norm{w_{t-1}}_\psi^\beta) - O(T^{-1 + \beta/2 + \zeta}) \cdot (\norm{w_T}_\psi - \norm{w_{t_0}}_\psi)\nonumber\\
    ={}& \norm{w_T}_\psi^\beta - O(T^{-1 + \beta/2 + \zeta} \norm{w_T}_\psi) \label{equ:normalized-cross-telescoping}
\end{align}
where we defer the computation on the last inequality to Section~\ref{sec:main-thm-aux-pow}.

We now apply the inequality in \eqref{equ:normalized-cross-telescoping} to \eqref{equ:normalized-succ-cross-term}.
Note that $\psi(c_T \reg{}) = (1+\alpha)^\beta\norm{w_T}_\psi^\beta$.
We now have the following:
\[ \brg{(1+\alpha)\norm{w_T}_\psi \reg{\psi}}{w_T} \le \norm{w_T}_\psi^\beta ((1+\alpha)^\beta - 1) + O(\sqrt{T} + T^{-1 + \beta/2 + \zeta} \norm{w_T}_\psi).\]
After applying homogeneity of Bregman divergence, and define $\varepsilon \in (0, 1)$ so that $\alpha = \frac{\varepsilon}{1-\varepsilon}$, we have
\[ \brg{\reg{\psi}}{(1-\varepsilon)\frac{w_T}{\norm{w_T}}_\psi} \le \frac{\norm{w_T}_\psi^\beta (1 - (1-\varepsilon)^\beta)}{\norm{w_T}_\psi^\beta} + O\left(\frac{\sqrt{T} + T^{-1 + \beta/2 + \zeta} \norm{w_T}_\psi}{\norm{w_T}_\psi^\beta}\right).\]
For the remainder term to vanish, we apply Lemma~\ref{thm:normalized-norm-rate} and want that for sufficiently small $\zeta> 0$,
\[
\begin{cases}
    \beta ((3-\beta)/2 - \zeta) &> 1/2, \\
    (\beta-1) ((3-\beta)/2 - \zeta) &> -1 + \beta/ 2 + \zeta.
\end{cases}
\]
Solving the system gives that we need $1 < \beta < \frac{1}{2} (3 + \sqrt{5})$.
Therefore, for $1 < \beta < \frac{1}{2} (3 + \sqrt{5})$, we have
\[ \brg{\reg{\psi}}{(1-\varepsilon)\frac{w_T}{\norm{w_T}}_\psi} \le \frac{\norm{w_T}_\psi^\beta (1 - (1-\varepsilon)^\beta)}{\norm{w_T}_\psi^\beta} + o(1).\]
The lower-order term can be more precisely written as 
\[O\left(T^{-(3\beta - 1 - \beta^2)/2 + \beta\zeta}\right),\]
which gives us the rate at which the error term vanishes.
Now, to finish showing convergence, the rest of the proof follows exactly as that of Theorem~\ref{thm:primal-bias}.
\end{proof} 
\clearpage

\section{Practicality of \algname}
\label{sec:practicality}
To illustrate that \algname can be easily implemented, we show a proof-of-concept implementation in PyTorch.
This implementation can directly replace existing optimizers and thus require only minor changes to any existing training code. 

We also note that while the \algname update step requires more arithmetic operations than a standard gradient descent update, this does not significantly impact the total runtime because differentiation is the most computationally intense step.
We observed from our experiments that training with \algname is approximately 10\% slower than with PyTorch's \texttt{optim.SGD} (in the same number of epochs),\footnote{This measurement may not be very accurate because we were using shared computing resources.} and we believe that this gap can be closed with optimization.

\begin{lstlisting}[caption={Sample PyTorch implementation of \algname}]
import torch
from torch.optim import Optimizer

class pnormSGD(Optimizer):
    def __init__(self, params, lr=0.01, pnorm=2.0):
        # p-norm must be strictly greater than 1
        if not 1.01 <= pnorm:
            raise ValueError("Invalid p-norm value: {}".format(pnorm))

        defaults = dict(lr=lr, pnorm=pnorm)
        super(pnormSGD, self).__init__(params, defaults)

    def __setstate__(self, state):
        super(pnormSGD, self).__setstate__(state)

    def step(self, closure=None):
        loss = None
        if closure is not None:
            with torch.enable_grad():
                loss = closure()

        for group in self.param_groups:
            lr = group["lr"]
            pnorm = group["pnorm"]

            for param in group["params"]:
                if param.grad is None:
                    continue

                x, dx = param.data, param.grad.data

                # \ell_p^p potential function
                update = torch.pow(torch.abs(x), pnorm-1) * \
                                torch.sign(x) -  lr * dx           
                param.data =  torch.sign(update) * \
                            torch.pow(torch.abs(update), 1/(pnorm-1))

        return loss
\end{lstlisting} 

\section{Experimental Details}
\label{sec:experiment-detail}
All of the following experiments were performed on compute nodes equipped with an Intel Skylake CPU + one Nvidia V100 GPU.
The experiments involving linear models were CPU only and experiments with convolutional network models took advantage of the GPU's acceleration.

\subsection{Linear classification}
Here, we describe the details behind our experiments from Section~\ref{sec:linear-classifier}.
First, we note that we can absorb the labels $y_i$ by replacing $(x_i, y_i)$ with $(y_ix_i, 1)$.
This way, we can choose points with the same $+1$ label.

For the $\RR^2$ experiment, we first select three points $(\frac{1}{6}, \frac{1}{2}), (\frac{1}{2}, \frac{1}{6})$ and $(\frac{1}{3}, \frac{1}{3})$ so that the maximum margin direction is approximately $\frac{1}{\sqrt{2}}(1, 1)$.
Then we sample 12 additional points from $\mathcal{N}((\frac{1}{2}, \frac{1}{2}), 0.15 I_2)$.
The initial weight $w_0$ is selected from $\mathcal{N}(0, I_2)$.
We ran \algname with fixed step size $10^{-3}$ for 1 million steps.
As for the scatter plot of the data, we randomly re-assign a label and plot out $(x_i, 1)$ or $(-x_i, -1)$ uniformly at random.

For the $\RR^{100}$ experiment, we select 15 sparse vectors that each have up to 10 nonzero entries.
Each nonzero entry is i.i.d. sampled from $\mathcal{U}(-2, 4)$.
Because we are in the over-parameterized case, these vectors are linearly separable with high probability.
The initial weight $w_0$ is selected from $\mathcal{N}(0, 0.1 I_{100})$.
We ran \algname with step size $10^{-4}$ for 1 million steps.

\subsection{Normalized MD experiements}
For the first experiment with the synthetic dataset in $\RR^2$, we use the same choices of hyper-parameters as above.
The only difference is that we only run for 25000 iterations due to faster convergence of normalized MD.

For the MNIST experiments, we used two different models, the first one is a two-layer fully connected neural network with 300 neurons in the hidden layer and ReLU activation and the second one is a convolutional network with two convolution layers (see exact specification next page).
For both models, we applied cross-entropy loss and batch size of 512.
To avoid numerical issues in the normalized MD update \eqref{equ:normalized-md-rescaled}, we divide by $\max(L(w_t), 10^{-5})$ instead of the empirical loss directly.

For the fully connected network, we train for 200 epochs in total and use a learning rate schedule that starts with $\eta = 0.1$ and decays by a factor of 5 at the 120th, 150th, and 180th epochs.
In the normalized MD update \eqref{equ:normalized-md-rescaled}, the base step size $\eta_0$ follows the same schedule and has scale factor $\lambda = 0.1$.
These parameters were chosen to closely match the setup in~\cite[Section 4.2]{nacson2019convergence} when $p = 2$, but our experiments used mini-batch instead to better reflect a practical training scenario.

For the convolutional network, we train for 50 epochs in total and use a learning rate schedule that starts with $\eta = 0.02$ and decays by a factor of 5 at the 30th and 40th epochs.
In the normalized MD update \eqref{equ:normalized-md-rescaled}, the base step size $\eta_0$ follows the same schedule and has scale factor $\lambda = 1.0$.

\clearpage
\begin{verbatim}
===================================================================
Layer                                    Output Shape              
===================================================================                 
|-Conv2d: 3x3 kernel                     [512, 32, 26, 26]         
|-BatchNorm2d:                           [512, 32, 26, 26]         
|-ReLU:                                  [512, 32, 26, 26]         
|-MaxPool2d: 2x2 kernel                  [512, 32, 13, 13]         
|-Conv2d: 3x3 kernel                     [512, 32, 11, 11]         
|-BatchNorm2d:                           [512, 32, 11, 11]         
|-ReLU:                                  [512, 32, 11, 11]         
|-MaxPool2d: 2x2 kernel                  [512, 32, 5, 5]           
|-Flatten:                               [512, 800]                
|-Linear:                                [512, 64]                 
|-ReLU:                                  [512, 64]                 
|-Linear:                                [512, 10]                 
===================================================================
\end{verbatim}

\subsection{CIFAR-10 experiments}
For the experiments with the CIFAR-10 dataset, we adopted the example implementation from the \texttt{FFCV} library.\footnote{\url{https://github.com/libffcv/ffcv/tree/main/examples/cifar}}
For consistency, we ran \algname with the same hyper-parameters for all neural networks and values of $p$.
We used a cyclic learning rate schedule with a maximum learning rate of 0.1 and ran for 400 epochs so the training loss is almost equal to 0.\footnote{This differs from the setup from~\cite{azizan2021stochastic}, where they used a fixed small learning rate and much larger number of epochs.}

\subsection{ImageNet experiments}
For the experiments with the ImageNet dataset, we used the example implementation from the \texttt{FFCV} library.\footnote{\url{https://github.com/libffcv/ffcv-imagenet/}}
For consistency, we ran \algname with the same hyper-parameters for all neural networks and values of $p$.
We used a cyclic learning rate schedule with a maximum learning rate of 0.5 and ran for 120 epochs.
Note that, to more accurately measure the effect of \algname on generalization, we turned off any parameters that may affect regularization, e.g. with momentum set to 0, weight decay set to 0, and label smoothing set to 0, etc.

\vspace{-0.75em} 
\section{Additional Experimental Results}
\label{sec:add-experiments}
\subsection{Linear classification}
\label{sec:add-experiment-synthetic}
We present a more complete result for the setting of  Section \ref{sec:linear-classifier} with more values of $p$.
Note that Table~\ref{tab:linear-bias} is a subset of Table~\ref{tab:linear-bias-full-1} and Table~\ref{tab:linear-bias-md} is a subset of Table~\ref{tab:linear-bias-md-full-1}, as shown below.
Within each trial, we use the same dataset generated from a fixed random seed.

We first observe the results of \algname in Tables~\ref{tab:linear-bias-full-1} and~\ref{tab:linear-bias-full-2}.
Except for $p = 1.1$, \algname produces the smallest linear classifier under the corresponding $\ell_p$-norm and thus consistent with the prediction of Theorem~\ref{thm:primal-bias}.
When $p = 1.1$, Corollary~\ref{thm:final-convg-rate} predicts a much slower convergence rate.
So, for the number of iterations we have, \algname with $p = 1.1$ in fact cannot compete against \algname with $p = 1.5$, which has a much faster convergence rate but similar implicit bias.
The second trial shows a rare case where \algname with $p = 1.1$ could not even match \algname with $p = 2$ under the $\ell_{1.1}$-norm.

Next, we fix the value of $\beta$ and look at the results of MD with potential $\psi(\cdot) = \frac{1}{p} \norm{\cdot}_p^2$ in Tables~\ref{tab:linear-bias-md-full-1} and~\ref{tab:linear-bias-md-full-2}.
We are able to observe the same general trend as we did for \algname.
However, because we fixed the value of the exponent $\beta$, the various linear classifiers generated by MD would have similar rates of convergence.
We note that in both trials, MD with potential $\psi(\cdot) = \frac{1}{p} \norm{\cdot}_{1.1}^2$ led to the smallest classifier in $\ell_{1.1}$-norm, which is consistent with Theorem~\ref{thm:primal-bias}.

\begin{table}[!htb]
    \centering
    \setlength{\tabcolsep}{4.5pt}
    \begin{tabular}{l| c|c|c|c|c|c|c|c}
    \hline
    & $\ell_1$ norm & $\ell_{1.1}$ norm & $\ell_{1.5}$ norm & $\ell_{2}$ norm & $\ell_{3}$ norm & $\ell_{6}$ norm & $\ell_{10}$ norm & $\ell_{\infty}$ norm \\
    \hline\hline
    $p=1.1$ & \textbf{7.398} & 5.477 & 2.592 & 1.637 & 1.093 & 0.780 & 0.696 & 0.629 \\
    $p=1.5$ & 7.544 & \textbf{5.348} & \textbf{2.237} & 1.296 & 0.803 & 0.558 & 0.514 & 0.505 \\
    $p=2$ & 8.985 & 6.161 & 2.315 & \textbf{1.224} & 0.684 & 0.429 & 0.382 & 0.366 \\
    $p=3$ & 10.820 & 7.299 & 2.592 & 1.296 & \textbf{0.667} & 0.369 & 0.309 & 0.278 \\
    $p=6$ & 12.714 & 8.523 & 2.957 & 1.441 & 0.711 & \textbf{0.360} & 0.281 & 0.229 \\
    $p=10$ & 13.484 & 9.032 & 3.123 & 1.515 & 0.740 & 0.367 & \textbf{0.280} & \textbf{0.213} \\
    \hline
    \end{tabular}
    \caption{Size of the linear classifiers generated by \algname (after rescaling) in $\ell_1, \ell_{1.1}, \ell_{1.5}, \ell_2, \ell_3, \ell_6$ and $\ell_{10}$ norms.
    For each norm, we highlight the value of $p$ for which \algname generates the smallest classifier under that norm. (Trial 1)}
    \label{tab:linear-bias-full-1}
\end{table}

\begin{table}[!htb]
    \centering
    \setlength{\tabcolsep}{4.5pt}
    \begin{tabular}{l| c|c|c|c|c|c|c|c}
    \hline
    & $\ell_1$ norm & $\ell_{1.1}$ norm & $\ell_{1.5}$ norm & $\ell_{2}$ norm & $\ell_{3}$ norm & $\ell_{6}$ norm & $\ell_{10}$ norm & $\ell_{\infty}$ norm \\
    \hline\hline
    $p=1.1$ & 10.278 & 7.731 & 3.776 & 2.408 & 1.610 & 1.162 & 1.058 & 0.973 \\
    $p=1.5$ & \textbf{8.835} & \textbf{6.222} & \textbf{2.549} & 1.450 & 0.873 & 0.577 & 0.512 & 0.463 \\
    $p=2$ & 10.161 & 6.962 & 2.609 & \textbf{1.375} & 0.760 & 0.464 & 0.406 & 0.387 \\
    $p=3$ & 11.681 & 7.926 & 2.863 & 1.449 & \textbf{0.754} & 0.419 & 0.348 & 0.316 \\
    $p=6$ & 13.454 & 9.083 & 3.217 & 1.592 & 0.797 & \textbf{0.410} & 0.321 & 0.261 \\
    $p=10$ & 14.290 & 9.630 & 3.392 & 1.669 & 0.828 & 0.417 & \textbf{0.321} & \textbf{0.244} \\
    \hline
    \end{tabular}
    \caption{Size of the linear classifiers generated by \algname (after rescaling) in $\ell_1, \ell_{1.1}, \ell_{1.5}, \ell_2, \ell_3, \ell_6$ and $\ell_{10}$ norms.
    For each norm, we highlight the value of $p$ for which \algname generates the smallest classifier under that norm. (Trial 2)}
    \label{tab:linear-bias-full-2}
\end{table}

\clearpage

\begin{table}[!htb]
    \centering
    \setlength{\tabcolsep}{4.5pt}
    \begin{tabular}{l| c|c|c|c|c|c|c|c}
    \hline
    & $\ell_1$ norm & $\ell_{1.1}$ norm & $\ell_{1.5}$ norm & $\ell_{2}$ norm & $\ell_{3}$ norm & $\ell_{6}$ norm & $\ell_{10}$ norm & $\ell_{\infty}$ norm \\
    \hline\hline
    $p=1.1$ & \textbf{6.526} & \textbf{5.136} & 2.780 & 1.864 & 1.276 & 0.900 & 0.795 & 0.700 \\
    $p=1.5$ & 7.338 & 5.231 & \textbf{2.215} & 1.292 & 0.803 & 0.552 & 0.498 & 0.473 \\
    $p=2$ & 8.985 & 6.161 & 2.315 & \textbf{1.224} & 0.684 & 0.429 & 0.382 & 0.366 \\
    $p=3$ & 10.871 & 7.305 & 2.567 & 1.275 & \textbf{0.652} & 0.360 & 0.301 & 0.276 \\
    $p=6$ & 12.836 & 8.553 & 2.919 & 1.406 & 0.687 & \textbf{0.346} & 0.269 & 0.220 \\
    $p=10$ & 13.738 & 9.132 & 3.091 & 1.477 & 0.712 & 0.349 & \textbf{0.266} & \textbf{0.201} \\
    \hline
    \end{tabular}
    \caption{Size of the linear classifiers generated by MD with potential $\psi(\cdot) = \frac{1}{p} \norm{\cdot}_p^2$ (after rescaling) in $\ell_1, \ell_{1.1}, \ell_{1.5}, \ell_2, \ell_3, \ell_6$ and $\ell_{10}$ norms.
    For each norm, we highlight the value of $p$ for which MD with potential $\psi(\cdot) = \frac{1}{p} \norm{\cdot}_p^2$ generates the smallest classifier under that norm. (Trial 1)}
    \label{tab:linear-bias-md-full-1}
\end{table}

\begin{table}[!htb]
    \centering
    \setlength{\tabcolsep}{4.5pt}
    \begin{tabular}{l| c|c|c|c|c|c|c|c}
    \hline
    & $\ell_1$ norm & $\ell_{1.1}$ norm & $\ell_{1.5}$ norm & $\ell_{2}$ norm & $\ell_{3}$ norm & $\ell_{6}$ norm & $\ell_{10}$ norm & $\ell_{\infty}$ norm \\
    \hline\hline
    $p=1.1$ & \textbf{7.277} & \textbf{5.646} & 2.996 & 2.032 & 1.453 & 1.118 & 1.039 & 0.980 \\
    $p=1.5$ & 8.671 & 6.133 & \textbf{2.529} & 1.438 & 0.867 & 0.582 & 0.527 & 0.509 \\
    $p=2$ & 10.161 & 6.962 & 2.609 & \textbf{1.375} & 0.760 & 0.464 & 0.406 & 0.387 \\
    $p=3$ & 11.783 & 7.948 & 2.828 & 1.420 & \textbf{0.735} & 0.409 & 0.340 & 0.306 \\
    $p=6$ & 13.877 & 9.272 & 3.194 & 1.553 & 0.767 & \textbf{0.392} & 0.309 & 0.250 \\
    $p=10$ & 14.685 & 9.800 & 3.360 & 1.625 & 0.795 & 0.397 & \textbf{0.306} & \textbf{0.237} \\
    \hline
    \end{tabular}
    \caption{Size of the linear classifiers generated by MD with potential $\psi(\cdot) = \frac{1}{p} \norm{\cdot}_p^2$ (after rescaling) in $\ell_1, \ell_{1.1}, \ell_{1.5}, \ell_2, \ell_3, \ell_6$ and $\ell_{10}$ norms.
    For each norm, we highlight the value of $p$ for which MD with potential $\psi(\cdot) = \frac{1}{p} \norm{\cdot}_p^2$ generates the smallest classifier under that norm. (Trial 2)}
    \label{tab:linear-bias-md-full-2}
\end{table}

\subsection{Experiments with normalized MD}
\label{sec:add-experiment-normalized}
We present additional experiments on normalized MD that expand upon what we presented in Section~\ref{sec:normalized-exp}.
We compare normalized \algname against the standard \algname for $p = 1.5, 2$ and $2.5$.

For $p = 1.5$ and $2$, normalized \algname enjoys much faster convergence and we can see that its training loss in both synthetic dataset and MNIST is significantly lower than that of standard \algname.
The picture for $p = 2.5$ is less clear, where normalized \algname enjoys a similarly sizable advantage in rate of convergence for synthetic dataset while struggling somewhat on the MNIST dataset before learning rate decay kicks in.
But in all cases, the final loss achieved by normalized \algname is much lower and we see that the lower training loss translates to better test performance on the MINST dataset, where normalized \algname is better than standard \algname by about 0.2--0.7 percent.
These observations are consistent with our analysis on more general normalized mirror descent, as reflected by the statement of Theorem~\ref{thm:normalized-md-rate}.

\begin{figure}[!ht]
    \centering
    \begin{subfigure}[b]{\textwidth}
        \centering
        \includegraphics[width=0.3\textwidth]{figure/synthetic/loss_n15.pdf}
        ~
        \includegraphics[width=0.3\textwidth]{figure/synthetic/angle_n15.pdf}
        ~
        \includegraphics[width=0.3\textwidth]{figure/synthetic/norm_n15.pdf}
        \caption{$p = 1.5$}
    \end{subfigure}
    \begin{subfigure}[b]{\textwidth}
        \centering
        \includegraphics[width=0.3\textwidth]{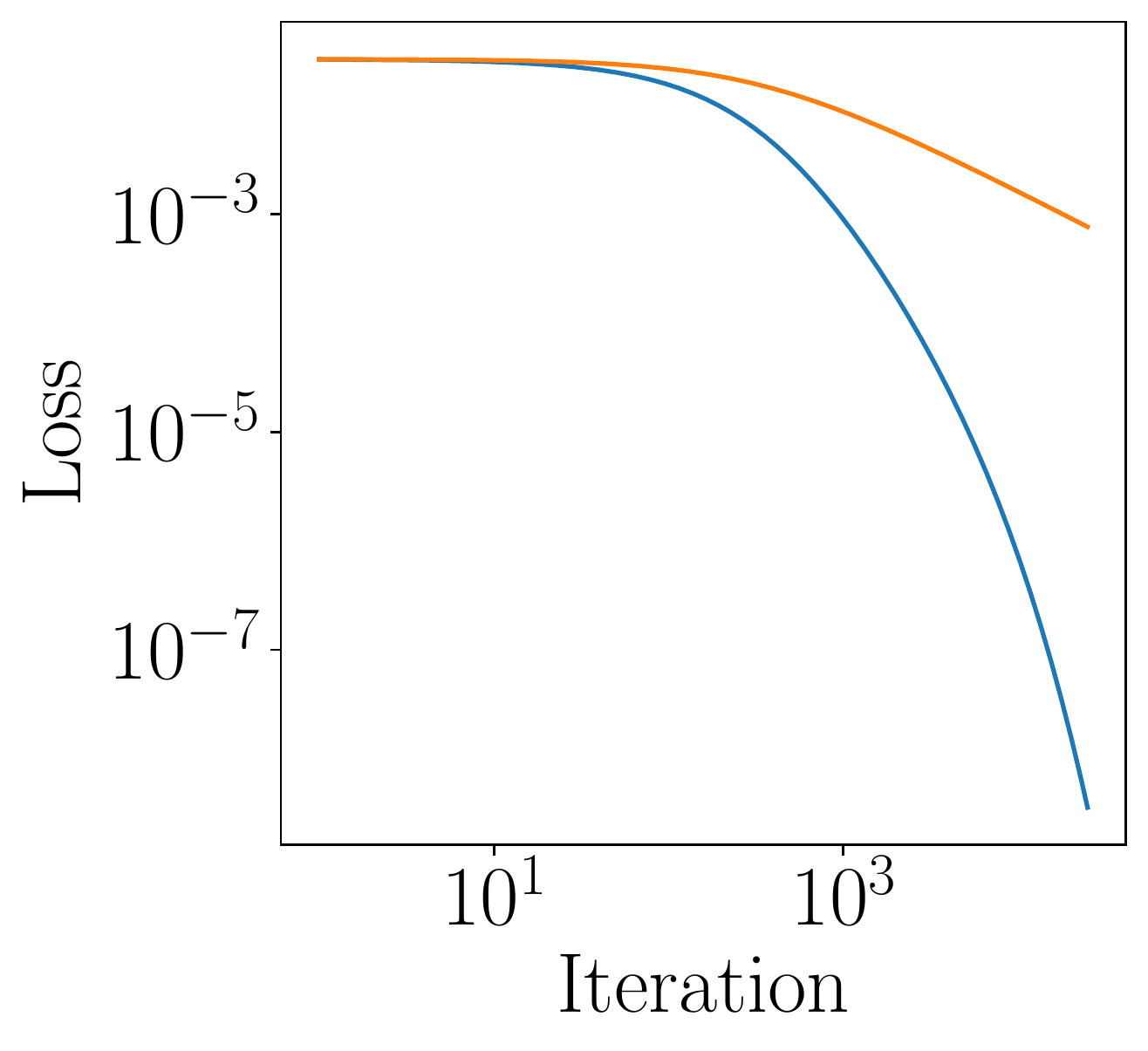}
        ~
        \includegraphics[width=0.3\textwidth]{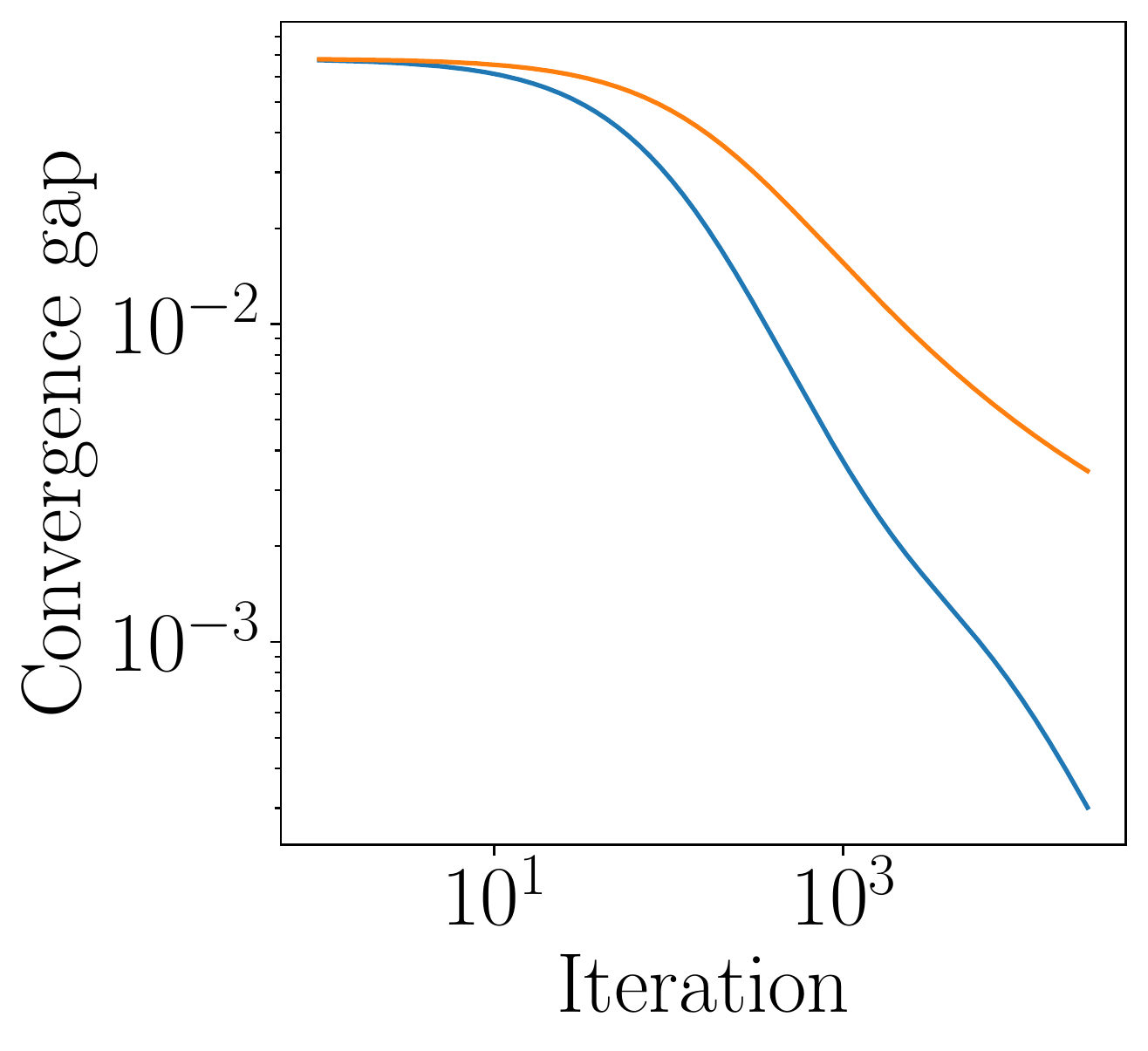}
        ~
        \includegraphics[width=0.3\textwidth]{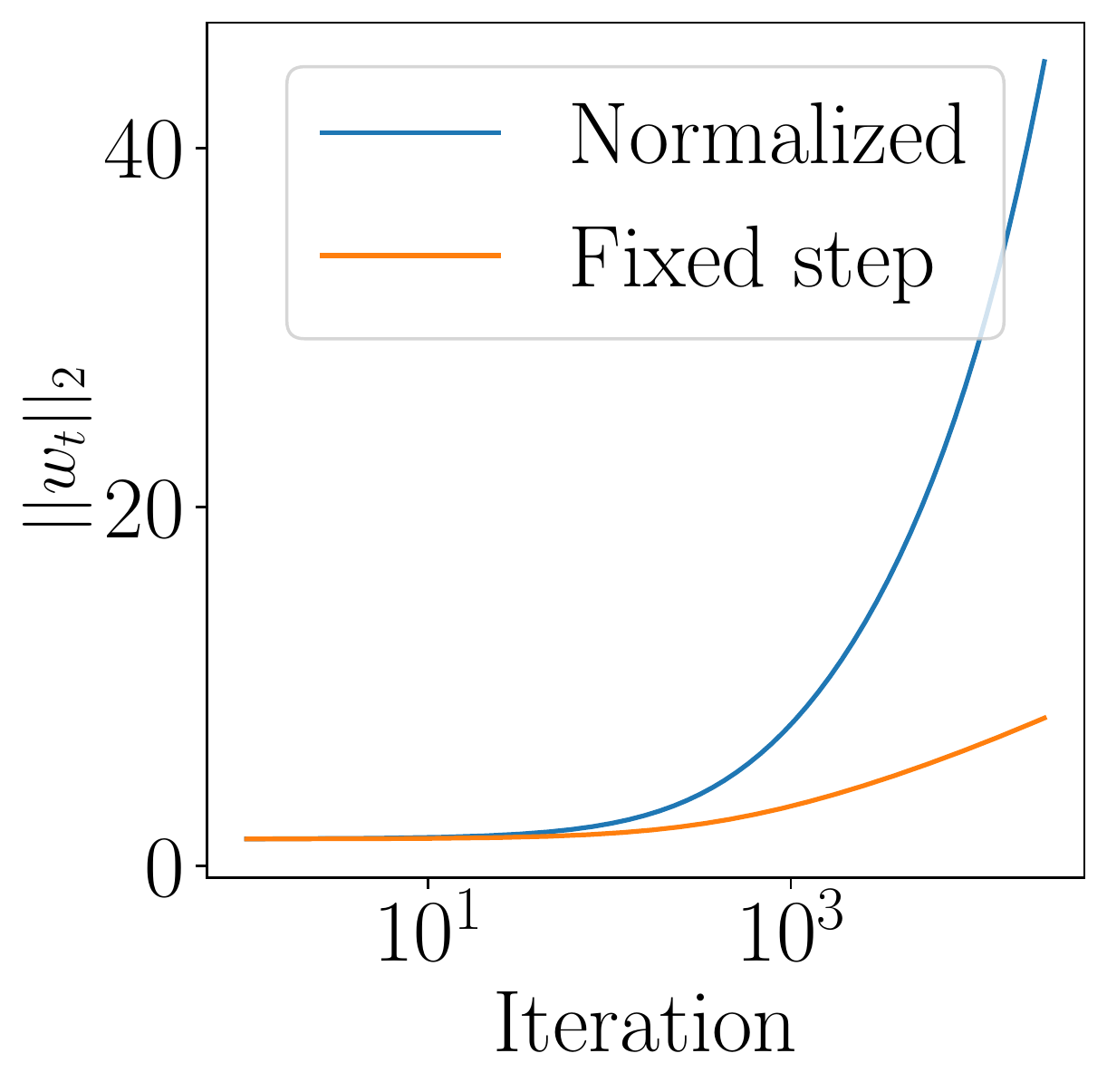}
        \caption{$p = 2$}
    \end{subfigure}
        \begin{subfigure}[b]{\textwidth}
        \centering
        \includegraphics[width=0.3\textwidth]{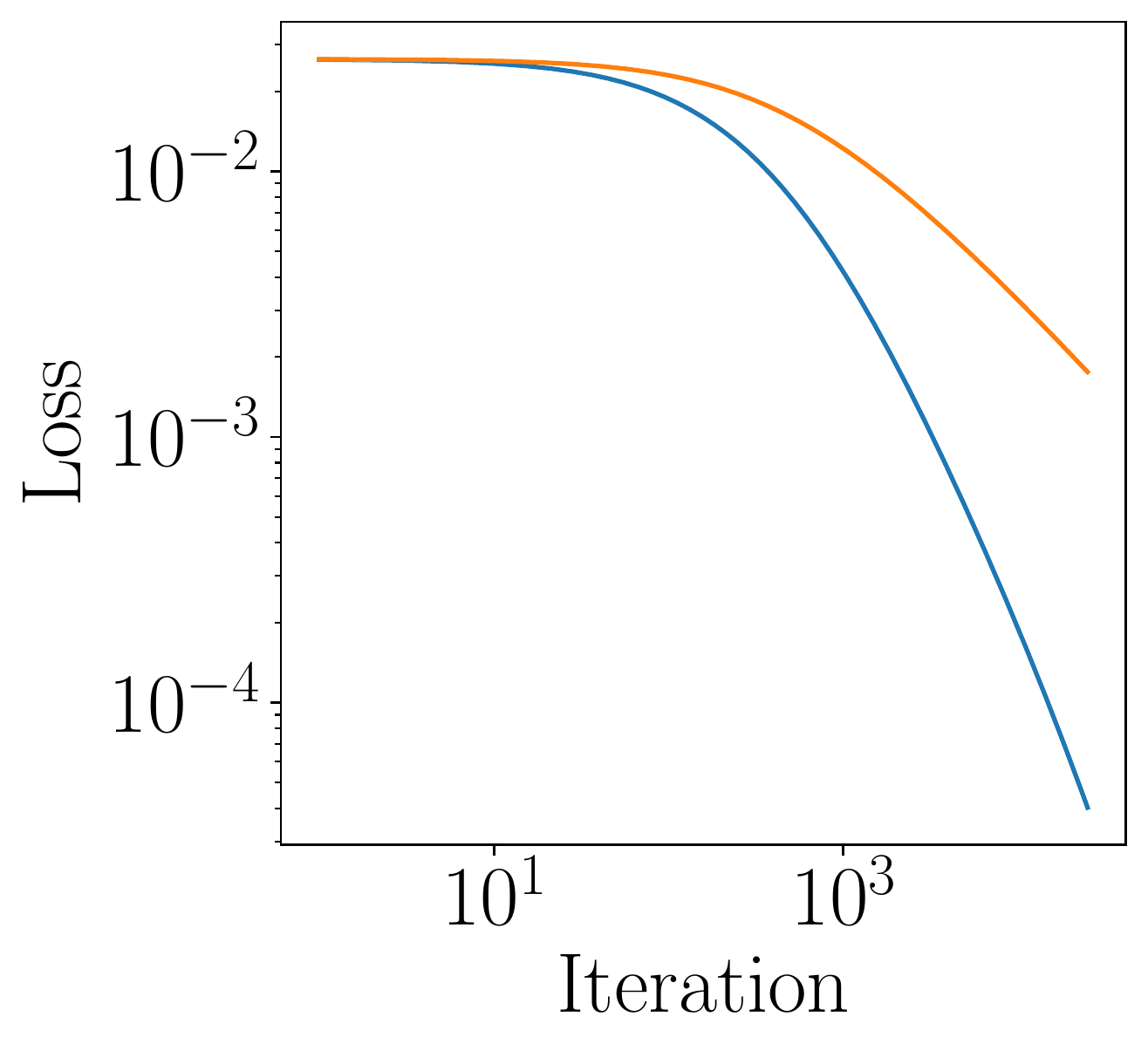}
        ~
        \includegraphics[width=0.3\textwidth]{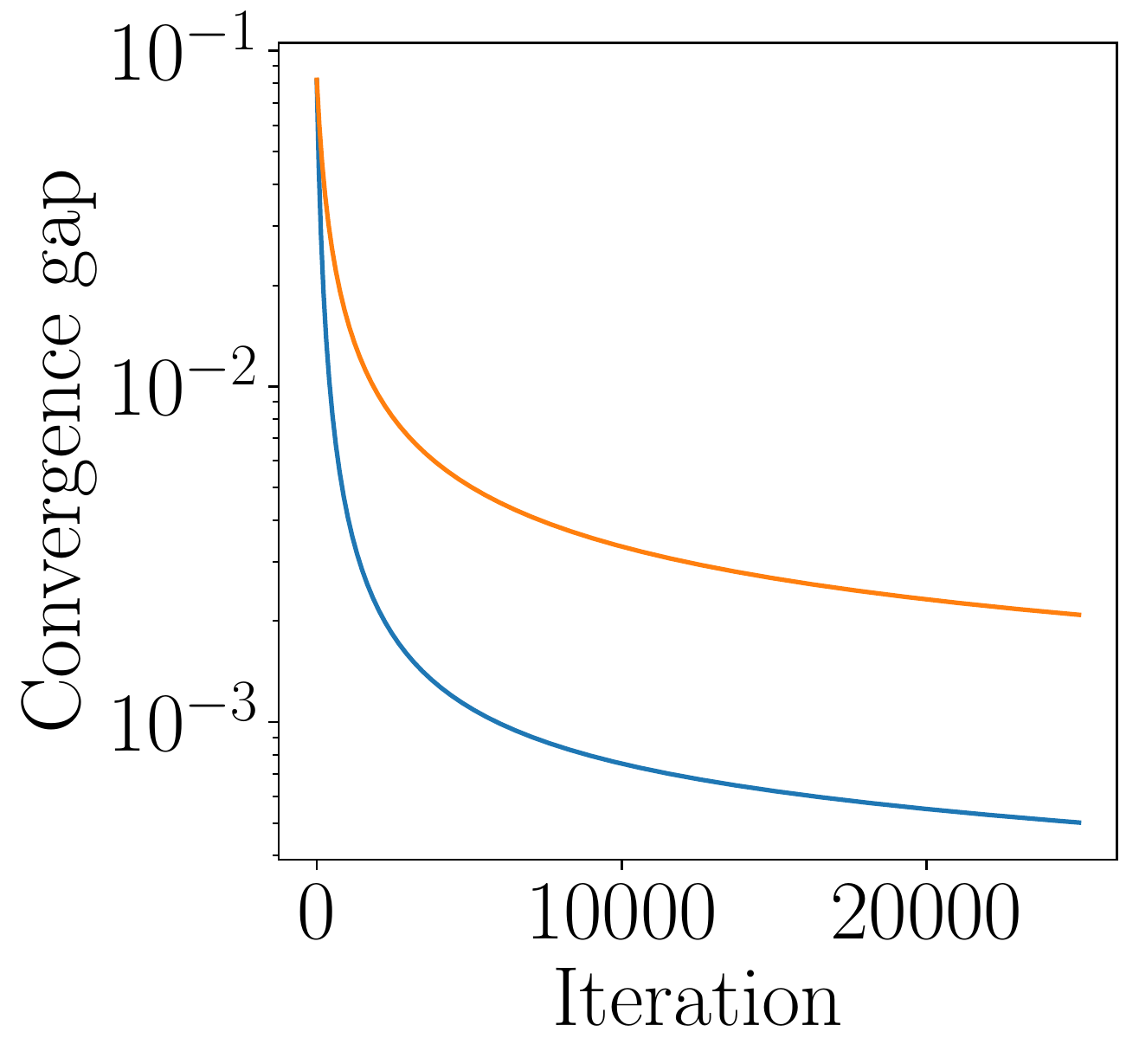}
        ~
        \includegraphics[width=0.3\textwidth]{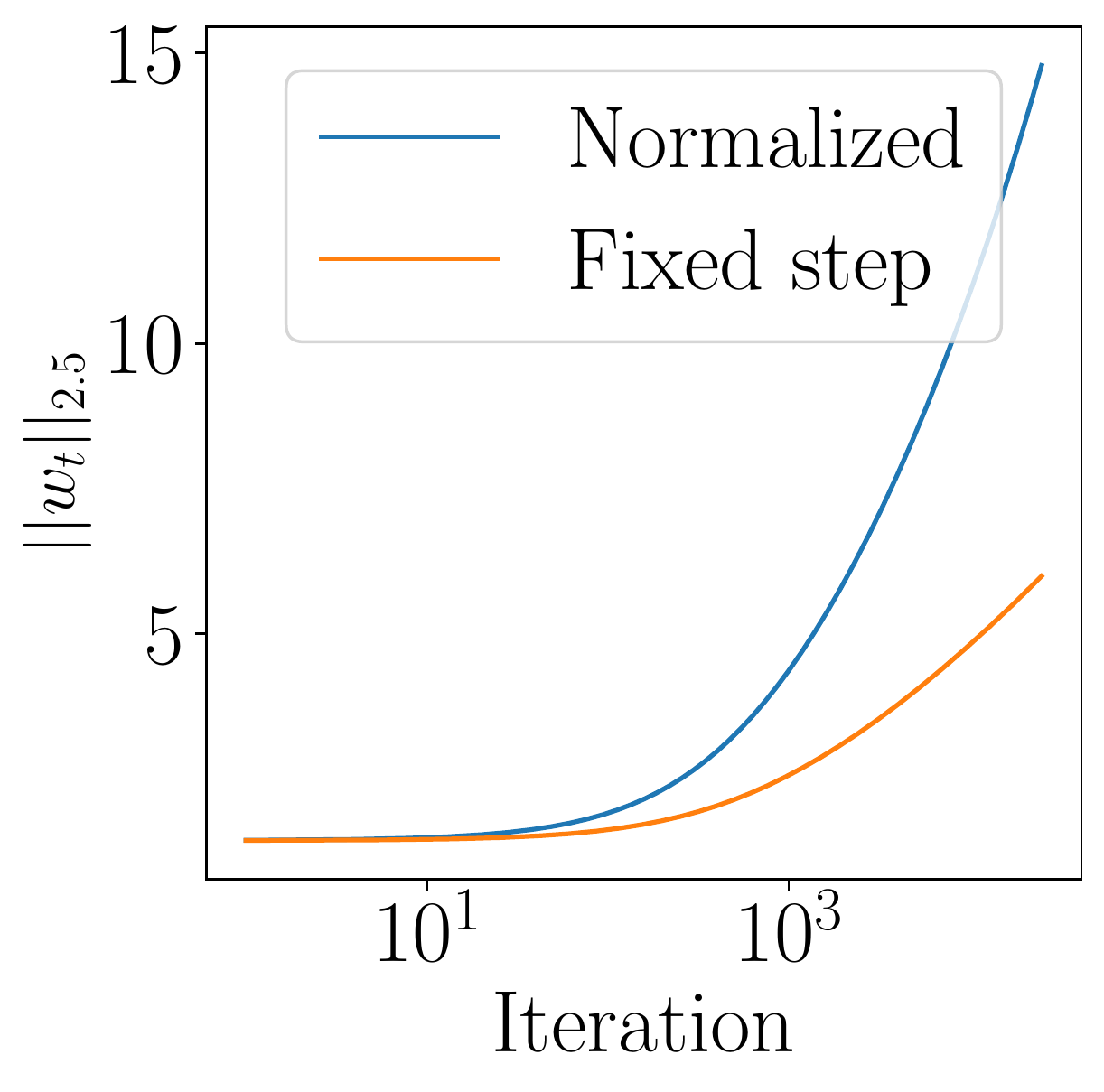}
        \caption{$p = 2.5$}
    \end{subfigure}
    \caption{Examples of \algname and normalized \algname on randomly generated data with exponential loss and $p = 1.5, 2, 2.5$. 
    \textbf{(1)} The left plot is the empirical loss.
    \textbf{(2)} The middle plot shows the rate which the quantity $\brg{\reg{p}}{w_t / \norm{w_t}_t}$ converges to 0.
    \textbf{(3)} The right plot shows how fast the $p$-norm of $w_t$ growths.
    }
\end{figure}

\clearpage

\begin{figure}[!ht]
    \centering
    \begin{subfigure}[b]{\textwidth}
        \centering
        \includegraphics[width=0.45\textwidth]{figure/mnist/mnist_loss_15.pdf}
        ~
        \includegraphics[width=0.45\textwidth]{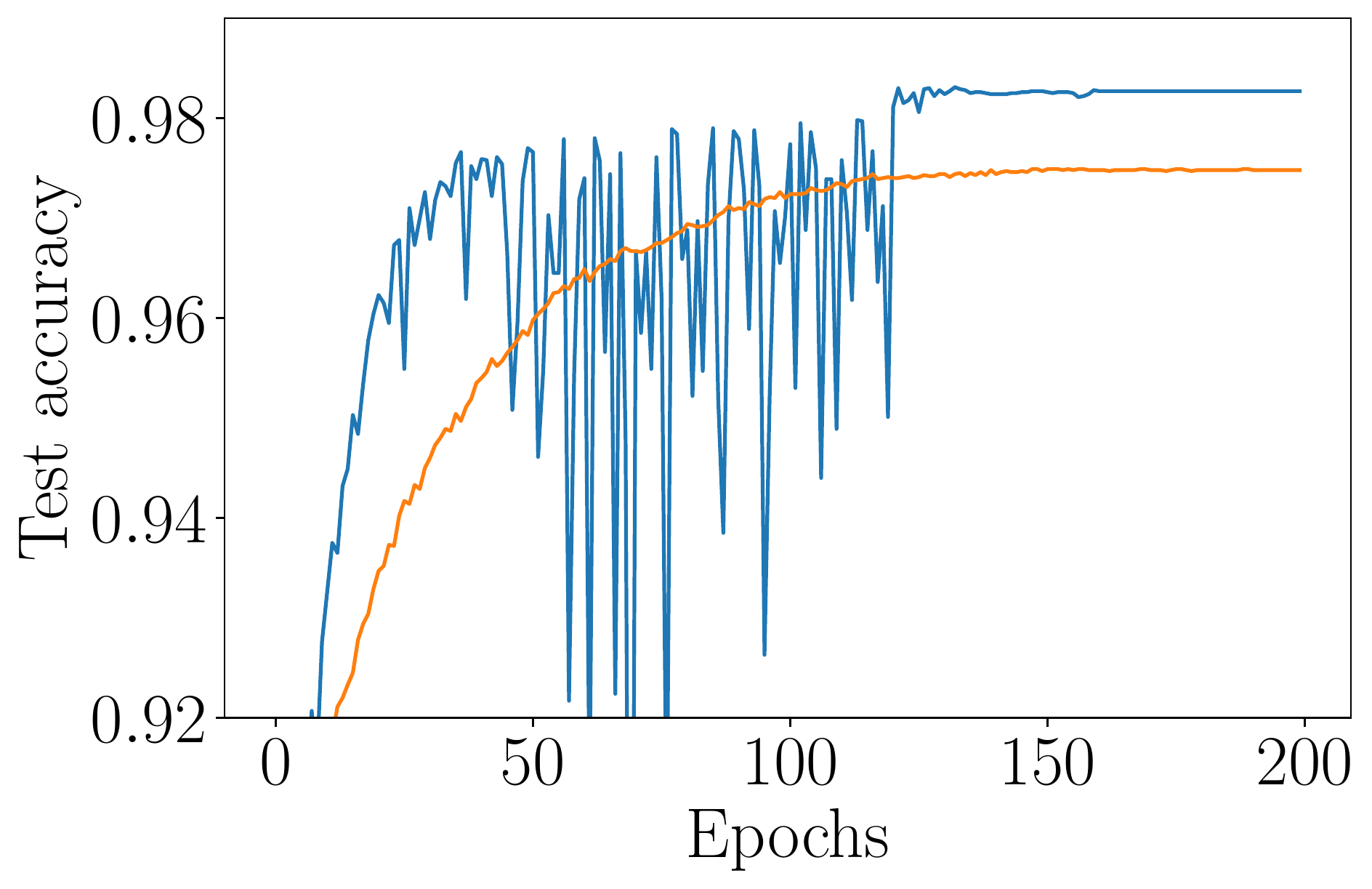}
        \caption{$p = 1.5$}
    \end{subfigure}
    \begin{subfigure}[b]{\textwidth}
        \centering
        \includegraphics[width=0.45\textwidth]{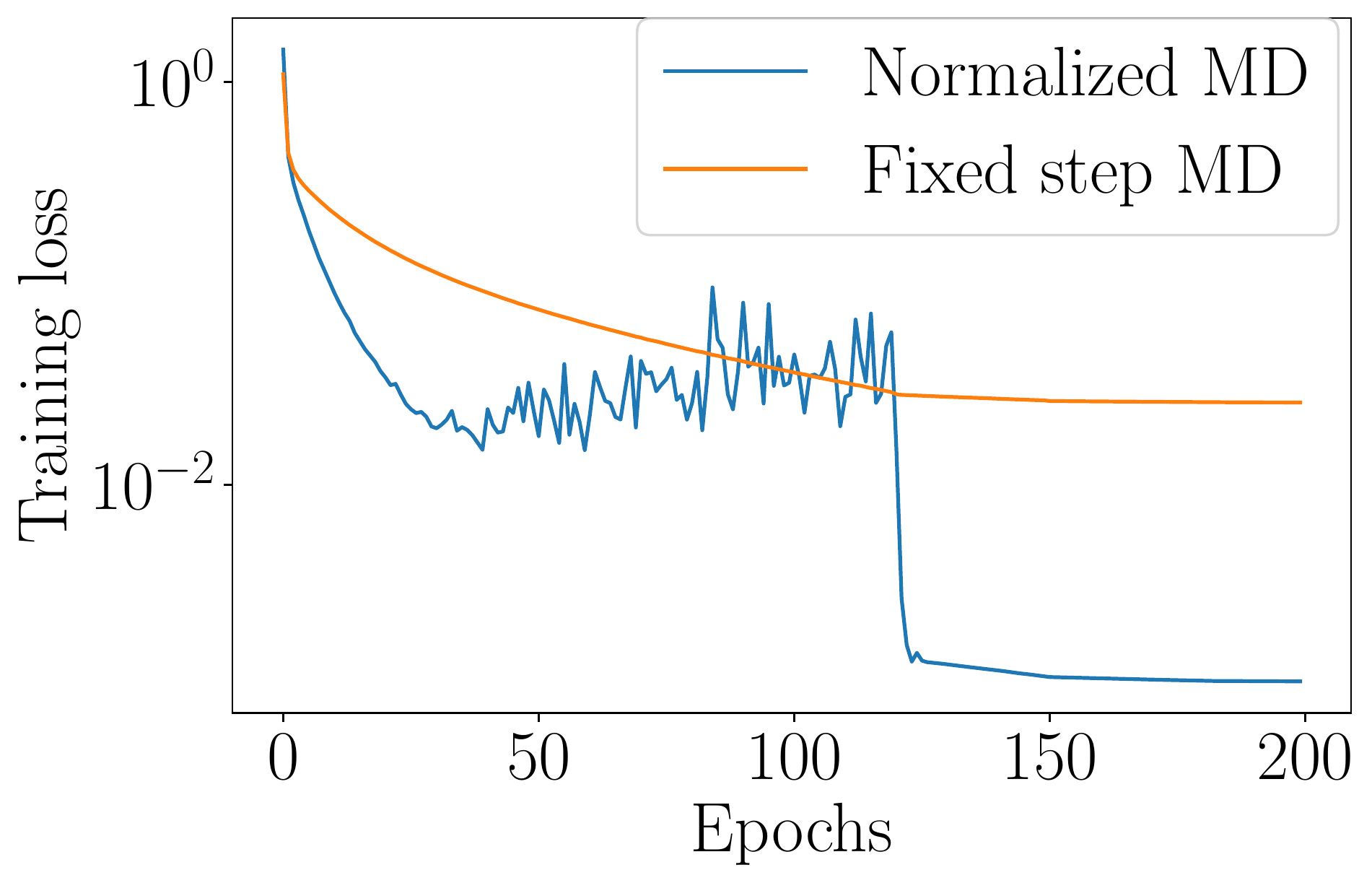}
        ~
        \includegraphics[width=0.45\textwidth]{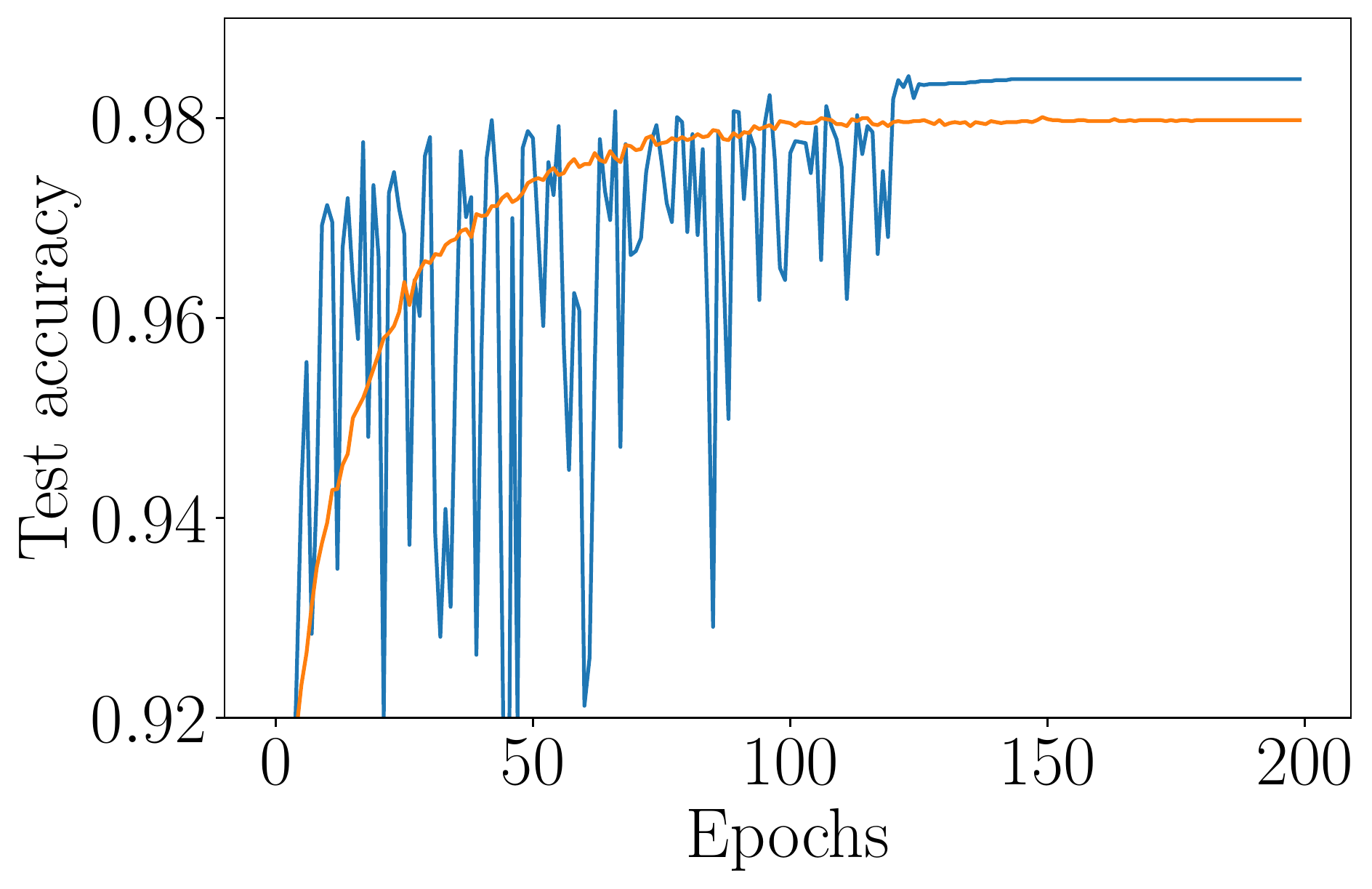}
        \caption{$p = 2$}
    \end{subfigure}
    \begin{subfigure}[b]{\textwidth}
        \centering
        \includegraphics[width=0.45\textwidth]{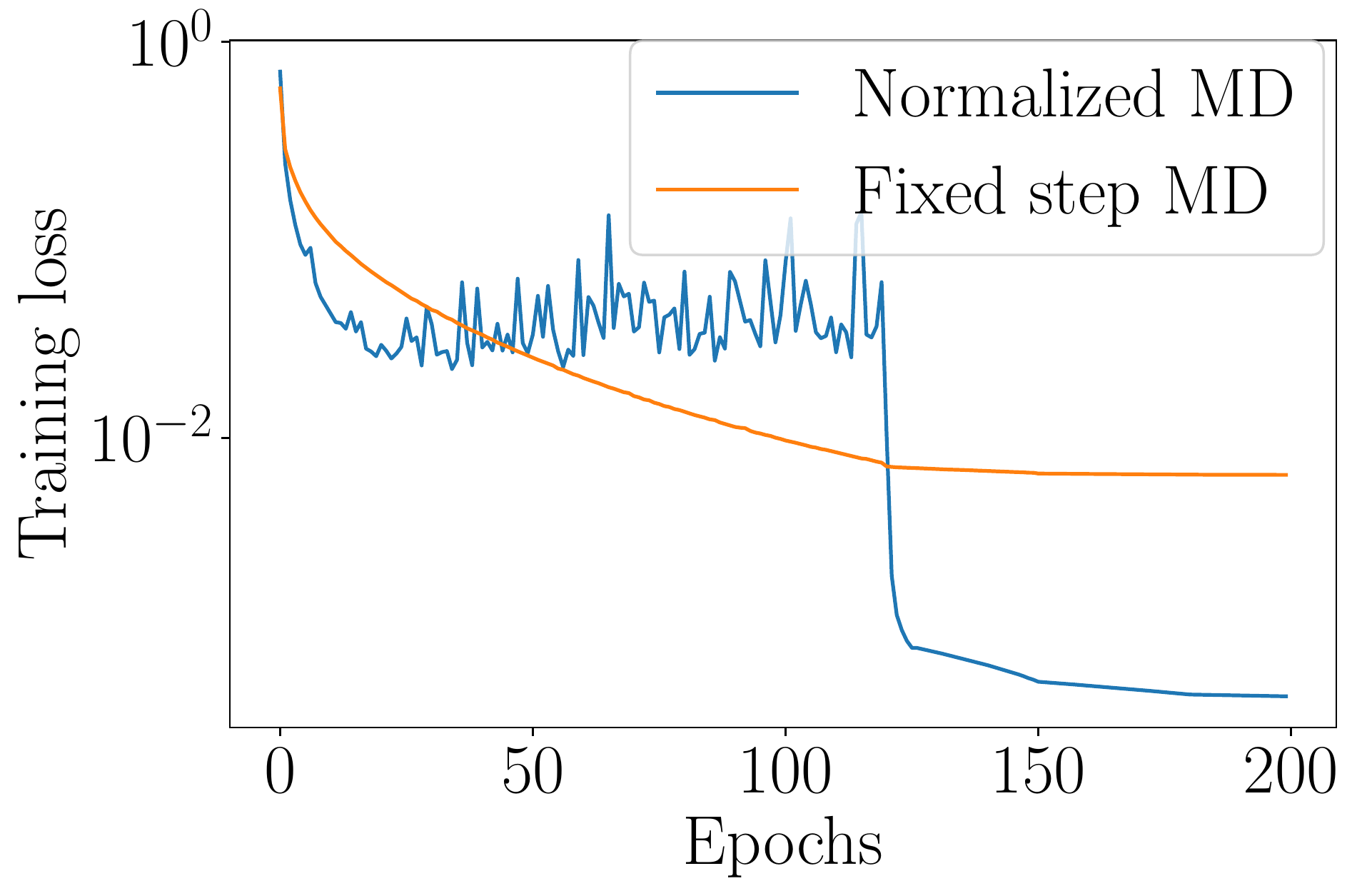}
        ~
        \includegraphics[width=0.45\textwidth]{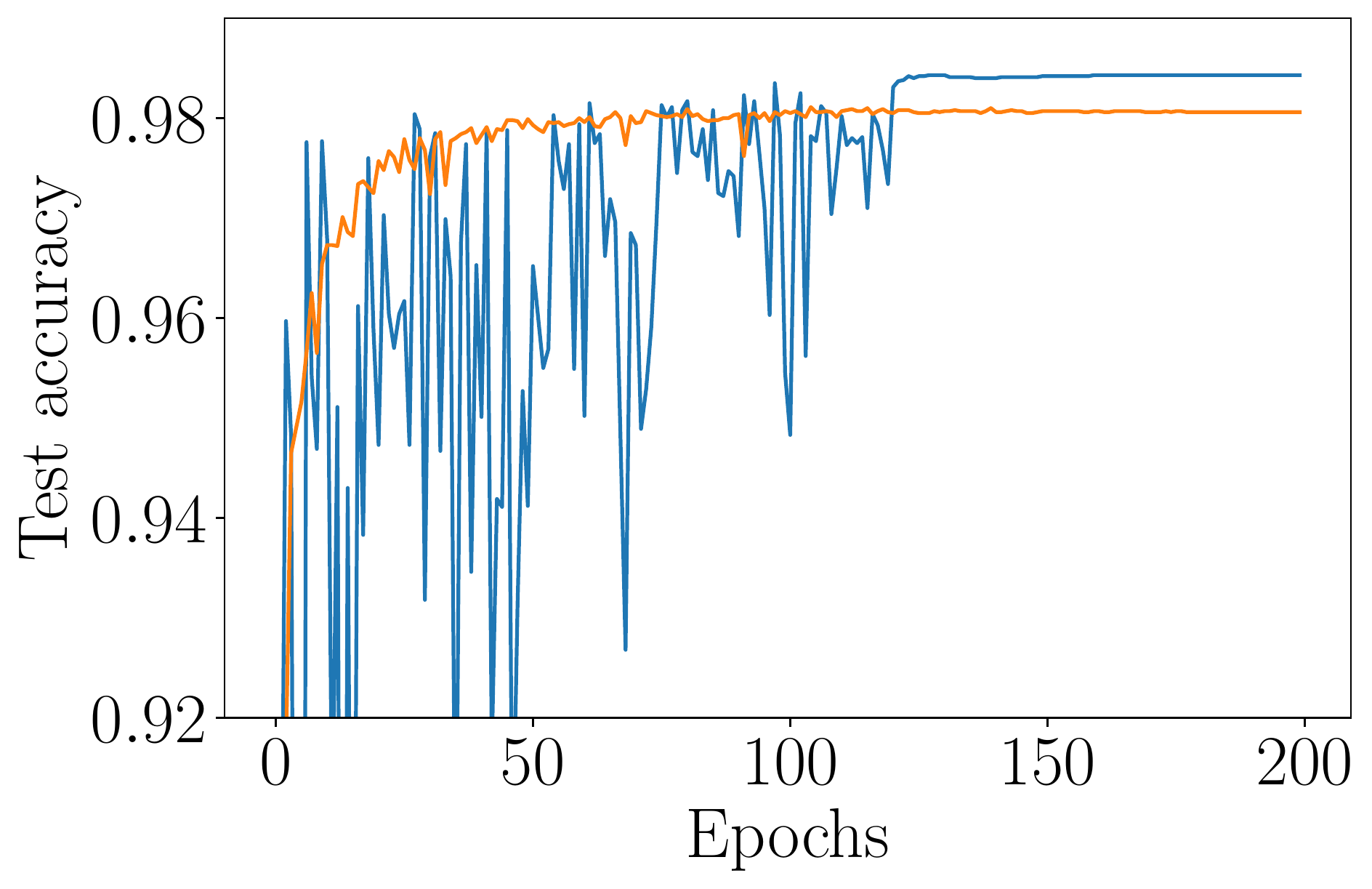}
        \caption{$p = 2.5$}
    \end{subfigure}
    \caption{Example of \algname and normalized \algname on the MNIST dataset and $p = 1.5, 2, 2.5$. \\
    \textbf{(1)} The left plot is the empirical loss at training time.
    \textbf{(2)} The right plot is the test accuracy.
    }
\end{figure}

\begin{table}[!ht]
    \centering
    \begin{tabular}{l| c | c }         
        \hline
        & Mirror descent (\algname) & Normalized \algname \\
        \hline\hline
        $p=1.5$ & 97.65 & 98.37 \\
        $p=2$ (SGD) & 97.95 &  98.34 \\
        $p=2.5$ & 98.29 & 98.48 \\
         \hline
    \end{tabular}
    \caption{MNIST accuracy (\%) of \algname versus normalized \algname for a fully connected network. Note that the normalized version has better generalization for all values of $p$.}
\end{table}

\begin{figure}[!ht]
    \centering
    \begin{subfigure}[b]{\textwidth}
        \centering
        \includegraphics[width=0.45\textwidth]{figure/mnist/mnist_loss_conv_15.pdf}
        ~
        \includegraphics[width=0.45\textwidth]{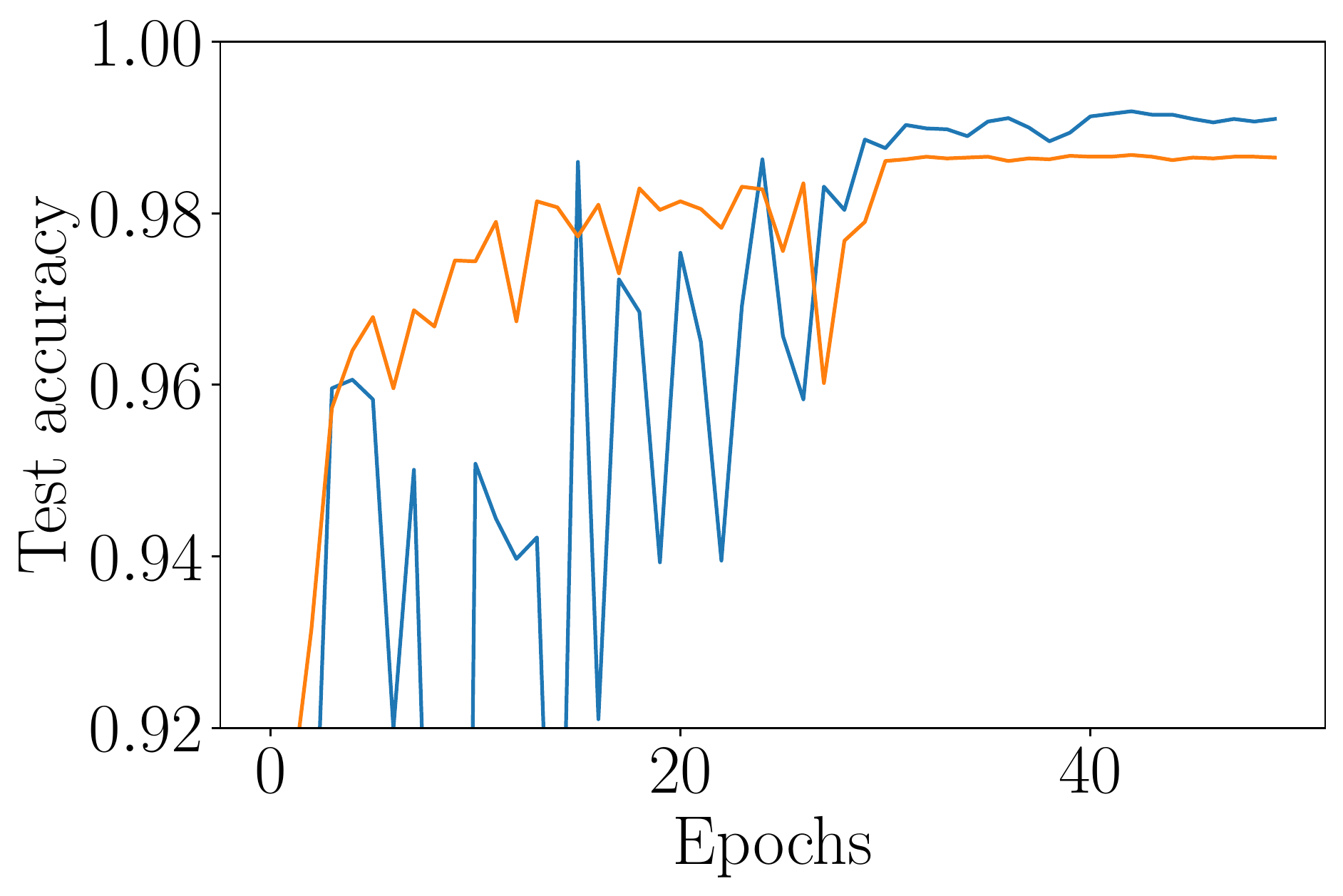}
        \caption{$p = 1.5$}
    \end{subfigure}
    \begin{subfigure}[b]{\textwidth}
        \centering
        \includegraphics[width=0.45\textwidth]{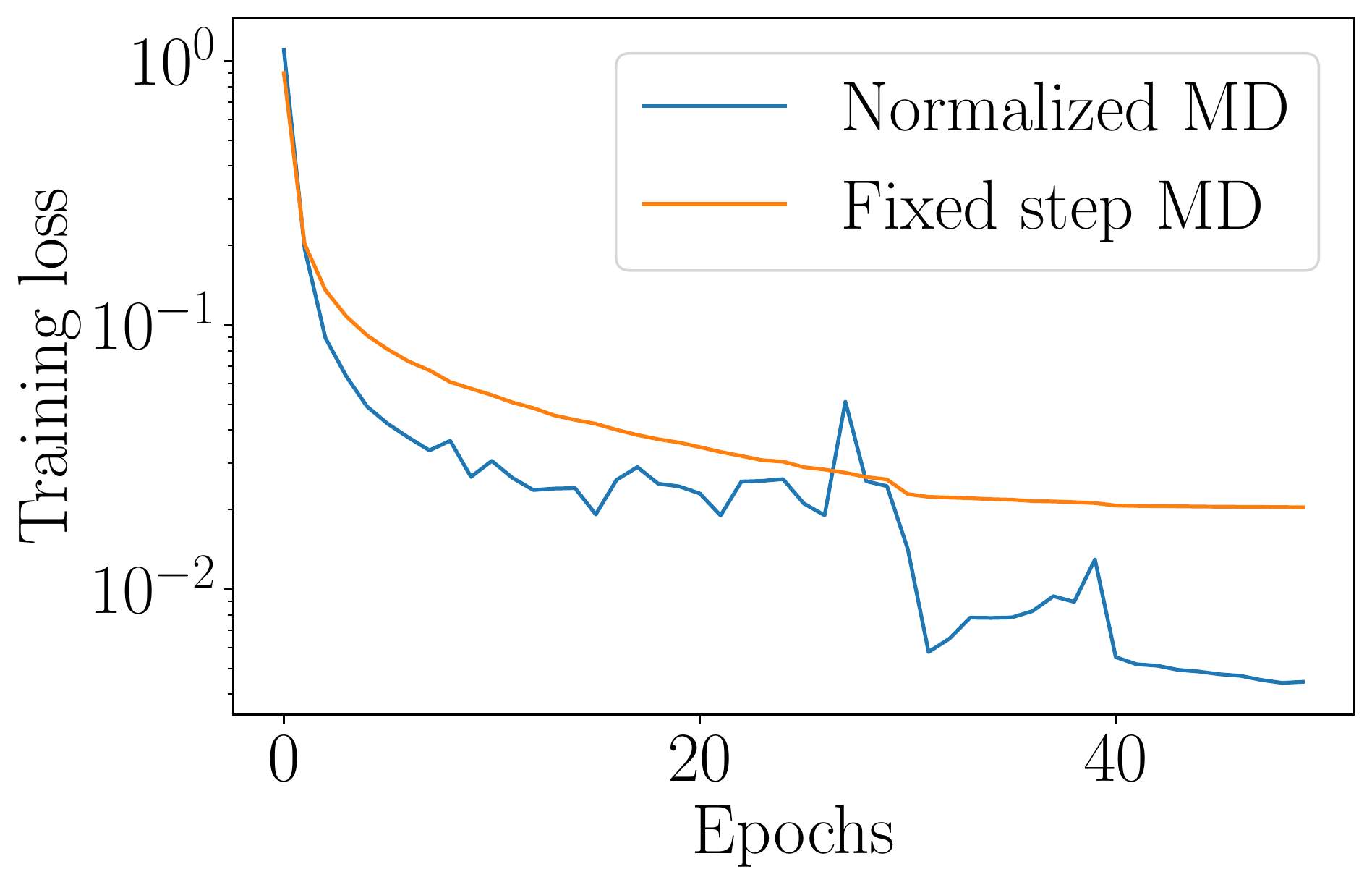}
        ~
        \includegraphics[width=0.45\textwidth]{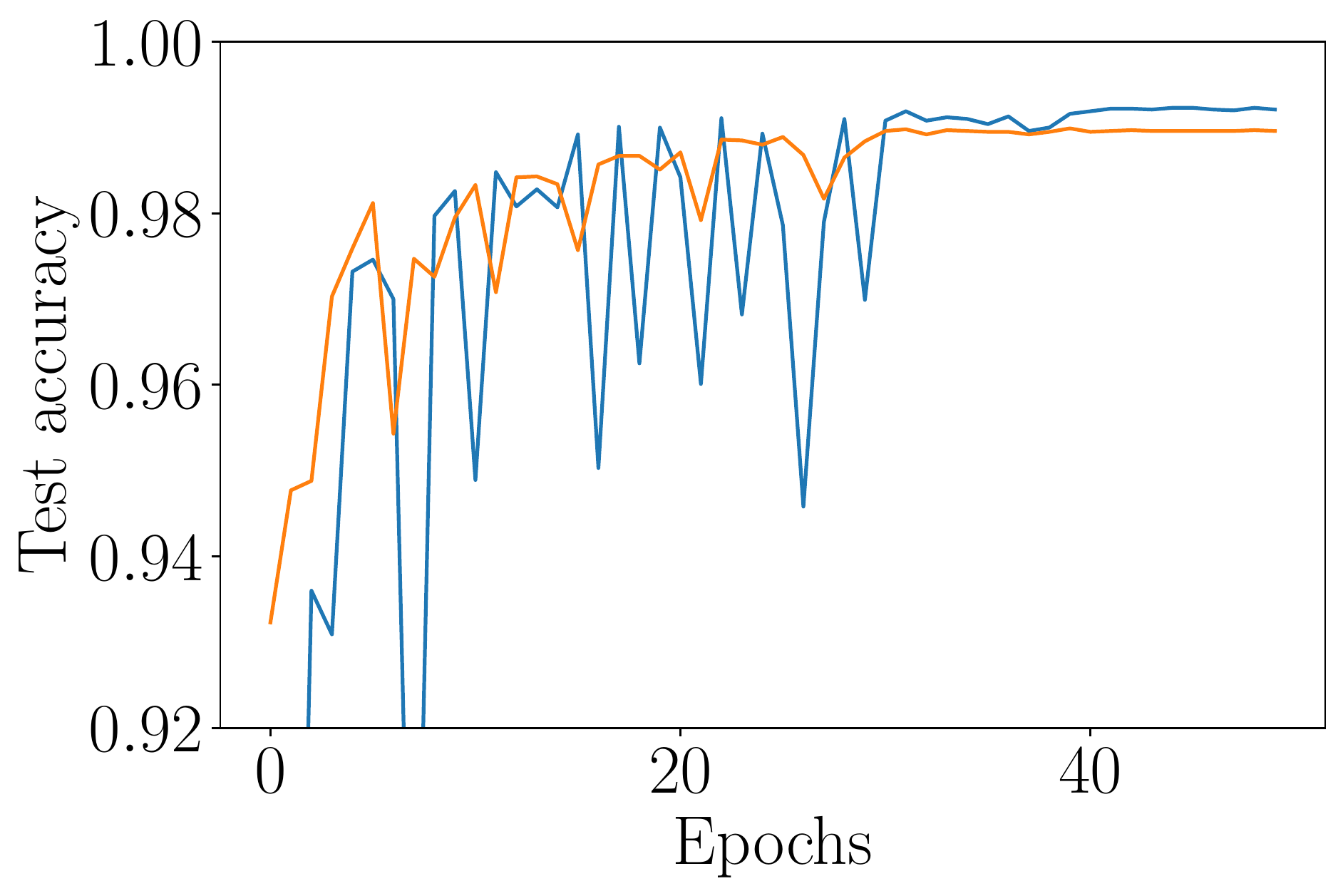}
        \caption{$p = 2$}
    \end{subfigure}
    \begin{subfigure}[b]{\textwidth}
        \centering
        \includegraphics[width=0.45\textwidth]{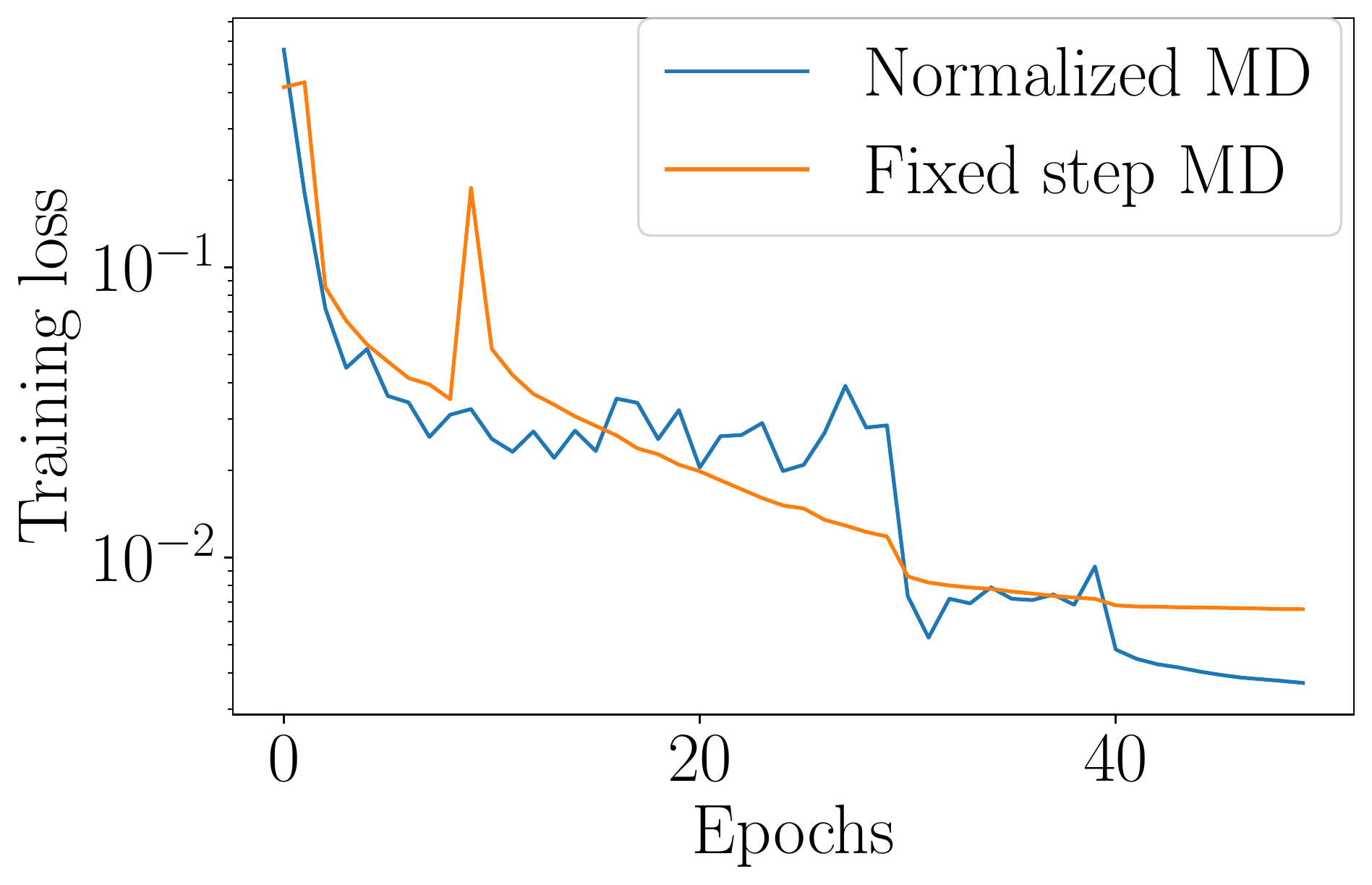}
        ~
        \includegraphics[width=0.45\textwidth]{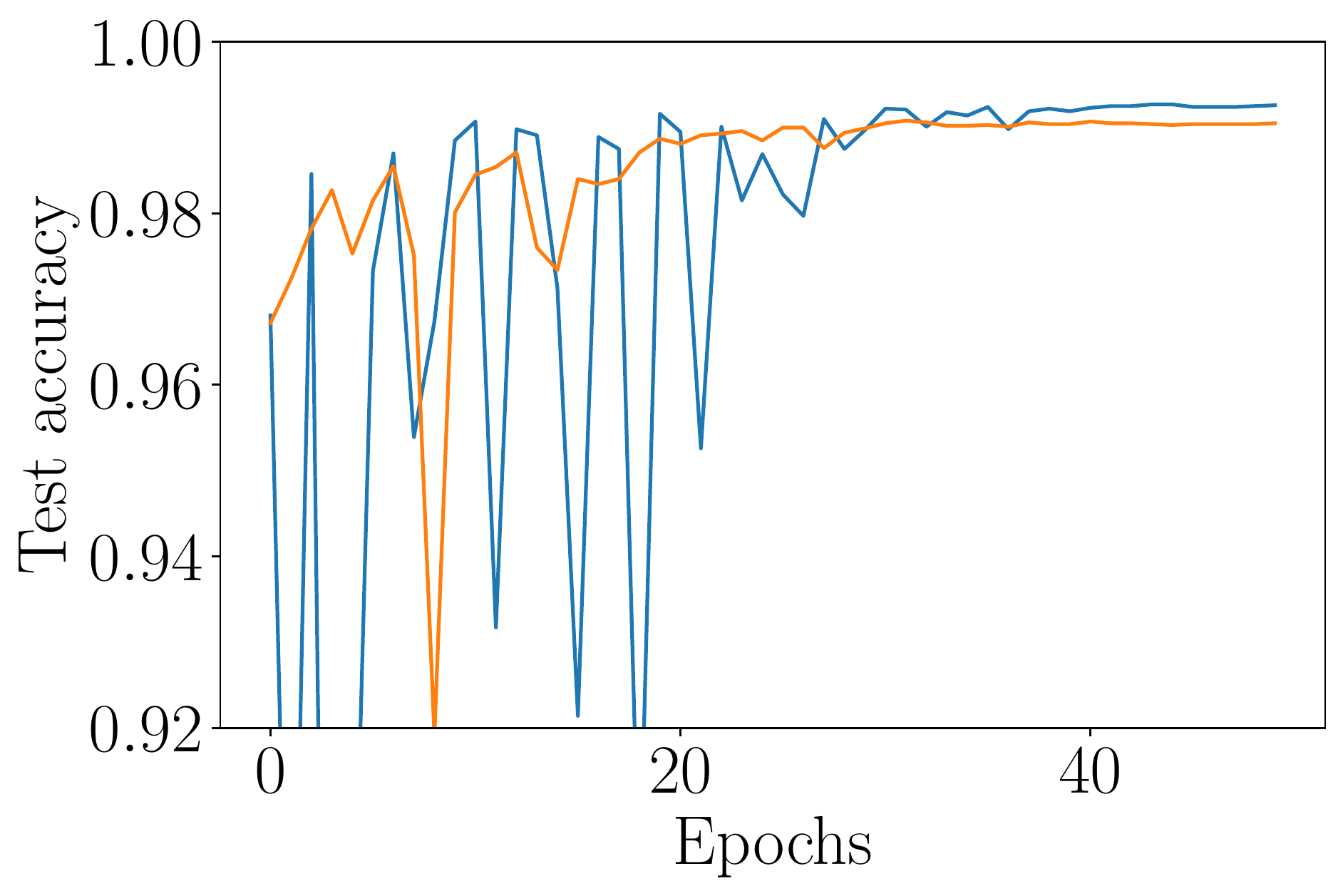}
        \caption{$p = 2.5$}
    \end{subfigure}
    \caption{Example of \algname and normalized \algname on the MNIST dataset and $p = 1.5, 2, 2.5$. \\
    \textbf{(1)} The left plot is the empirical loss at training time.
    \textbf{(2)} The right plot is the test accuracy.
    }
\end{figure}

\begin{table}[!ht]
    \centering
    \begin{tabular}{l| c | c }
        \hline
        & Mirror descent (\algname) & Normalized \algname \\
        \hline\hline
        $p=1.5$ & 98.68 & 99.19 \\
        $p=2$ (SGD) & 98.99 &  99.23 \\
        $p=2.5$ & 99.08 & 99.27 \\
         \hline
    \end{tabular}
    \caption{MNIST accuracy (\%) of \algname versus normalized \algname for a convolutional network. Note that the normalized version has better generalization for all values of $p$.}
\end{table}

\clearpage

\subsection{CIFAR-10 experiments: implicit bias}
We present more complete illustrations of the implicit bias trends of trained models in CIFAR-10.
Compared to Figure~\ref{fig:cifar10-hist}, the plots below include data from additional values for additional values of $p$ and more deep neural network architectures.

We see that the trends we observed in Section~\ref{sec:cifar} continue to hold under architectures other than \textsc{ResNet}.
In particular, for smaller $p$'s, the weight distributions of models trained with \algname have higher peaks around zero, and higher $p$'s result in smaller maximum weights. 

\label{sec:add-experiment-cifar-bias}
\begin{figure}[!h]
    \centering
    \begin{subfigure}[b]{0.48\textwidth}
        \centering
        \includegraphics[width=\textwidth]{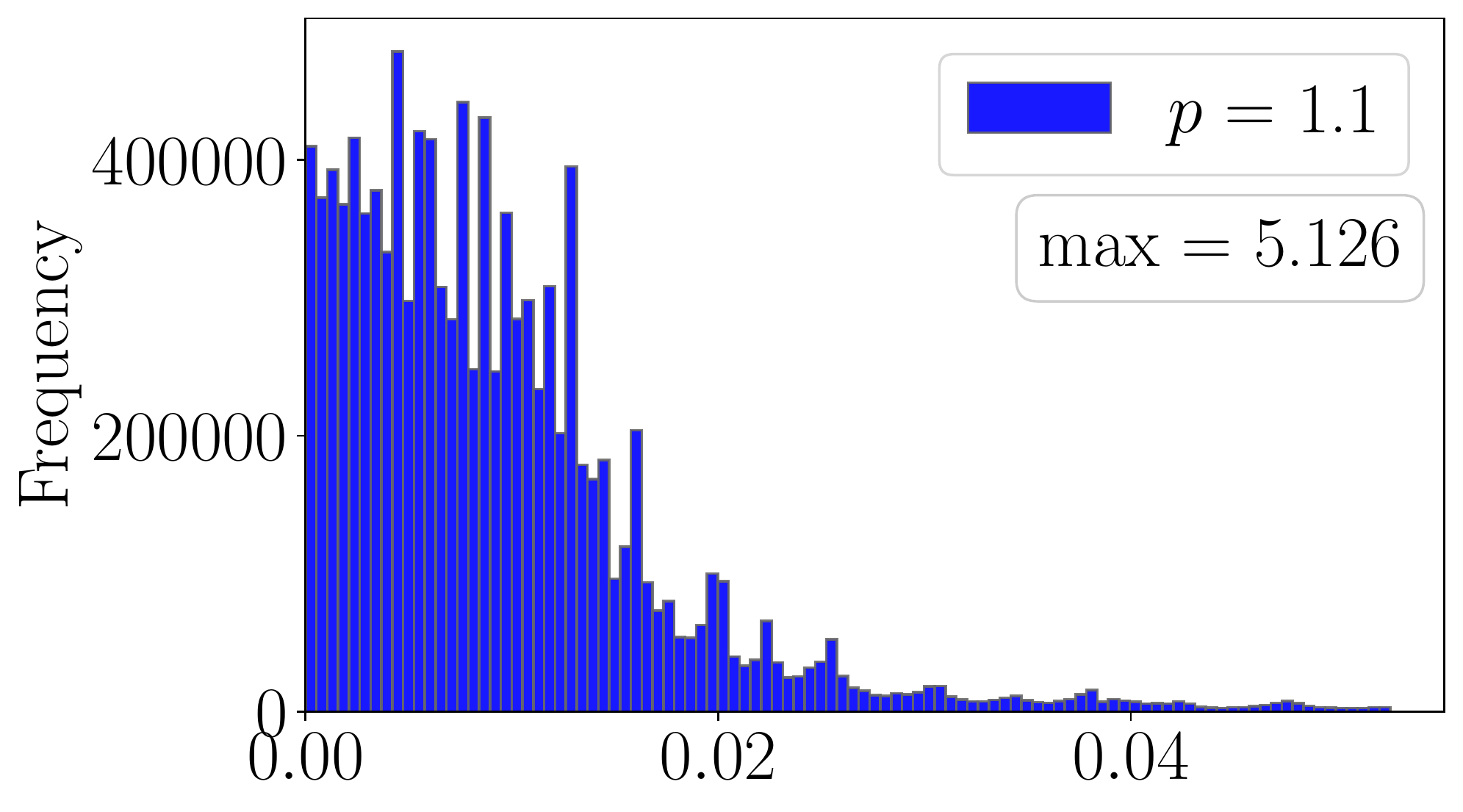}
    \end{subfigure}
    ~
    \begin{subfigure}[b]{0.45\textwidth}
        \includegraphics[width=\textwidth]{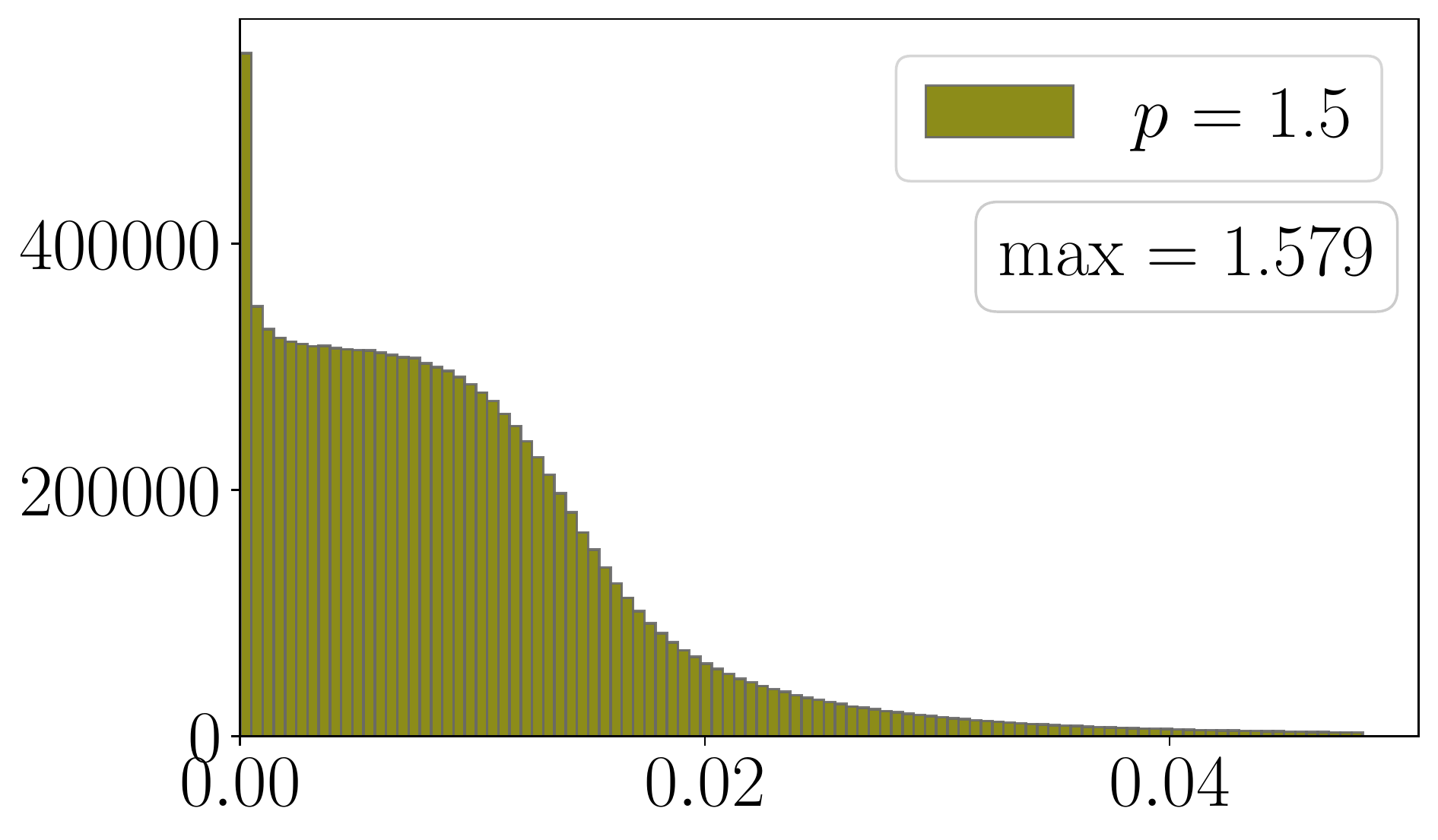}
    \end{subfigure}
    \begin{subfigure}[b]{0.48\textwidth}
        \includegraphics[width=\textwidth]{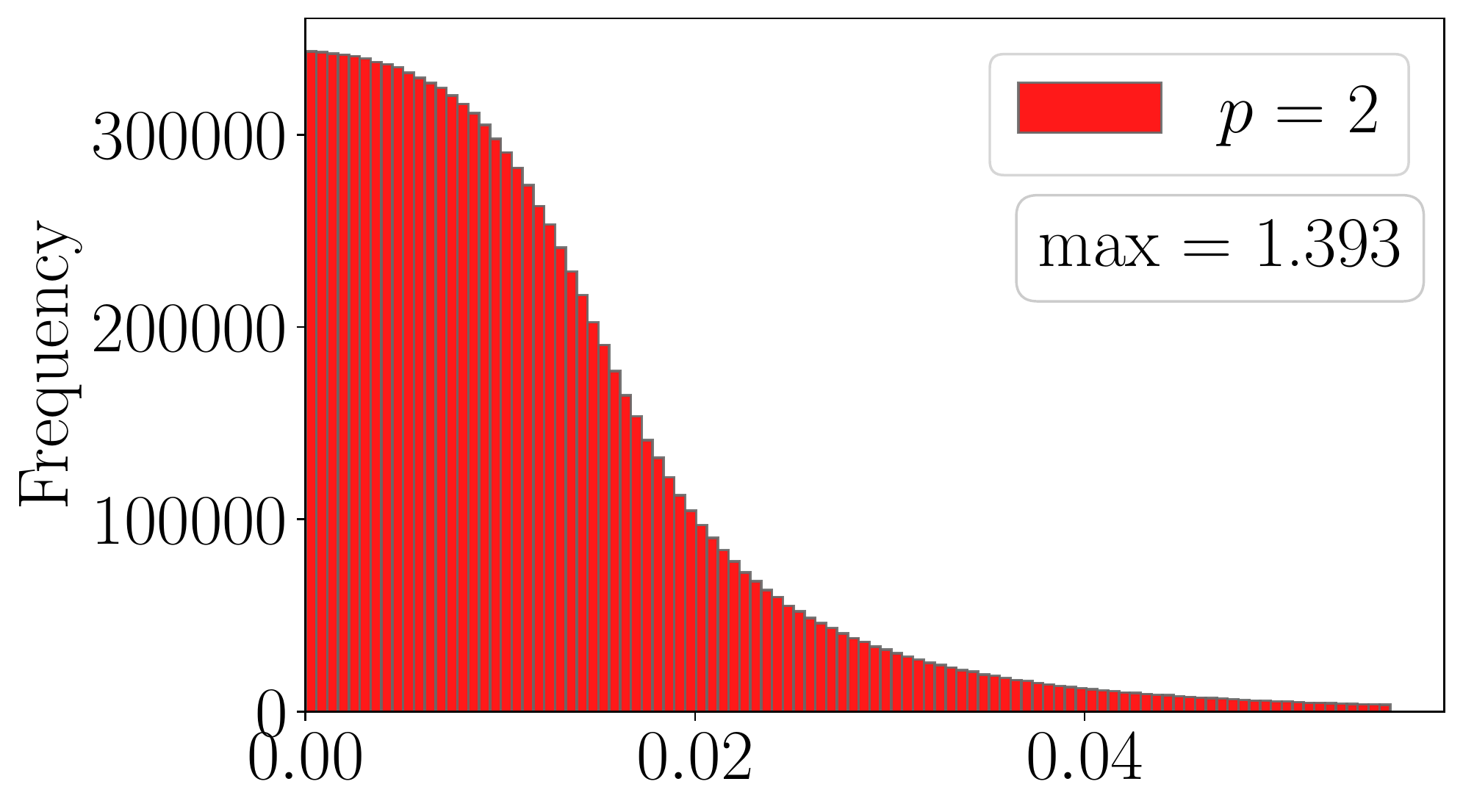}
    \end{subfigure}
    ~
    \begin{subfigure}[b]{0.45\textwidth}
        \includegraphics[width=\textwidth]{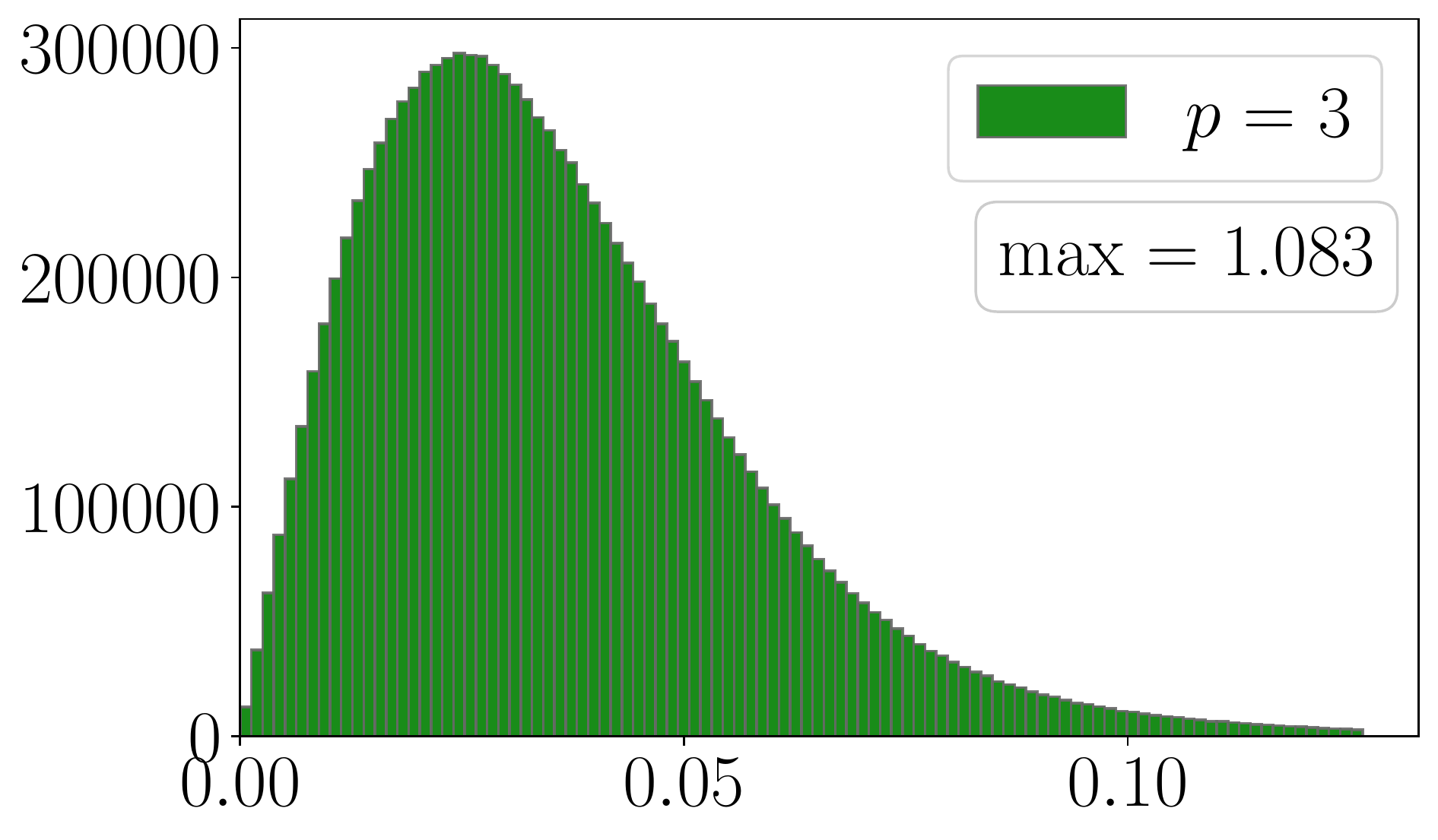}
    \end{subfigure}
    \begin{subfigure}[b]{0.48\textwidth}
        \includegraphics[width=\textwidth]{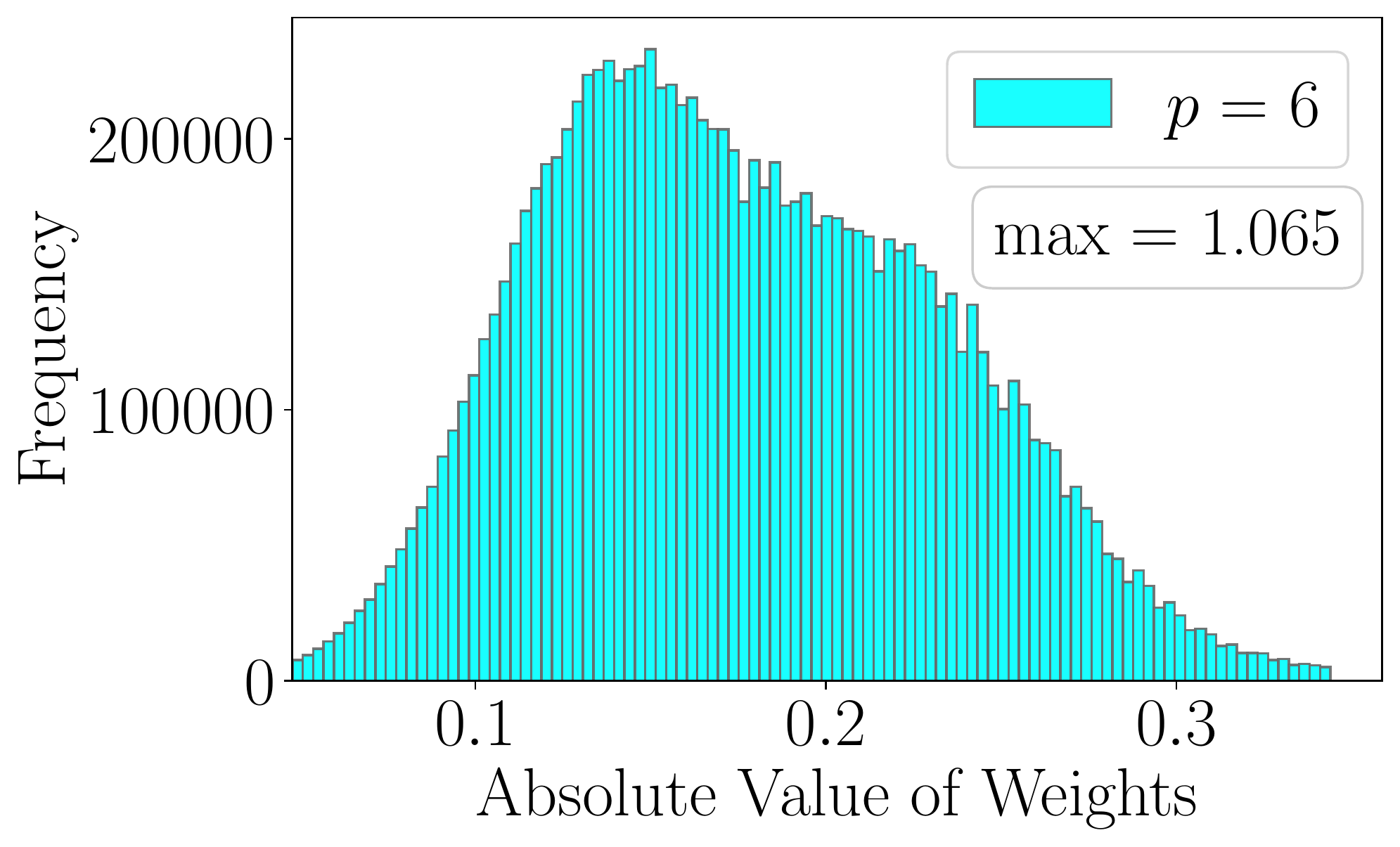}
    \end{subfigure}
    ~
    \begin{subfigure}[b]{0.45\textwidth}
        \includegraphics[width=\textwidth]{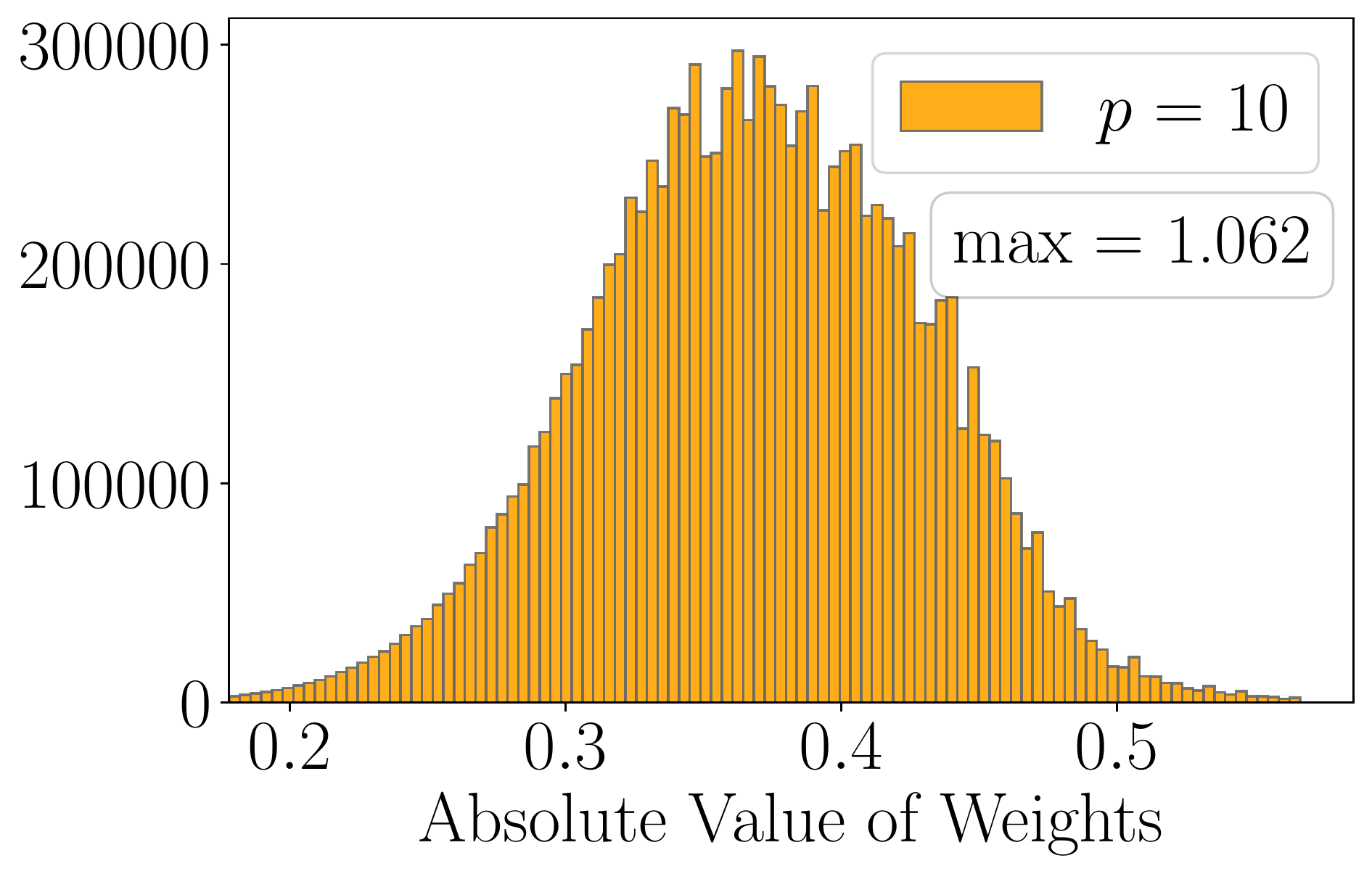}
    \end{subfigure}
    \caption{The histogram of weights in \textsc{ResNet-18} models trained with \algname for the CIFAR-10 dataset. 
    For clarity, we cropped out the tails and each plot has 100 bins after cropping.
    Note that the scale on the $y$-axis differs per graph.
    }
    \label{fig:cifar10-hist-resnet-full}
\end{figure}

\begin{figure}
    \centering
    \begin{subfigure}[b]{0.48\textwidth}
        \centering
        \includegraphics[width=\textwidth]{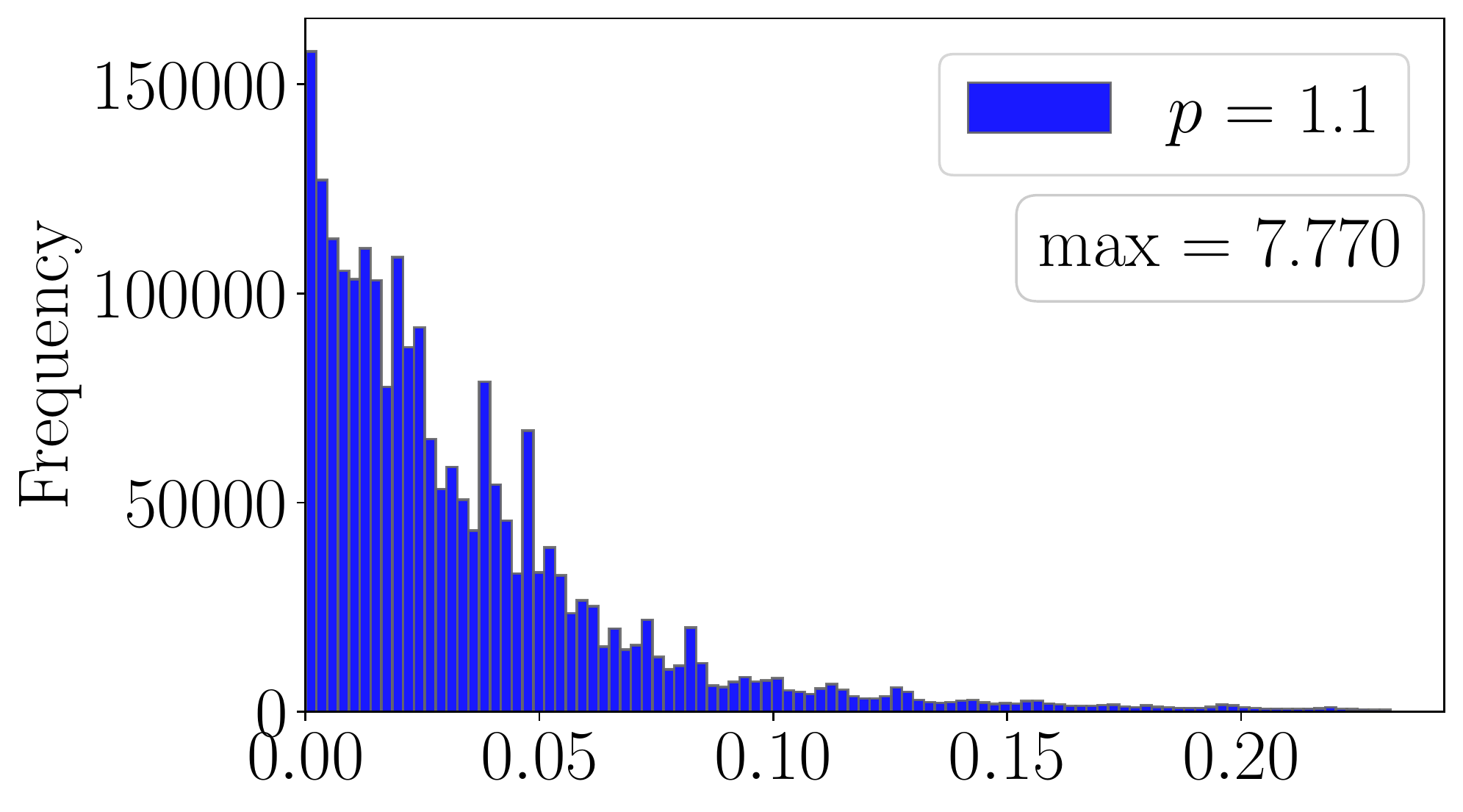}
    \end{subfigure}
    ~
    \begin{subfigure}[b]{0.45\textwidth}
        \includegraphics[width=\textwidth]{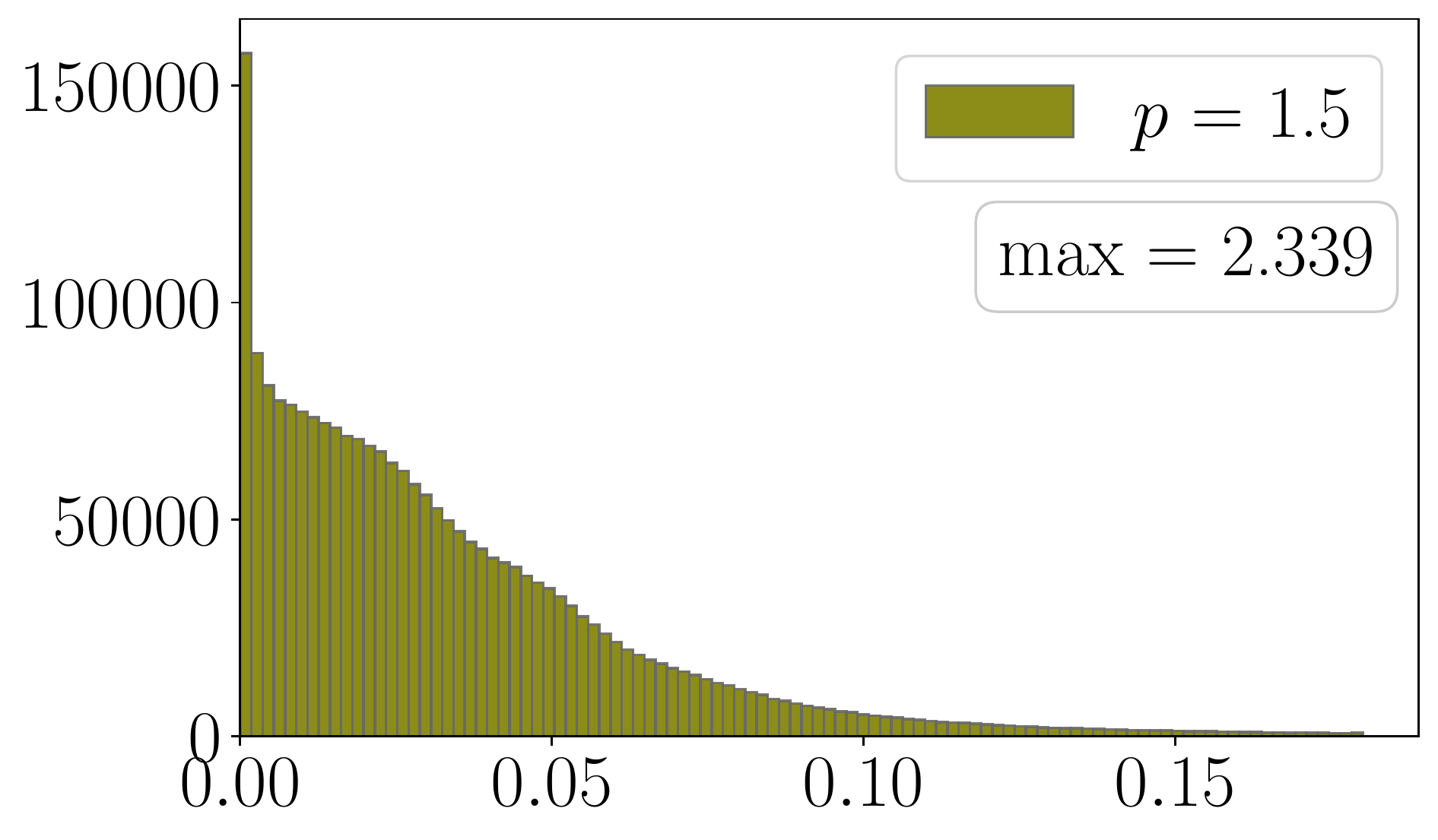}
    \end{subfigure}
    \begin{subfigure}[b]{0.48\textwidth}
        \includegraphics[width=\textwidth]{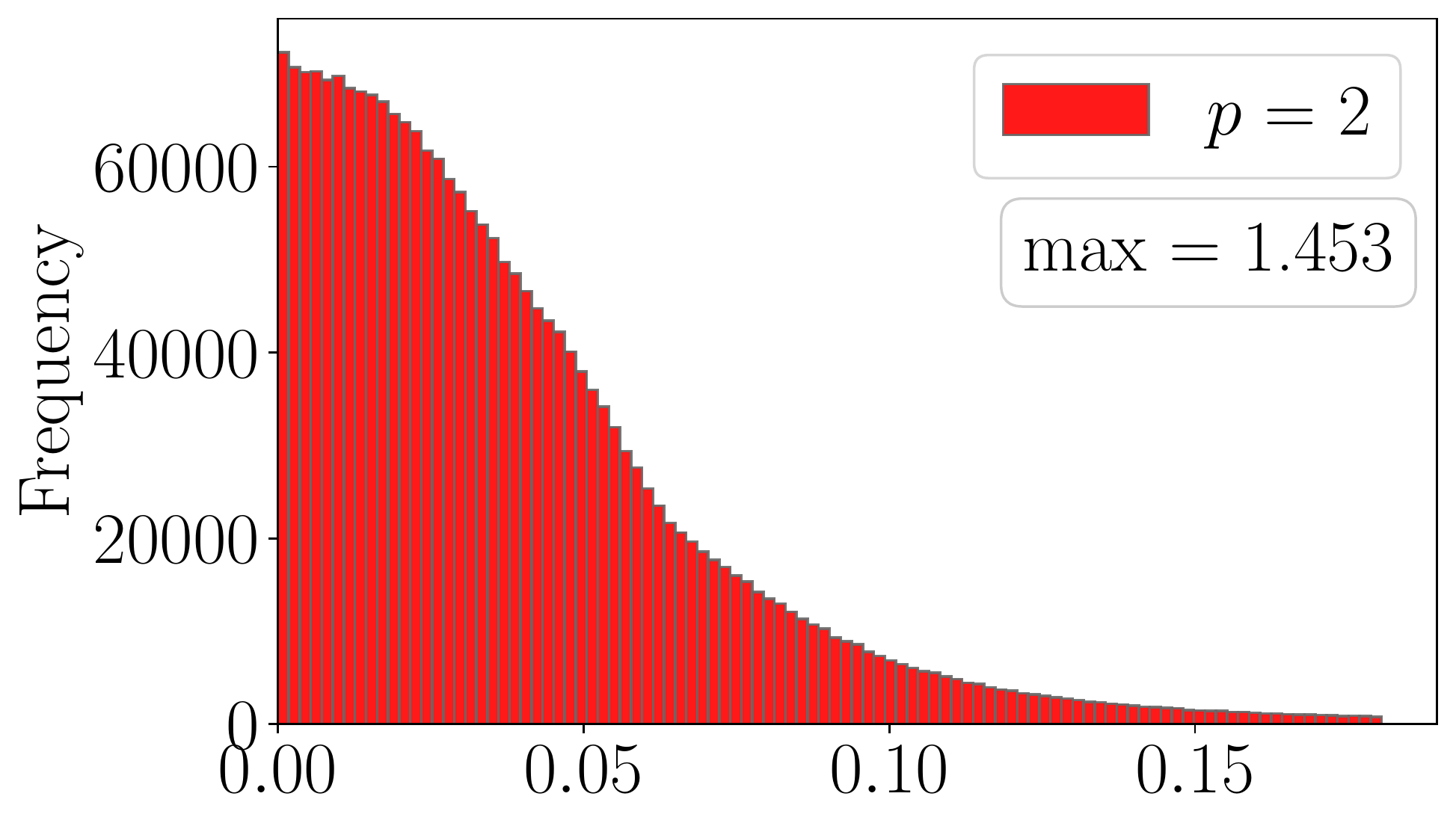}
    \end{subfigure}
    ~
    \begin{subfigure}[b]{0.45\textwidth}
        \includegraphics[width=\textwidth]{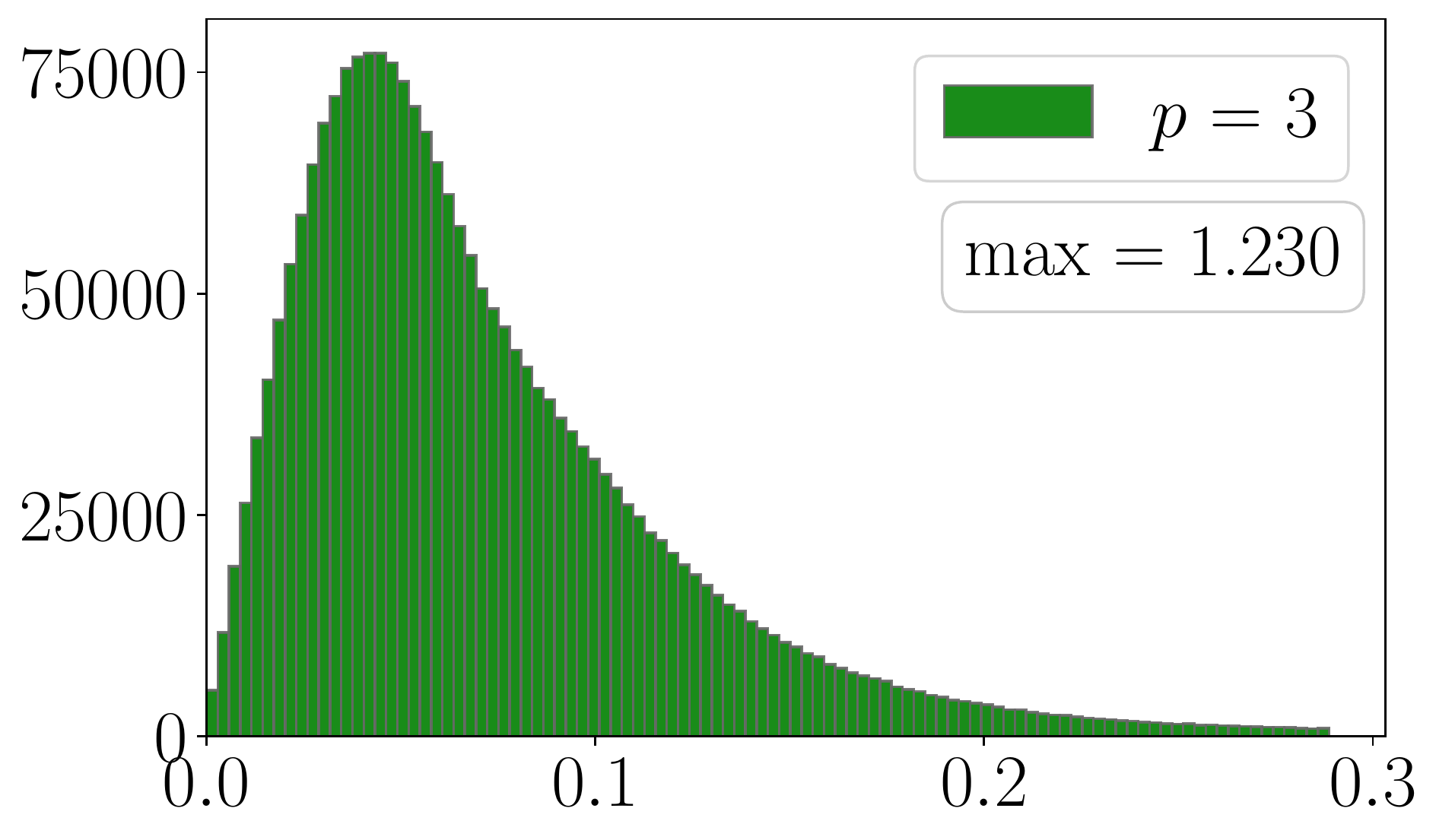}
    \end{subfigure}
    \begin{subfigure}[b]{0.48\textwidth}
        \includegraphics[width=\textwidth]{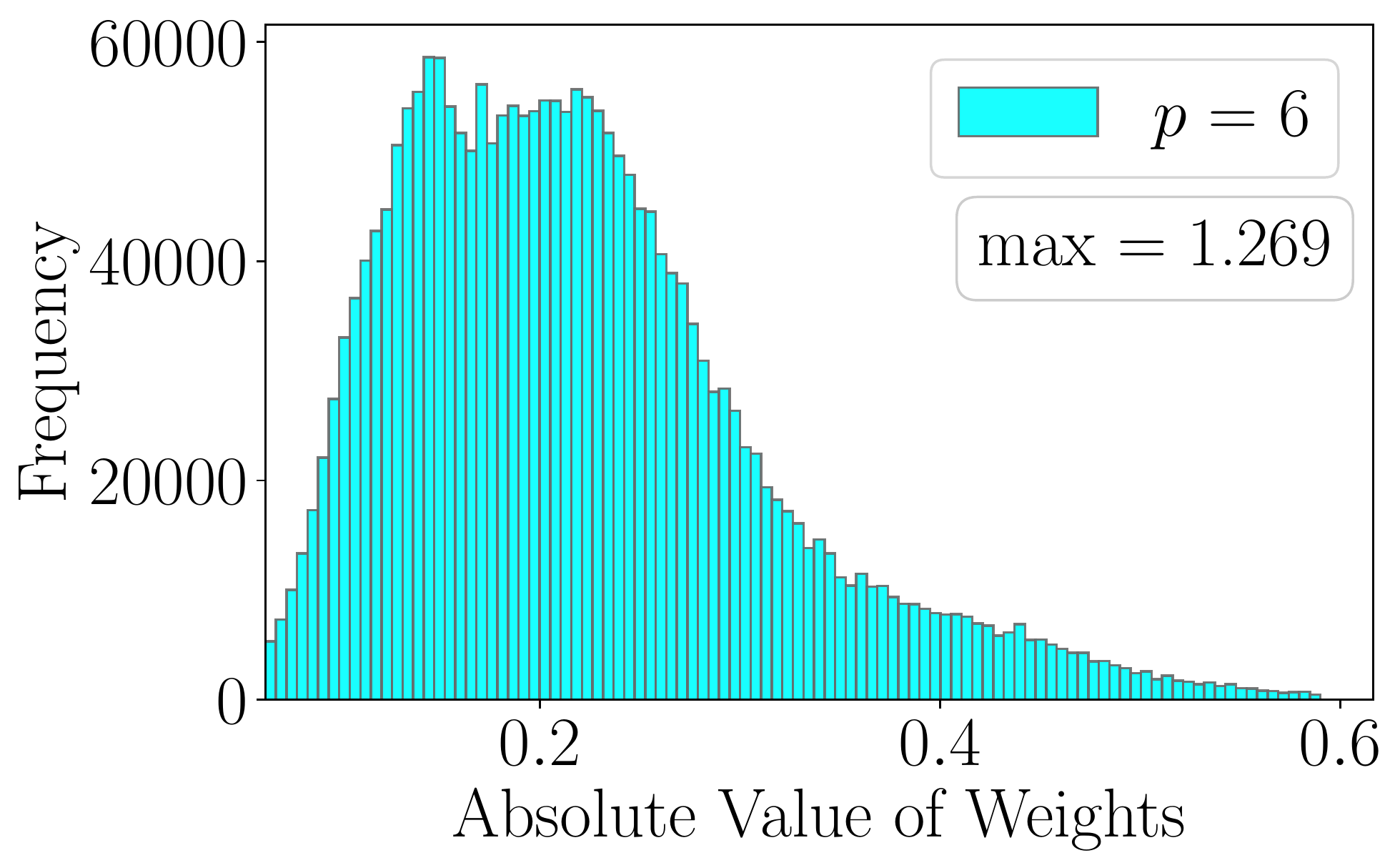}
    \end{subfigure}
    ~
    \begin{subfigure}[b]{0.45\textwidth}
        \includegraphics[width=\textwidth]{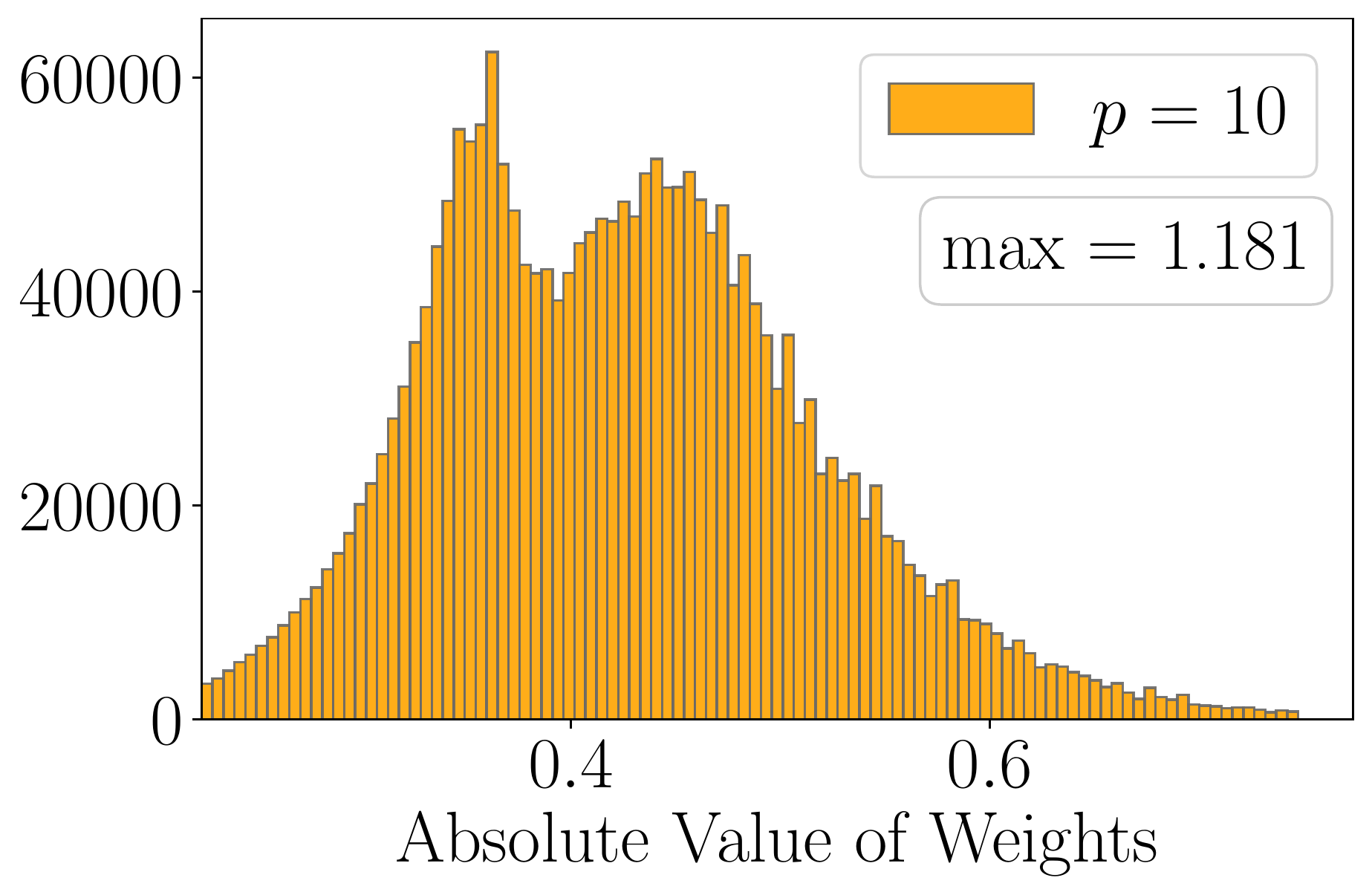}
    \end{subfigure}
    \caption{The histogram of weights in \textsc{MobileNet-v2} models trained with \algname for the CIFAR-10 dataset. 
    For clarity, we cropped out the tails and each plot has 100 bins after cropping.
    }
    \label{fig:cifar10-hist-mobilenet-full}
\end{figure}

\begin{figure}
    \centering
    \begin{subfigure}[b]{0.48\textwidth}
        \centering
        \includegraphics[width=\textwidth]{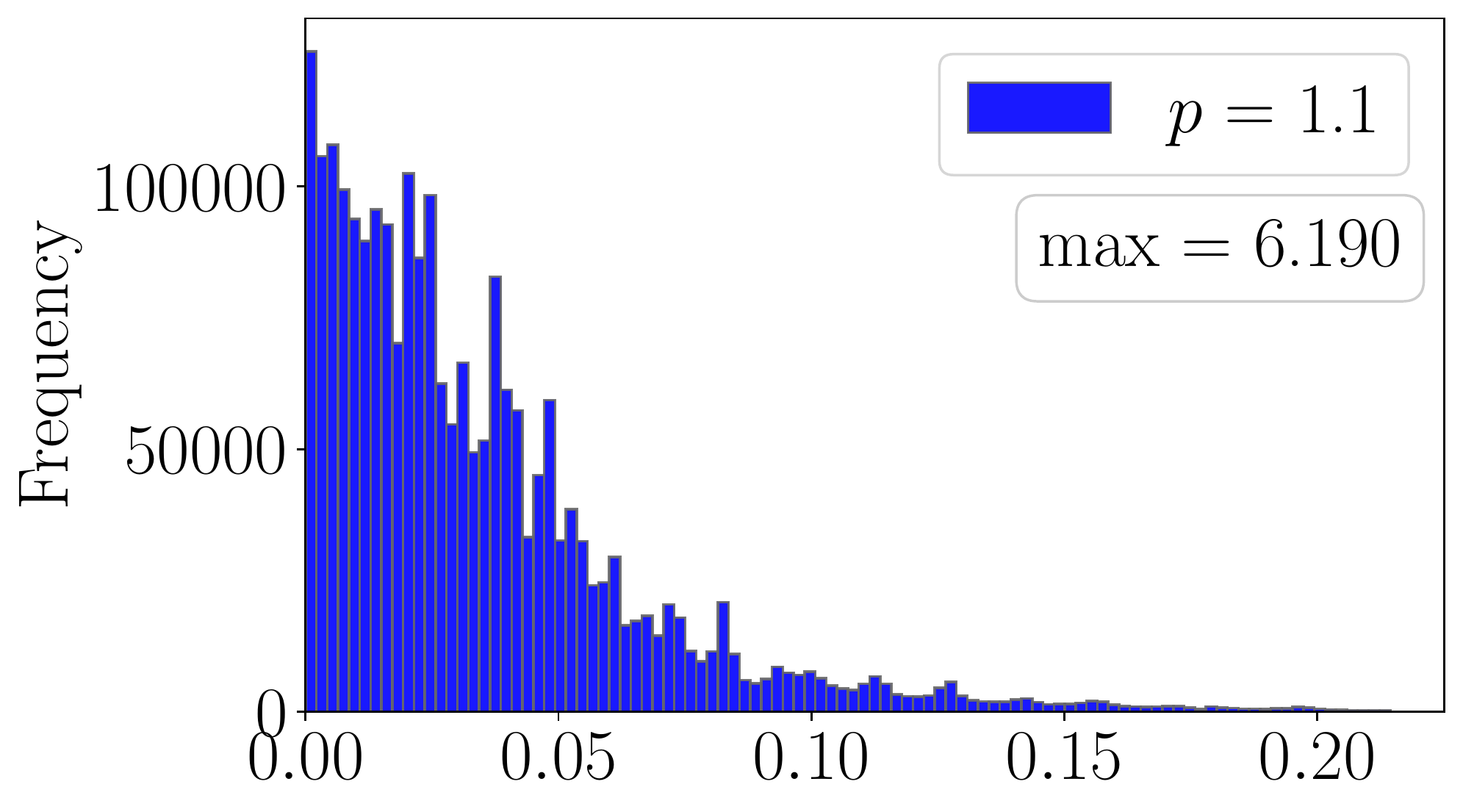}
    \end{subfigure}
    ~
    \begin{subfigure}[b]{0.45\textwidth}
        \includegraphics[width=\textwidth]{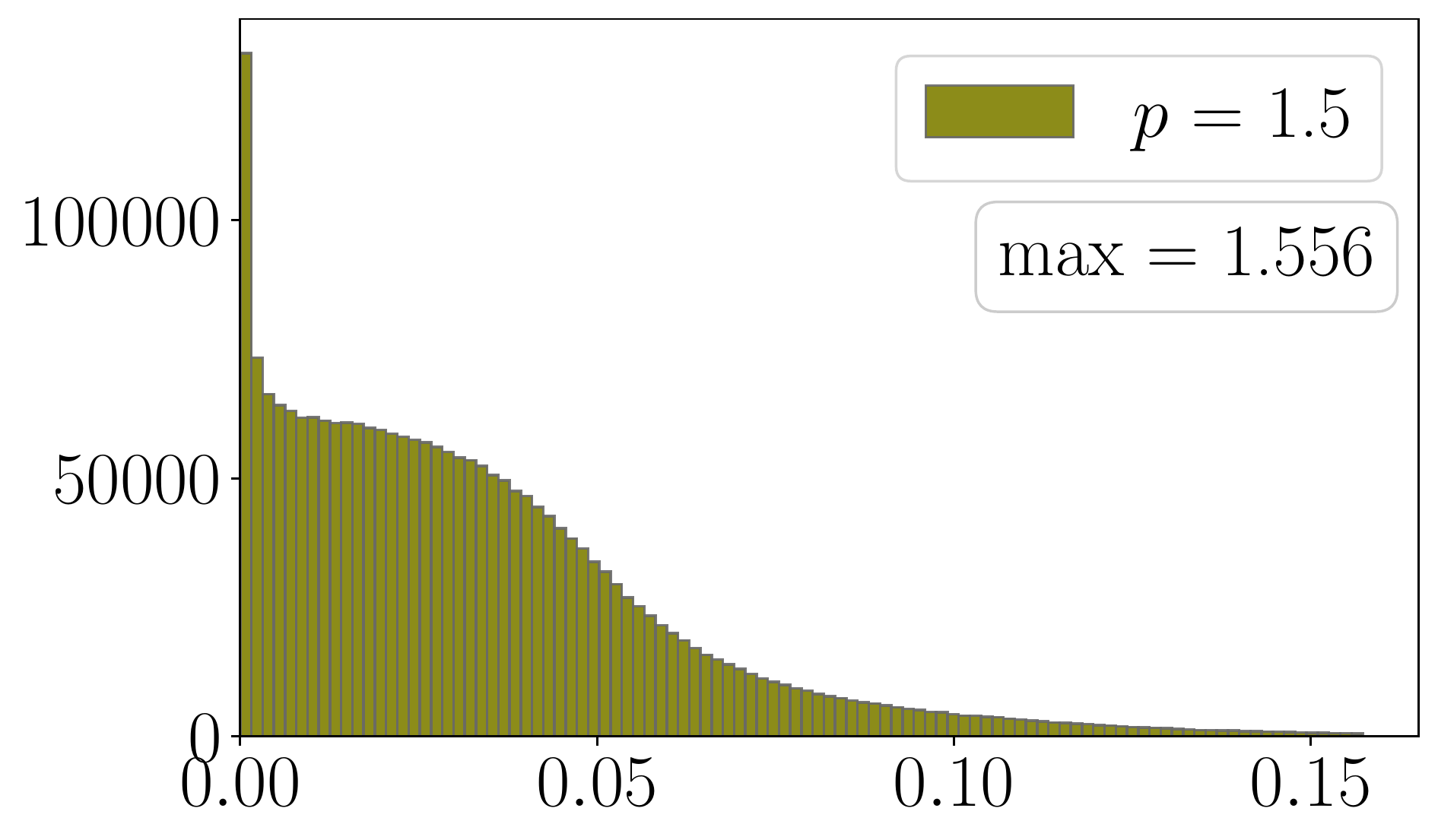}
    \end{subfigure}
    \begin{subfigure}[b]{0.48\textwidth}
        \includegraphics[width=\textwidth]{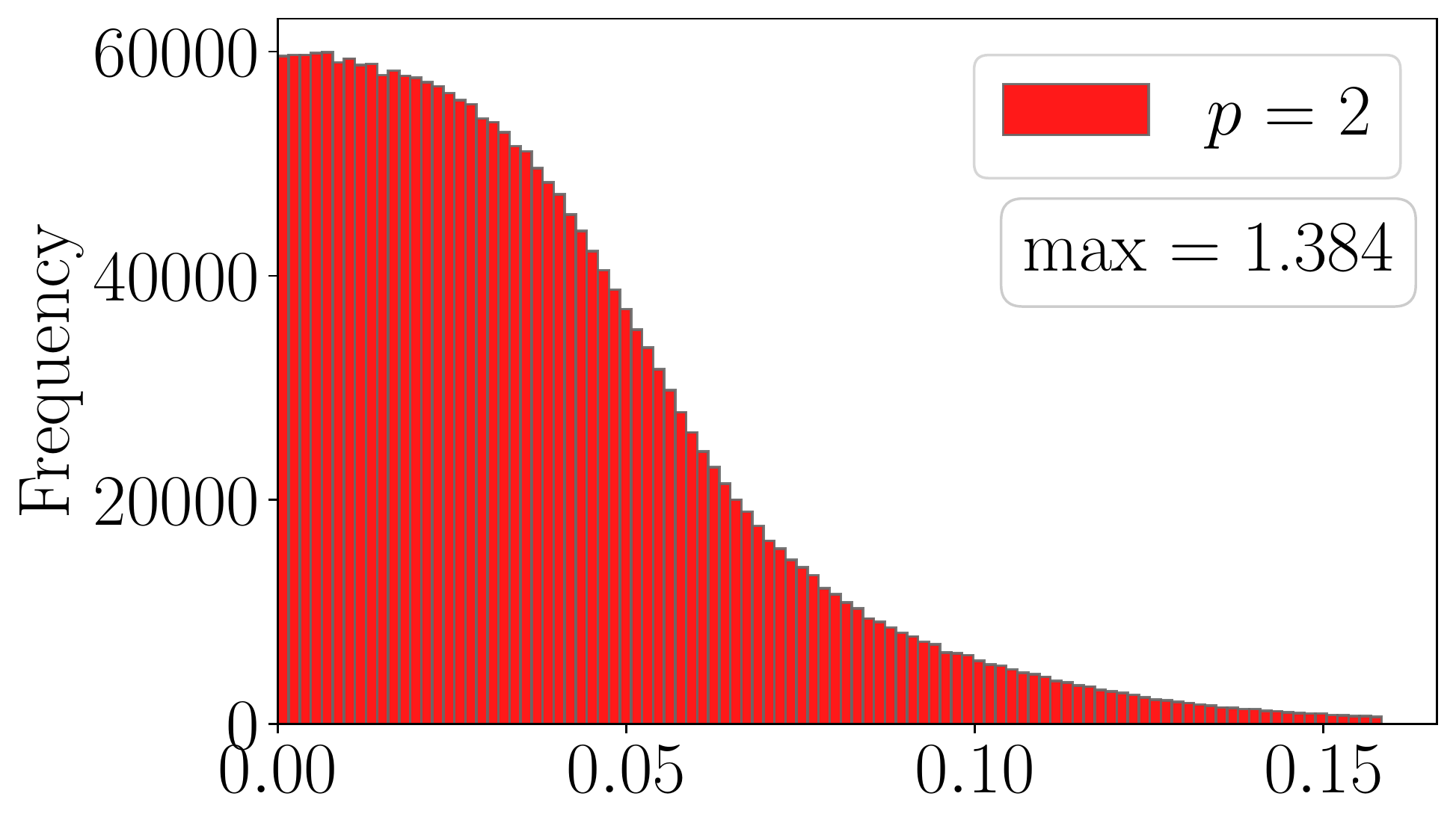}
    \end{subfigure}
    ~
    \begin{subfigure}[b]{0.45\textwidth}
        \includegraphics[width=\textwidth]{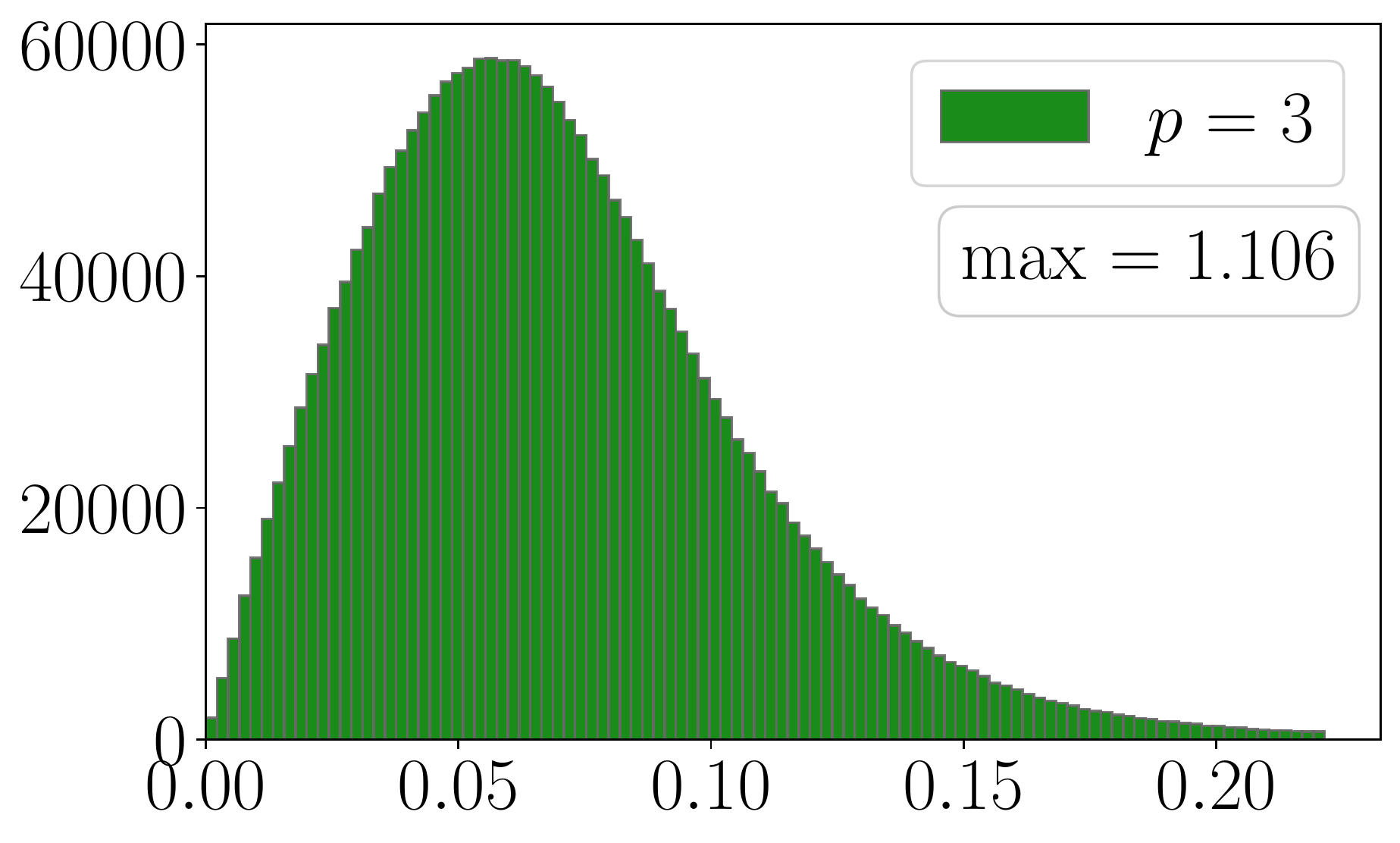}
    \end{subfigure}
    \begin{subfigure}[b]{0.48\textwidth}
        \includegraphics[width=\textwidth]{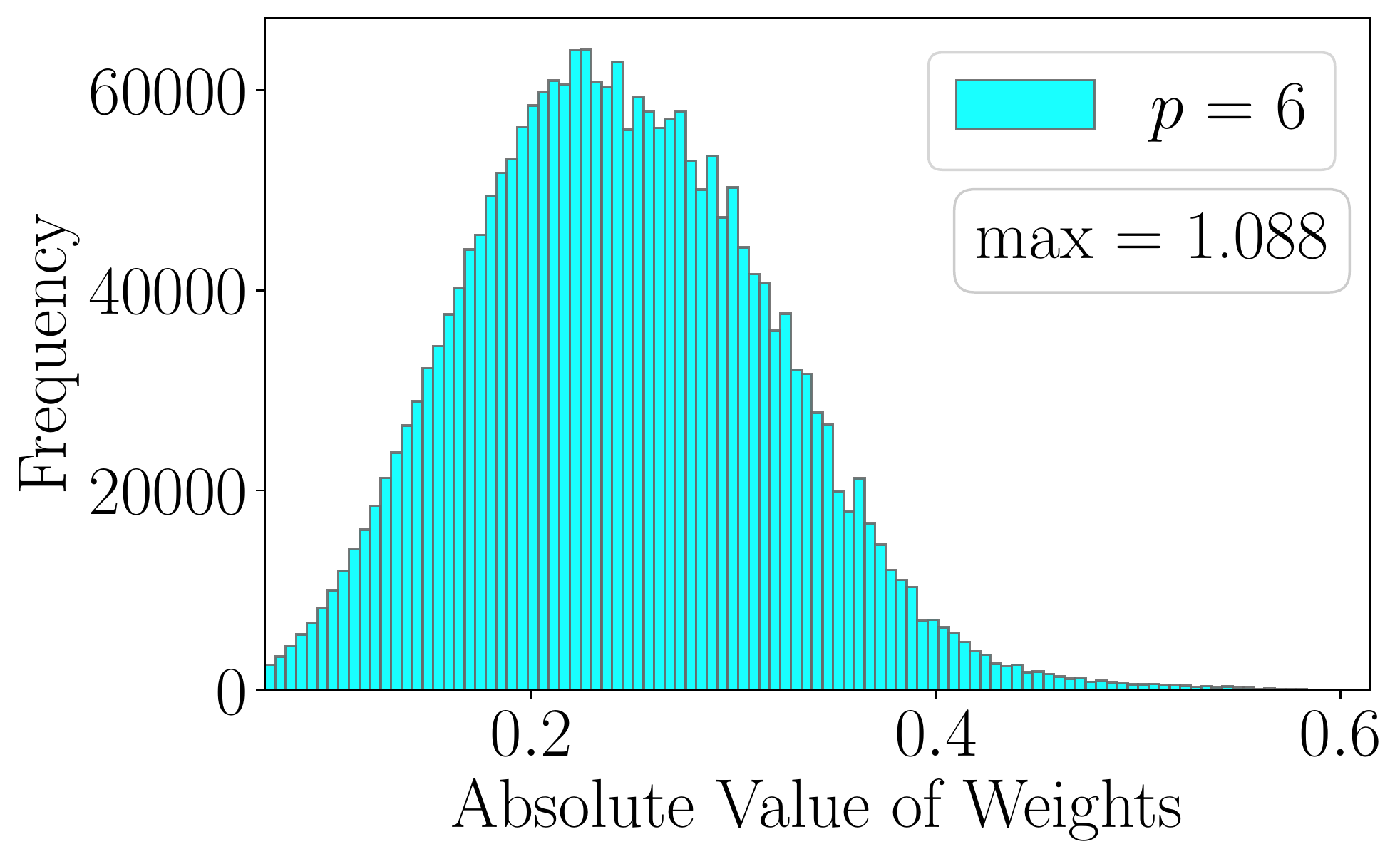}
    \end{subfigure}
    ~
    \begin{subfigure}[b]{0.45\textwidth}
        \includegraphics[width=\textwidth]{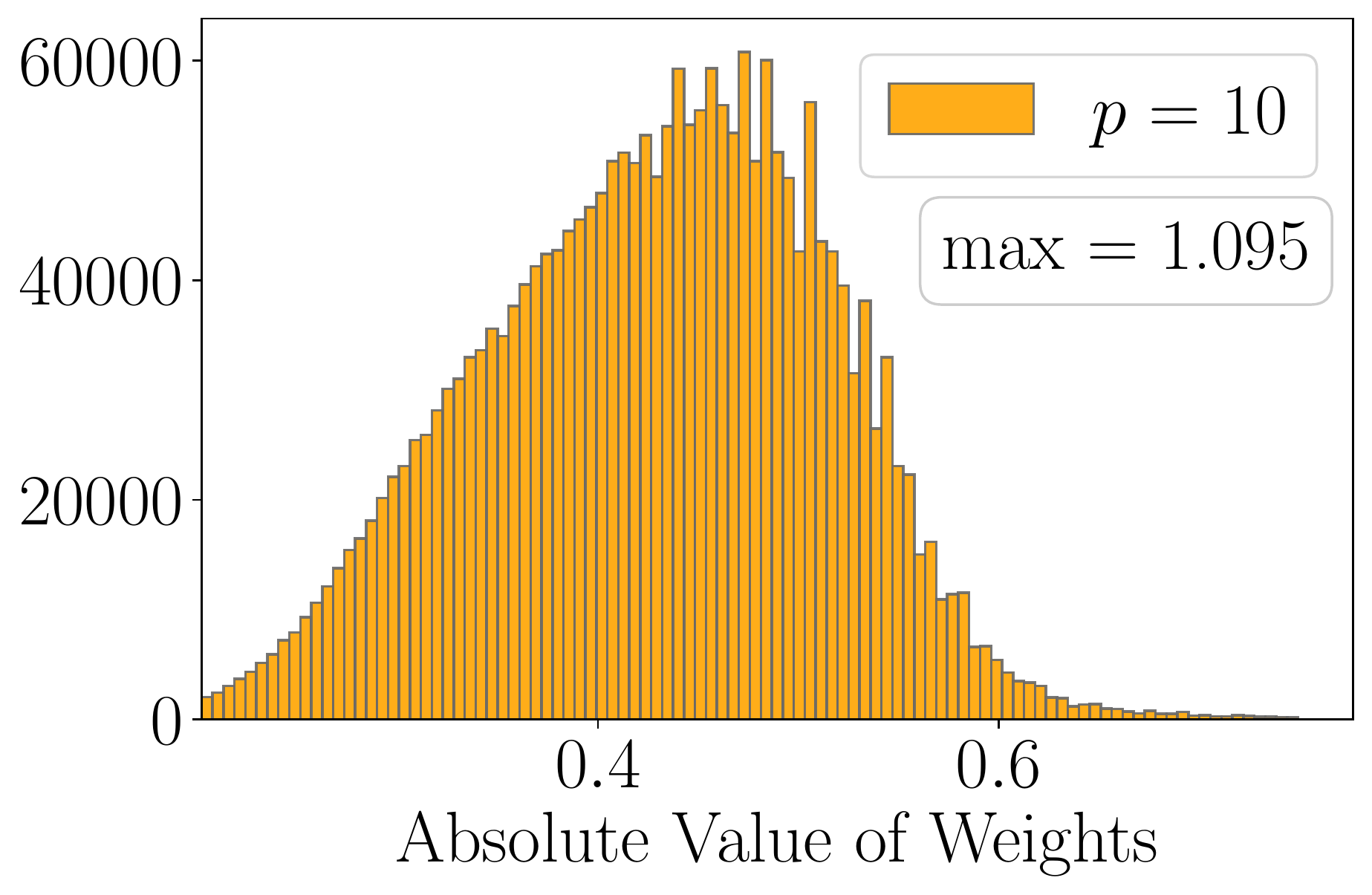}
    \end{subfigure}
    \caption{The histogram of weights in \textsc{RegNetX-200mf} models trained with \algname for the CIFAR-10 dataset. 
    For clarity, we cropped out the tails and each plot has 100 bins after cropping.
    }
    \label{fig:cifar10-hist-regnet-full}
\end{figure}

\begin{figure}
    \centering
    \begin{subfigure}[b]{0.48\textwidth}
        \centering
        \includegraphics[width=\textwidth]{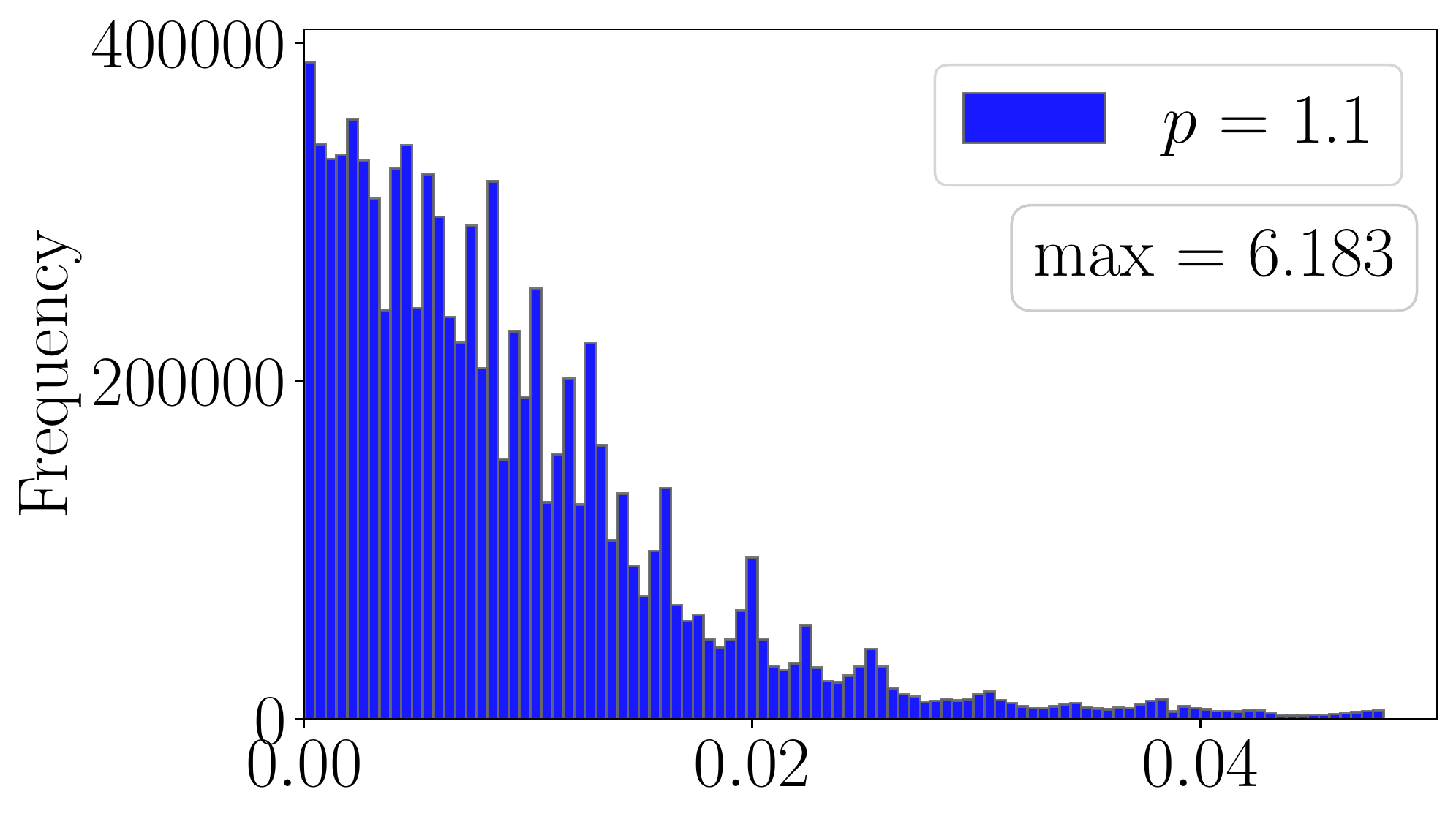}
    \end{subfigure}
    ~
    \begin{subfigure}[b]{0.45\textwidth}
        \includegraphics[width=\textwidth]{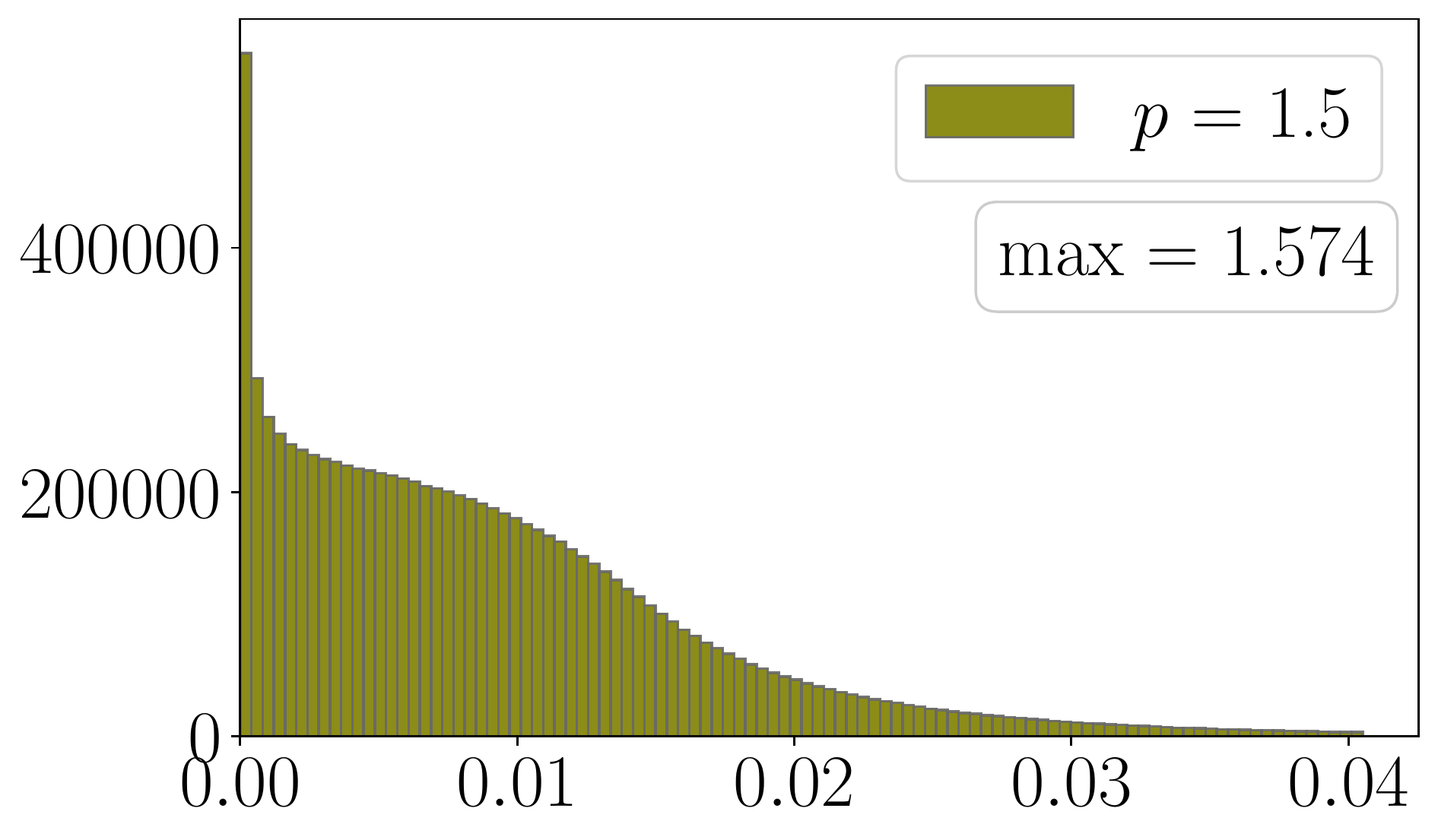}
    \end{subfigure}
    \begin{subfigure}[b]{0.48\textwidth}
        \includegraphics[width=\textwidth]{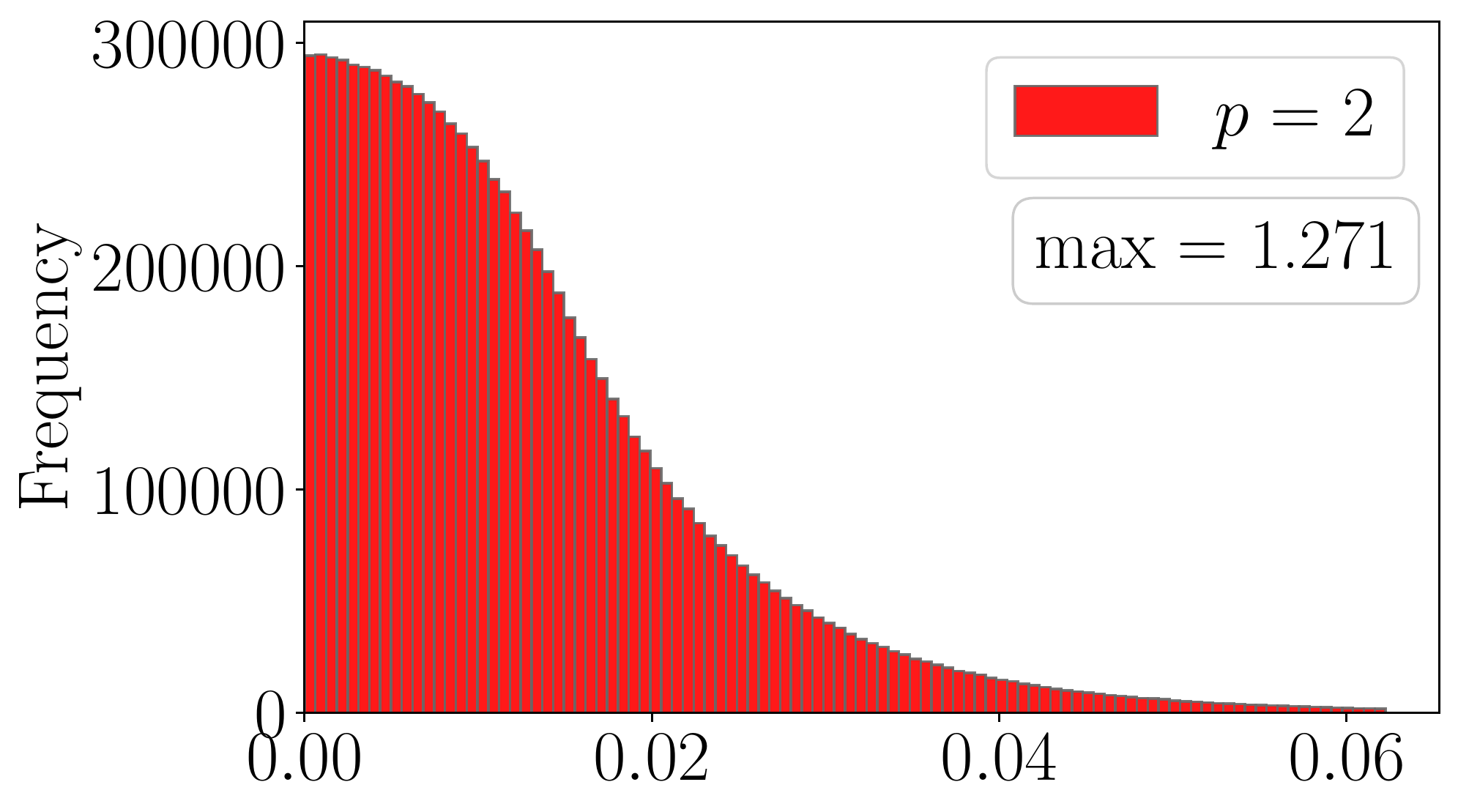}
    \end{subfigure}
    ~
    \begin{subfigure}[b]{0.45\textwidth}
        \includegraphics[width=\textwidth]{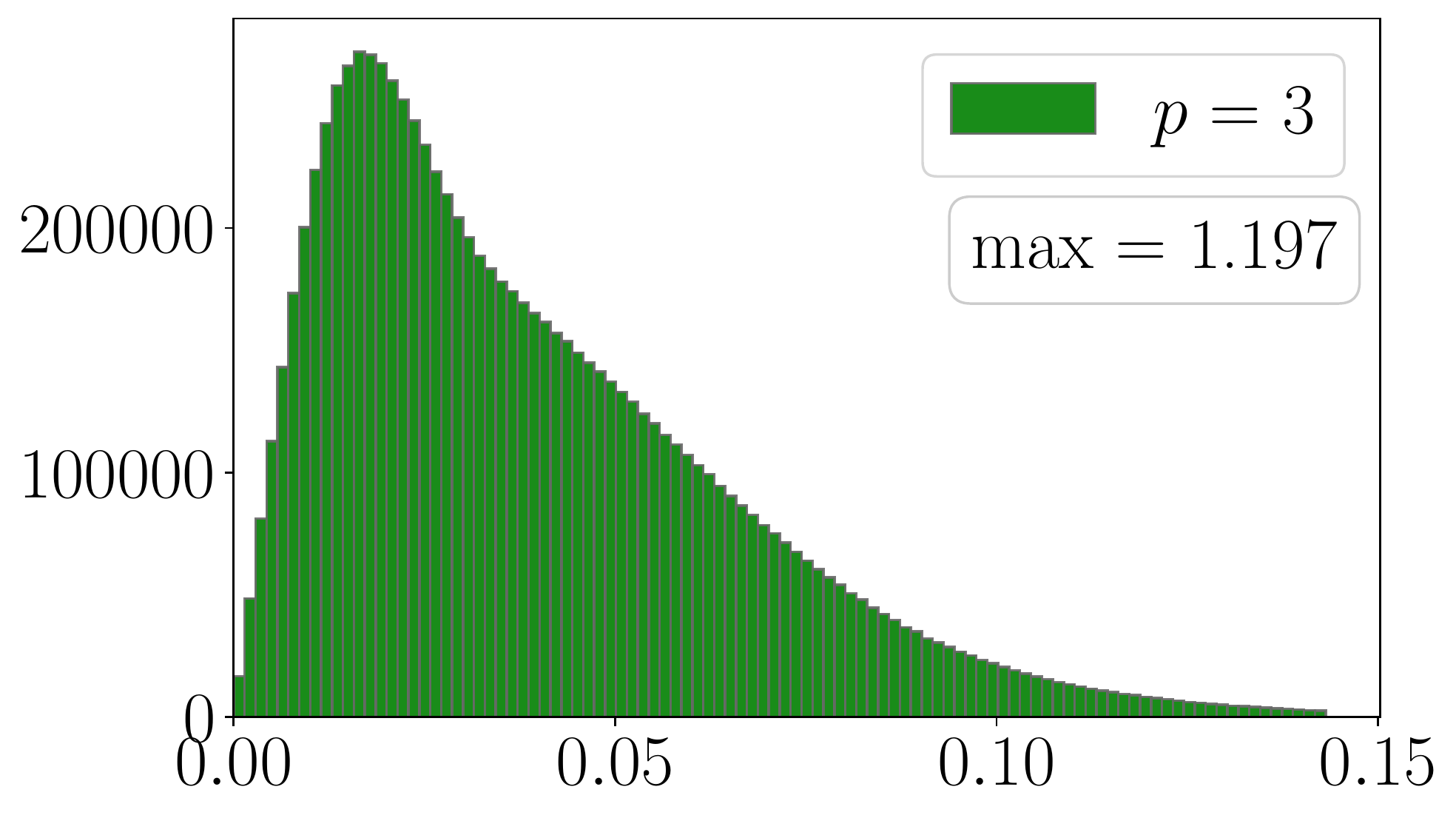}
    \end{subfigure}
    \begin{subfigure}[b]{0.48\textwidth}
        \includegraphics[width=\textwidth]{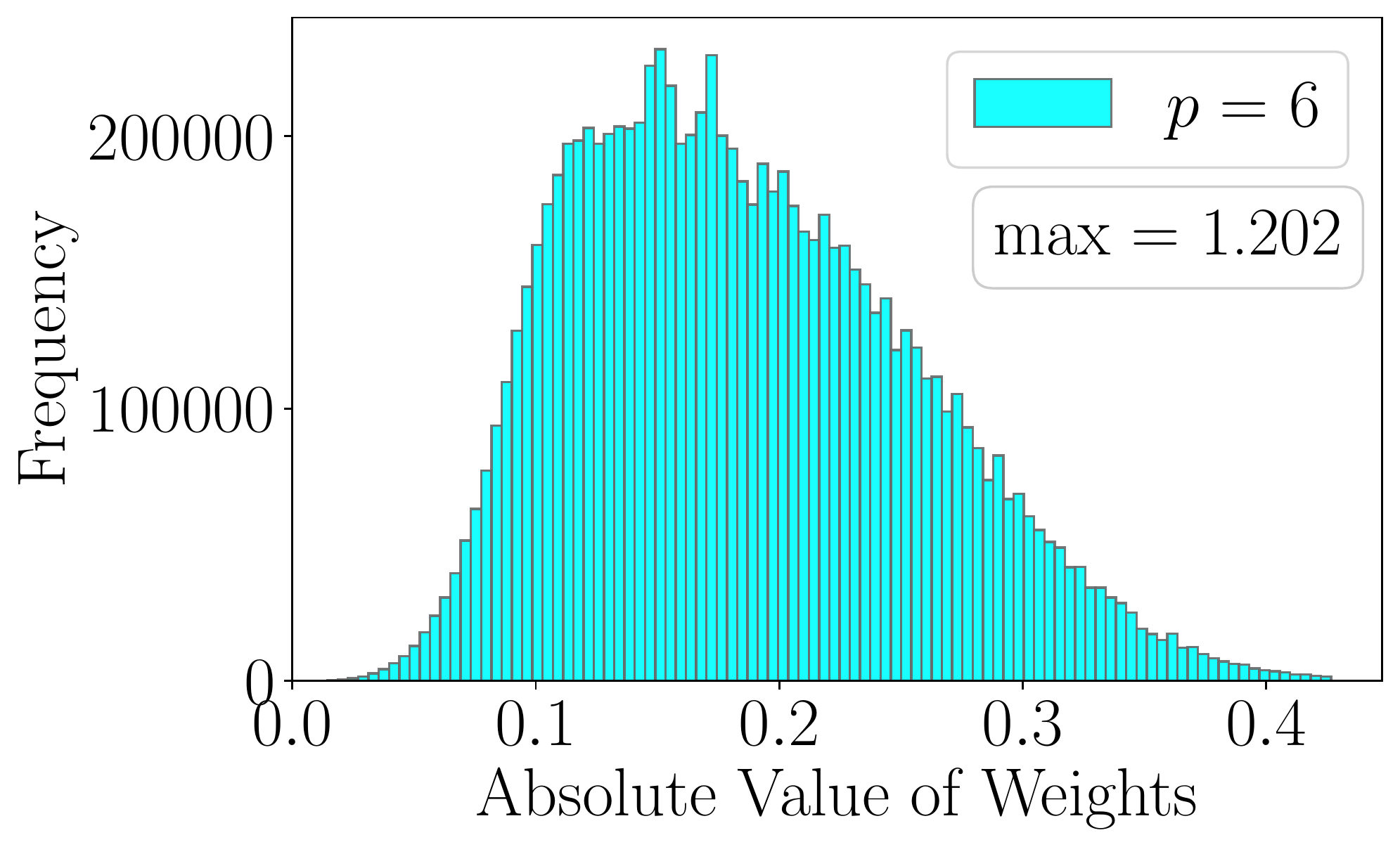}
    \end{subfigure}
    ~
    \begin{subfigure}[b]{0.45\textwidth}
        \includegraphics[width=\textwidth]{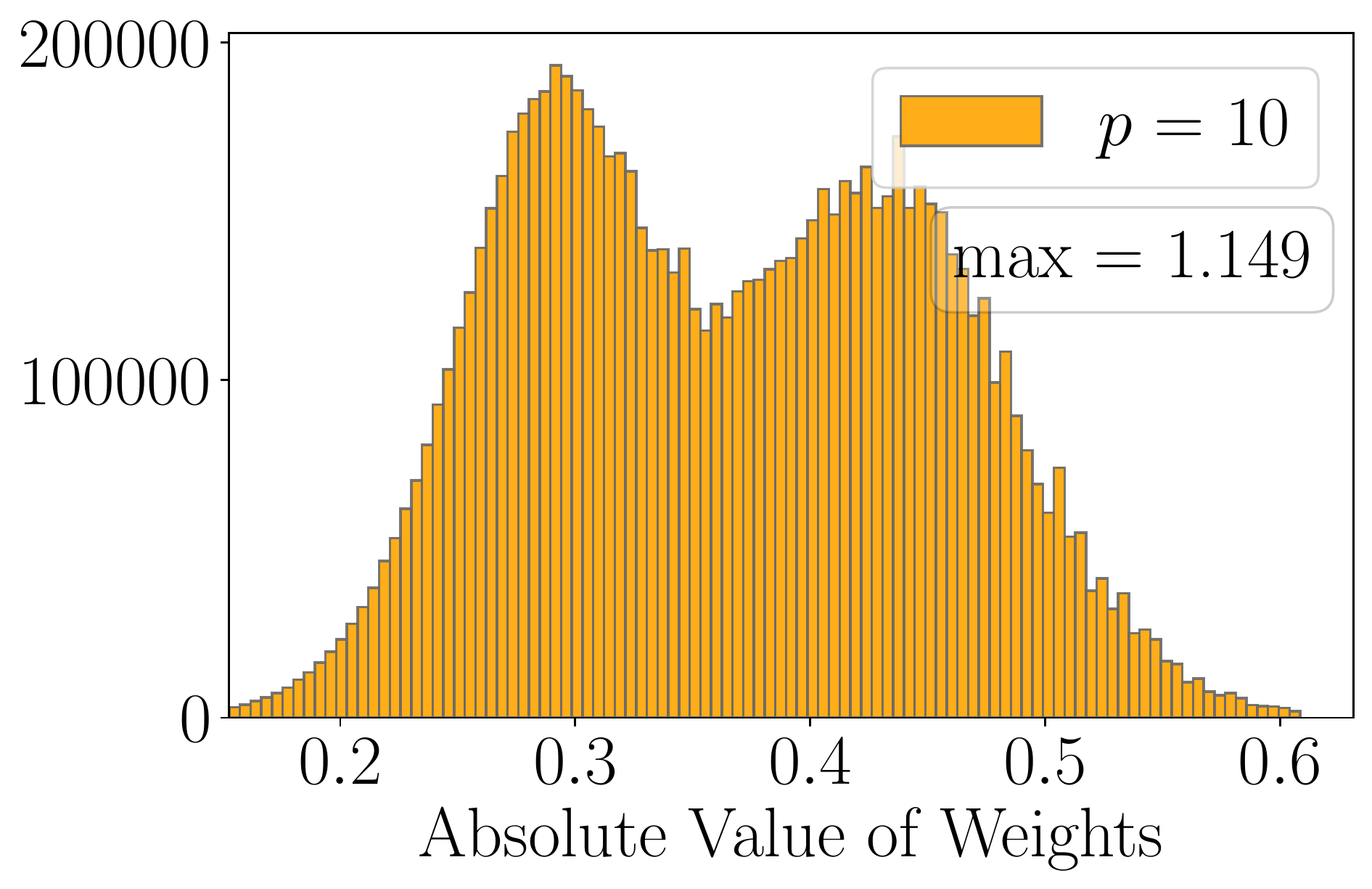}
    \end{subfigure}
    \caption{The histogram of weights in \textsc{VGG-11} models trained with \algname for the CIFAR-10 dataset. 
    For clarity, we cropped out the tails and each plot has 100 bins after cropping.
    }
    \label{fig:cifar10-hist-vgg-full}
\end{figure}

\clearpage

\subsection{CIFAR-10 experiments: generalization}
We present a more complete result for the CIFAR-10 generalization experiment in Section~\ref{sec:cifar} with additional values of $p$.

In the following table, we see that \algname with $p = 3$ continues to have the highest generalization performance for all deep neural networks.

\label{sec:add-experiment-cifar-generalization}
\begin{table}[!h]
    \centering
    \setlength{\tabcolsep}{5.5pt}
    \begin{tabular}{l| c|c|c|c}
         \hline
         &  \hspace{1.25em} \textsc{VGG-11} \hspace{1.25em} & \hspace{0.75em} \textsc{ResNet-18} \hspace{0.75em} & \textsc{MobileNet-v2} & \textsc{RegNetX-200mf}  \\
         \hline \hline
         $p = 1.1$ & \pmval{88.19}{.17} & \pmval{92.63}{.12} & \pmval{91.16}{.09}& \pmval{91.21}{.18}  \\
         $p = 1.5$ & \pmval{88.45}{.29} & \pmval{92.73}{.11} & \pmval{90.81}{.19}& \pmval{90.91}{.12} \\
         $p = 2$ (SGD) & \pmval{90.15}{.16} & \bpmval{93.90}{.14} & \pmval{91.97}{.10}& \pmval{92.75}{.13} \\
         $p = 3$ & \bpmval{90.85}{.15} & \bpmval{94.01}{.13} & \bpmval{93.23}{.26}& \bpmval{94.07}{.12} \\
         $p = 6$ & \pmval{89.47}{.14} & \bpmval{93.87}{.13} & \pmval{92.84}{.15}& \pmval{93.03}{.17} \\
         $p = 10$ & \pmval{88.78}{.37} & \pmval{93.55}{.21} & \pmval{92.60}{.22}& \pmval{92.97}{.16} \\
         \hline
    \end{tabular}
    \caption{CIFAR-10 test accuracy (\%) of \algname on various deep neural networks. For each deep net and value of $p$, the average $\pm$ \textcolor{gray}{std. dev.} over 5 trials are reported. The best-performing value(s) of $p$ for each individual deep net is highlighted in \textbf{boldface}.}
    \label{tab:generalization-cifar10-full}
\end{table}

\subsection{ImageNet experiments}
\label{sec:add-experiment-imagenet}
To verify if our observations on the CIFAR-10 generalization performance hold up for other datasets, we also performed similar experiments for the much larger ImageNet dataset.
Due to computational constraints, we were only able to experiment with the \textsc{ResNet-18} and \textsc{MobileNet-v2} architectures and only for one trial.

It is worth noting that the neural networks we used cannot reach 100\% training accuracy on Imagenet.
The models we employed only achieved top-1 training accuracy in the mid-70s.
So, we are not in the so-called \textit{interpolation regime}, and there are many other factors that can significantly impact the generalization performance of the trained models.
In particular, we find that not having weight decay costs us around 3\% in validation accuracy in the $p = 2$ case and this explains why our reported numbers are lower than PyTorch's baseline for each corresponding architecture.
Despite this, we find that \algname with $p = 3$ has the best generalization performance on the ImageNet dataset, matching our observation from the CIFAR-10 dataset.

\begin{table}[!h]
    \centering
    \begin{tabular}{l| c | c}
         \hline
        & \textsc{ResNet-18} & \textsc{MobileNet-v2} \\
        \hline\hline
        $p=1.1$ & 64.08 & 63.41 \\
        $p=1.5$ & 65.14 & 65.75 \\
        $p=2$ (SGD) & 66.76 &  67.91 \\
        $p=3$ & \textbf{67.67} & \textbf{69.74} \\
        $p=6$ & 66.69 & 67.05 \\
        $p=10$ & 65.10 & 62.32 \\
         \hline
    \end{tabular}
    \caption{ImageNet top-1 validation accuracy (\%) of \algname on various deep neural networks. The best-performing value(s) of $p$ for each individual deep network is highlighted in \textbf{boldface}.}
    \label{tab:imagenet}
\end{table}

\end{document}